\documentclass[12pt]{article}
\usepackage[margin=1in]{geometry}

\usepackage{times,upgreek,morenotations,rotating}

\usepackage[utf8]{inputenc} 
\usepackage[T1]{fontenc}    
\usepackage{url}            
\usepackage{booktabs}       
\usepackage{amsfonts}       
\usepackage{nicefrac}       
\usepackage{microtype}      

\usepackage{morenotations} 

\usepackage{times}
\usepackage{graphicx} 
\usepackage{subfigure}
\usepackage{cleveref}
\usepackage{natbib}
\usepackage{bm} 
\usepackage{fancyhdr}
\usepackage{graphicx}
\usepackage{color}
\usepackage{natbib}
\usepackage{booktabs}

\usepackage{tikz}
\usetikzlibrary{arrows,decorations.pathmorphing,backgrounds,positioning,fit,petri}

\title{Private federated learning on vertically partitioned data via
entity resolution and additively homomorphic encryption}

\author{
Stephen Hardy $\:\:\:$ Wilko Henecka $\:\:\:$ Hamish Ivey-Law $\:\:\:$ Richard
Nock\\
Giorgio Patrini $\:\:\:$ Guillaume Smith $\:\:\:$ Brian Thorne\\\\
{\large N1Analytics, Data61
\footnote{All authors contributed equally. Richard Nock is jointly with the Australian National University $\&$ the University
  of Sydney. Giorgio Patrini is now at the University of Amsterdam.}
 }\\\\
  \texttt{firstname.lastname@data61.csiro.au}\\
 \texttt{g.patrini@uva.nl}
}
\date{}

\begin{document}
\thispagestyle{empty}
\maketitle

%

%

\begin{abstract}

  \noindent Consider two data providers, each maintaining 
   private records of different feature sets about common entities. They
   aim to learn  a linear model \emph{jointly} in a federated setting, 
  namely, data is local and
  a shared model is trained from locally computed updates.
  In contrast with most work on distributed learning, in this
  scenario (i) data is split \emph{vertically}, \emph{i.e.} by features,
  (ii) only one data provider knows the target variable and
  (iii) entities are \emph{not} linked across the data providers.
   Hence, to the challenge of private learning, we add the potentially
   negative consequences of mistakes in entity resolution.

Our contribution is twofold. First, 
  we describe a three-party end-to-end solution in two phases---privacy-preserving entity resolution and federated logistic regression over
  messages encrypted with an additively homomorphic scheme---, secure against a honest-but-curious
  adversary.  
  The system allows learning without either exposing data
    in the clear or sharing which entities the data providers have in common.
  Our implementation is as accurate as a naive non-private solution
  that brings all data in one place, and scales to problems with millions of entities with hundreds of features.
  Second, we provide what is to our knowledge the first formal analysis of the impact of entity
  resolution's mistakes on learning, with results on how optimal
  classifiers, empirical losses, margins and generalisation abilities are
  affected. Our results bring a clear and strong support for federated
  learning: under reasonable assumptions on the number and magnitude of entity
  resolution's mistakes, it can be extremely beneficial to carry out
  federated learning in the 
setting where each peer’s data provides a significant uplift to the other.

\end{abstract}

\section{Introduction}

With the ever-expanding collection and use of data,
there are increasing concerns about security and
privacy of the data that is being collected and/or shared (\cite{eahEA} and references therein). 
These concerns are both
on the part of the consumer, whose information is often used or traded
with little consent, and on the part of the collector, who is
often liable for protecting the collected data.

On the other hand, organisations are increasingly aware of the potential gain of combining their 
data assets, specifically in terms of increased statistical power for analytics and 
predictive tasks. For example, banks and insurance companies could collaborate 
for spotting fraudulent activities;
hospitals and medical facilities could leverage the medical history of common patients
in order to prevent chronic diseases and risks of future hospitalisation;
and online businesses could learn from purchase patterns of common users so 
as to improve recommendations. There is little work on systems that
can make use of distributed datasets for federated learning in a
sufficiently secure environment, and even less any formal analysis of
such a system, from the security and learning standpoint.

\begin{figure}[t]
\bignegspace
\begin{center}
\includegraphics[width=.80\textwidth]{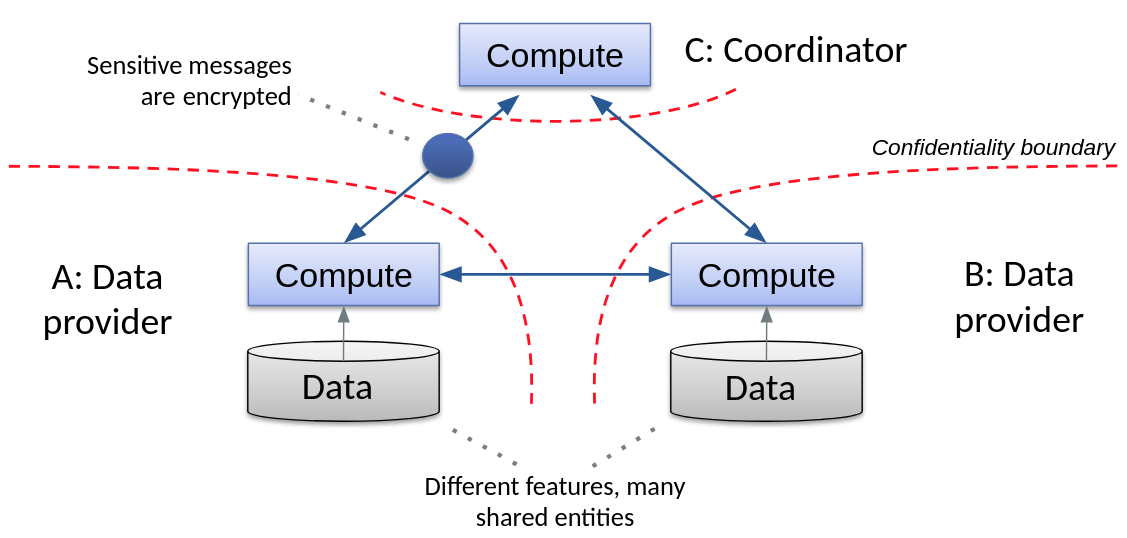}
\caption{Relationships between the Coordinator, \Coord, and the Data Providers, \AB.}
\label{fig:parties}
\end{center}
\bignegspace
\end{figure}

\noindent \textbf{Contributions} --- 
In this paper, we provide \textbf{both}:
\begin{itemize}
\item First, we propose an end-to-end solution for the case where
two organisations hold different data about (undisclosed) 
common entities. 
Two data providers, \FDPA~and~\LDPB, see only
their own side of a vertically partitioned dataset.
They aim to \emph{jointly} perform logistic regression in the cross-feature space.
Under the assumption that raw data \textit{cannot} be exchanged,
we present a secure protocol that is managed by a third party \C,
the coordinator, by employing privacy-preserving entity resolution and
an additively homomorphic encryption scheme. This is our first contribution.
Relationships between \FDPA, \LDPB\ and \Coord\ are presented in
Figure~\ref{fig:parties}. 

\item Our second contribution is the first formal study of an often overlooked source of errors of this
process: the impact of entity resolution errors on
learning. Since \FDPA~and~\LDPB~use different descriptive features, linking entities
 across databases is error prone \citep{christen12}. Intuitively we might expect such errors to have negative impact on
learning: for example, wrong matches of a hospital database with
pharmaceutical records with the objective to improve preventive
treatments could be disastrous. Case studies report that exact matching can be very
damaging when identifiers are not stable and error-prone: 25$\%$
true matches would have been missed by exact matching in a census operation \citep{sEP,wRL}.
We are not aware of results
quantifying the impact of error-prone entity-resolution on
learning. Such results would be highly desirable to (i) find the key components
of entity resolution errors that impact learning the most and then (ii)
devise improvements of
entity-resolution algorithms in a machine learning context. We provide in this paper four main contributions from that
standpoint. 
\begin{itemize}
\item First, we provide a formal bound on the deviation of the optimal classifier when such
errors occur. It shows for example that wrongly linking
examples of \textit{different classes} can be significantly more
damaging than wrongly linking
examples of the same class. 
\item Second, we show that under some reasonable
assumptions, the classifier learned is 
\textit{immune} to entity-resolution mistakes with respect to classification for \textit{large
  margin} examples. More precisely, examples that receive a large margin
classification by the optimal (unknown) classifier still receive the
same class by the classifier we learn from mistake-prone entity-resolved
data. 
\item Third, under the same assumptions, we bound the difference between the \textit{empirical loss} of our classifier on the \textit{true
data} (\textit{i.e.} built knowing the true entity-resolution) with respect to that of the optimal (unknown) classifier, and it
shows a convergence of both losses at a rate of order $1/n^\alpha$,
where $n$ is the number of examples and $\alpha \in (0,1]$ is an
assumption-dependent constant. The bound displays interesting dependencies on three
distinct penalties respectively depending on the optimal classifier,
entity resolution and a
sufficient statistics for the class in the true data.
\item Fourth, under the additional assumption that entity resolution
  mistakes are small enough in number, we show that not even rates for \textit{generalization} are notably
  affected by entity resolution. The same key penalties as for the
  empirical loss bounds appear to drive generalization bounds.
\end{itemize}
\end{itemize}
These contributions, we believe, represent \textit{very} strong advocacies for federated learning
when aggregating databases provides a significant uplift for
classification accuracy.\\

\noindent The rest of this paper is organised as follows. Section
\ref{sec:existingwork} presents related work. Section
\ref{sec:security-assumptions} presents the security environment and primitives. Section
\ref{sec:priv-pres-entity} details our approach for privacy-preserving
entity-resolution. Section \ref{sec:logist-regr-via} develops our
approach for secure logistic regression. Section \ref{sec:theory}
investigates the formal properties of learning in a federated learning
setting that relies on entity resolution. Section
\ref{sec:experiments} presents and discusses experiments. We provide a
conclusion in the last section. An appendix, starting page \pageref{sec:appendix},
provides all proofs and details on encryption, encoding, security
evaluation and cryptographic longterm keys.

\section{Related work}
\label{sec:existingwork}

The scenario is federated learning \citep{konecny16,
  mcmahan2016communication}, a distributed setting where data does not
leave its premises and data providers protect their privacy against a
central aggregator.  Our
interest is logistic regression on vertically partitioned data and,
importantly, we consider additional privacy requirements and entity
resolution with a given error rate.

Research in privacy-preserving machine learning is currently dominated
by the approach of differential privacy \citep{dwork2008differential, drTA}.
In the context of machine learning, this amounts to ensuring---with
high probability---that the output predictions are the same regardeless of
the presence or absence of an individual example in the dataset. This is
usually achieved by adding properly calibrated noise to an algorithm or to the data itself
\citep{chaudhuri2009privacy, duchi2013local}. While
computationally efficient, these techniques invariably degrade the
predictive performance of the model. 


We opt for security provided by
more traditional cryptographic guarantees such as homomorphic
encryption, \emph{e.g.}~Paillier \citep{paillier99} or Brakerski-Gentry-Vaikuntanathan
 cryptosystems \citep{fhe09}, secure multi-party computation, \emph{e.g.}~as garbled circuits
\citep{yao86} and secret sharing \citep{benor88}.
By employing additively homomorphic
encryption, we sit in this space. In contrast with differential privacy,
instead of sacrificing predictive power, we trade security for computational cost---the real 
expense of working with encryption.

Work in the area can be classified in terms of whether the data is vertically or
horizontally partitioned, the security framework, and the family of
learning models. We limit ourselves to mention previous work with
which we have common elements.
The overwhelming majority of previous work on secure distributed
learning considers a \textit{horizontal} data partition.  Solutions can take
advantage of the separability of loss functions which decompose the
loss by examples. Relevant approaches using
partially homomorphic encryption are
\cite{DBLP:journals/corr/XieWBB16,Aono:2016:SSL:2857705.2857731}. 

A vertical data partition requires a more complex and expensive
protocol, and therefore is less common.
 \cite{wu2013privacy} run logistic regression where one party holds the
features and the other holds the labels.
\cite{duverle2015spw} use
the Paillier encryption scheme to compute a variant of logistic
regression which produces a
$p$-value for each feature separately (rather than a logistic model as
we do here). Their partition is such that one party holds the private key, the
labels and a single categorical variable, while the other party holds
all of the features. \cite{gascon2017privacy} perform linear
regression on vertically partitioned data via a hybrid multi-party computation combining garbled
circuits with a tailored protocol for the inner product.
Recently, \cite{mohassel17} presented a system for privacy-preserving machine learning for linear 
regression, logistic regression and neural network training. They combine secret sharing, garbled 
circuits and oblivious transfer and rely on a setting with two un-trusted, but non-colluding servers.

\setlength{\tabcolsep}{0.5em}
\begin{table}[t]
\small
\centering
\begin{tabular}{r|a|c|c|c|c|c|c|c|c|}
\toprule 
& us & \mcrot{1}{c}{80}{\cite{gascon2017privacy}} & \mcrot{1}{c}{80}{\cite{duverle2015spw}} 
& \mcrot{1}{c}{80}{\cite{wu2013privacy}} & \mcrot{1}{c}{80}{\cite{eahEA}}
& \mcrot{1}{c}{80}{\cite{DBLP:journals/corr/XieWBB16}} & 
\mcrot{1}{c}{80}{\cite{Aono:2016:SSL:2857705.2857731}} & \mcrot{1}{c}{80}{\cite{konecny16}} \\\hline
Vertical partition & \tickYes & \tickYes &  \tickYes & \tickYes &
\tickNo & \tickNo & \tickNo & \tickNo\\\hline
Convergence  & \tickYes & \tickYes & \tickNo & \tickNo &
\tickYes & \tickYes & \tickNo & \tickYes\\\hline
ER & \tickYes & \tickNo & \tickNo & \tickNo &  &
 &  & \\\hline  \bottomrule
\end{tabular}
\caption{Comparison of related approaches on federated learning,
  (from top row to bottom row) whether they rely on vertically
  partitioned data, analyze convergence and/or entity resolution (ER); the presence of ER is justified only in case of vertical partition.}
\label{table:refs}
\end{table}

\textit{None} of the papers cited before consider the problem of
entity resolution (or entity matching, record linkage,
\cite{christen12}). For example, \cite{gascon2017privacy} assume that
the correspondence between rows in the datasets owned by different
parties is \textit{known a priori}. Such an assumption would not stand many
real-world applications where identifiers are not stable and/or
recorded with errors \citep{sEP,wRL}, making entity matching a prerequisite for working with 
vertically partitioned in most realistic scenarios.
Finding efficient and privacy-compliant algorithms is a field in itself, \textit{privacy-preserving entity resolution} \citep{hall2010privacy, christen12,
vatsalan13}. 

Table \ref{table:refs} summarizes some comparisons between our setting
and results to others. Perhaps surprisingly, there is an exception to that scheme: it was recently shown that entity
resolution is \textit{not necessary} for good approximations of the optimal model. 
Related methods exploit the fact that one can learn
from sufficient statistics for the class instead of examples \citep{npaRO, nockOR, patrini-thesis}, many of which do not 
require entity resolution to be computed, but rather a weak form of
entity resolution between groups of observations that share common subset of features \citep{pnhcFL}.  To our knowledge, \cite{pnhcFL} is also the only work other than
ours to study entity resolution and learning in a pipelined process,
although the privacy guarantees are different.
Crucially, it requires labels to be shared among all parties, which we
\textit{do not}, and also the theoretical guarantees are yet not as
comprehensive as the ones we are going to deliver.


\section{Security environment and primitives}
\label{sec:security-assumptions}

\noindent \textbf{Security model} --- We assume that the participants are \emph{honest-but-curious}:
(i)~they follow the protocol without tampering with it in any way;
(ii)~they do not collude with one another; but
(iii)~they will nevertheless try to infer as much as possible from the
  information received from the other participants.
The honest-but-curious assumption is reasonable in our context since
\AB~have an incentive to compute an accurate model.
The third party,~\C, holds the private key used for decryption;
however the only information~\C\ receives from \A\ and \B\ are
encrypted model updates, which we do not consider private in our setup.

We assume that \AB's data is secret, but that the schema (the number
of features and the type of each) of each data provider is available
to all parties.  We assume that the agents communicate on
pre-established secure channels.  We work under additional privacy
constraints:
\begin{enumerate}
\item Knowledge of common entities remain secret to \AB, as does
  the number of common entities.
\item No raw sensitive data  leaves \A~or \B~before encryption.
\end{enumerate}
On the other hand, a data provider can safely use its own
unencrypted records \emph{locally} anytime it is useful to do so.

\noindent \textbf{Additively homomorphic encryption} --- 
We recall here the main properties of
additively homomorphic encryption schemes such as \cite{paillier99}.
As a public key system, any party can encrypt their data with a known
\emph{public key} and perform computations with data encrypted by
others with the same public key. To extract the plaintext,
the result needs to be sent to the holder of the
\emph{private key}.

An additively homomorphic encryption scheme only provides arithmetic for elements of
its plaintext space. In order to support
algorithms over floating-point numbers, we must define an encoding
scheme that maps floats to modular integers and which preserves the
operations of addition and multiplication.
The encoding system we use is similar to floating-point
representation; a number is encoded as a pair consisting of an
encrypted significand and a unencrypted exponent. 
Details of the encoding scheme and
its limitations are given in Appendix \ref{app:encod-float-point}.

An additively homomorphic encryption scheme provides
an operation that produces the encryption of the sum of two numbers,
given only the encryptions of the numbers.  Let the encryption of a number
$u$ be $\Enc{u}$.  For simplicity we overload the
notation and we denote the operator with `$+$' as well.  For any
plaintexts $u$ and $v$ we have:
\begin{equation}
  \Enc{u} + \Enc{v} = \Enc{u + v}.
\end{equation}
Hence we can also multiply a ciphertext and a plaintext together by
repeated addition:
\begin{equation}
  \label{eq:2}
  v \cdot \Enc{u} = \Enc{vu},
\end{equation}
where $v$ is \emph{not} encrypted (it is not possible to multiply two
ciphertexts). In short, we can compute sums and products of plaintexts and
ciphertexts \emph{without leaving the space of encrypted numbers.}

These operations can be extended to work with vectors and matrices
component-wise. For example, we denote the inner product of two
vectors of plaintexts $\Vu$ and $\Vv$ by
$\Vv^\top~\Enc{\Vu} = \Enc{\Vv^\top \Vu}$
and the component-wise product by
$\Vv \circ \Enc{\Vu} = \Enc{\Vv \circ \Vu}$.  Summation and matrix
operations work similarly; see Appendix~\ref{app:part-homom-encrypt}
for details. Hence, using an additively homomorphic
encryption scheme we can implement useful linear algebra primitives 
for machine learning.

Doing arithmetic on encrypted numbers comes at a cost in memory and processing time.
For example, with Paillier encryption scheme, the encryption of an encoded
floating-point number (whether single or double precision) is $2m$ bits long,
where $m$ is typically at least 1024 \citep{keylength} and the addition
of two encrypted numbers is two to three orders
of magnitude slower than the unencrypted equivalent.  Nevertheless, as
we will see later, with a carefully engineered implementation of the
encryption scheme, a large proportion of real-world problems are
tractable.


\section{Privacy-preserving entity resolution}
\label{sec:priv-pres-entity}

When a dataset is vertically partitioned across organisations
the problem arises of how to identify corresponding entities.
The ideal solution would be joining datasets by common unique  IDs;
however, across organisations this is rarely feasible.
An approximate solution is given by techniques for \emph{entity resolution}~\citep{christen12}; Figure~\ref{fig:entity-resolution}
gives a pictorial representation.  Solving this problem is a requirement for
learning from the features of the two parties.

\begin{figure}[t]
\begin{center}
\includegraphics[width=.80\textwidth]{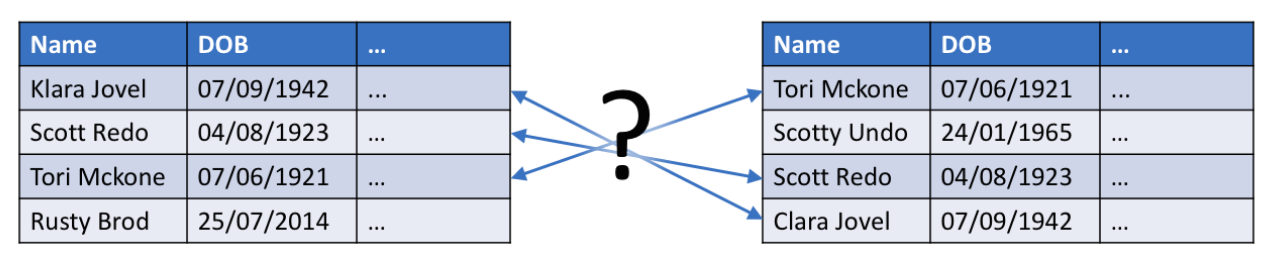}
\caption{The problem of entity resolution.}
\label{fig:entity-resolution}
\end{center}
\bignegspace
\end{figure}

In a non-privacy-preserving context, matching can be based on
shared \emph{personal identifiers}, such as name, address, gender, and date of birth,
that are effectively used as weak IDs.
In our scenario however, weak identifiers are considered private to each party.
Thus, in order to perform \emph{privacy-preserving} entity resolution we use an
anonymous linking code called a \emph{cryptographic longterm key} (CLK)
introduced by \cite{schnell11}.

The CLK is a Bloom filter encoding
of multiple personal identifiers.
The encoding method works by hashing n-gram sub-strings of selected
identifiers to bit positions in the Bloom filter.
A measure of similarity between CLKs is computed on the number of matching
bits by the \emph{Dice coefficient} \citep{schnell11}.\footnote{While
  CLKs are robust to typographical errors, they are susceptible to
  cryptanalytic attacks if insecure parameters are used, or if too few
  identifiers are hashed \citep{Kuzu2011, Kuzu2013, Niedermeyer2014}.}

\begin{algorithm}[t]
  \setstretch{1.5}
\caption{Privacy-preserving entity matching\label{alg:ER}}
\KwData{personal identifiers for entities in $\MX$ and $\MY$}
\negspace \KwResult{Permutations $\sigma, \tau$ and mask $\bm{m}$}
\negspace

\ThdDP{}{
\negspace	create a CLK for every entity and send them to \Coord\;
}

\negspace\ThdCoord{}{
\negspace	obtain $\sigma, \tau$ and $\bm{m}$ by matching CLKs\;
\negspace	encrypts $\Enc{\bm{m}}$\;
\negspace	sends $\sigma$ and $\tau$\ to \ADP\ respectively and $\Enc{\bm{m}}$ to both\;
}
\end{algorithm}    

The protocol for privacy-preserving entity matching is shown in Algorithm \ref{alg:ER}.
After CLKs of all entities from \ADP\ are received, \Coord\ matches them by
computing the Dice coefficient for all possible pairs of
CLKs, resulting in a number of comparisons equal to
the product of the datasets sizes. The most similar pairs are selected as matches, in a greedy fashion.
Faster computation is possible by \emph{blocking} \citep{vatsalan13b}.

The outputs of entity matching are two permutations $\sigma, \tau$ and a mask $\bm{m}$:
$\sigma, \tau$ describe how \AB~must
rearrange their rows so as to be consistent with each other;
$\bm{m}$ specifies whether a row corresponds to an entity available in both data providers, thus to
be used for learning. The encrypted mask and its integration into the process of
private learning is novel and part of our contribution.
No assumption was made on their relative size. For the learning phase though, they must be the same.
Hence the longer dataset is truncated, excluding only non-matched
entities.\footnote{As a consequence, the owner of the longer dataset will learn
that the truncated rows have no correspondence on the other dataset. This is a mild 
leak of information that does not violate the first security requirement in Section
 \ref{sec:security-assumptions}.}
 
More precisely, entity resolution of $\MX$ and $\MY$ relative
to CLKs consists of two permutations
$\sigma$ and $\tau$ of the rows of $\MX$ and $\MY$, and a mask 
$\Vm$ of length $n$ whose elements $i$ satisfy
\begin{equation}
  \label{eq:1}
  m_i = \begin{cases}
  1 & \text{if $\sigma(\MX^\textsc{CLK})_i \sim \tau(\MY^\textsc{CLK})_i$ , and}\\
  0 & \text{otherwise.}
  \end{cases}
\end{equation}
The operator `$\sim$' can be read as ``the most likely match'', in the sense of \cite{schnell11}.
Permutations and mask are randomized subject to the relation~\eqref{eq:1}, where
$m_i=0$ means that there is no record in the other dataset which could match with a 
high enough probability.

In our scenario, whether a record in a data provider is a match is
considered private; see requirement 1. in Section
\ref{sec:security-assumptions}. For example, when linking to a
medical dataset of patients, successful matching of an
entity could reveal that a person in an unrelated database suffers
the medical condition. In order to keep the mask confidential,
\Coord\ encrypts it with the Paillier scheme before sending it to
\AB.  Details on the use of mask are given in Section
\ref{sec:secure-stoch-grad}.

\section{Logistic regression, Taylor approximation}
\label{sec:logist-regr-via}

\begin{figure}
\centering
\begin{tabular}{cc}
\begin{tikzpicture}[rounded corners=2pt,scale=2.5]
\draw (1, -0.7) node {weights};
\draw (0, -0.25) rectangle node {$\Vk$} (2, -0.5);
\draw[very thin] (1.25, -0.25) rectangle node {$\theta_j$} (1.5, -0.5);

\draw (2.65, 1) node[rotate=270] {labels};
\draw (2.25, 0) rectangle node[above=10pt] {} (2.5, 2);
\draw (2.375, 1.5) node {$\Vy$};
\draw[very thin] (2.25, 0.75) rectangle node {$y_i$} (2.5, 1);

\draw (0, 0) rectangle (2, 2);
\draw (1, 2.25) node {features};
\draw (-0.4, 1) node[rotate=90] {observations};

\draw (-0.25, 2.15) node {$\Data$};
\draw (0.5, 2.1) node {$\DPFALL$};
\draw (0.5, 1.5) node {$\DPFH$};
\draw (0.5, 0.5) node {$\DPFS$};
\draw (1.5, 2.1) node {$\DPLALL$};
\draw (1.5, 1.5) node {$\DPLH$};
\draw (1.5, 0.5) node {$\DPLS$};
\draw[style=densely dotted,color=red] (1, -0.1) -- (1, 2.1);

\draw (-0.15, 1.5) node { $\Data_\holdoutset$ };
\draw (-0.15, 0.5) node { $\Data_\trainingset$ };
\draw[style=densely dotted,color=red] (-0.1, 1.2) -- (2.6, 1.2);

\draw[very thin] (0,1) rectangle node[left=15pt] {$\Vx_i$} (2, 0.75);
\draw[very thin] (1.25, 0.75) rectangle node {$x_{ij}$} (1.5, 1);
\end{tikzpicture} 
\end{tabular}
\caption{Overview of the notation for, and relationships between, the
  different variables in logistic regression.\label{fig:notation}}
\end{figure}
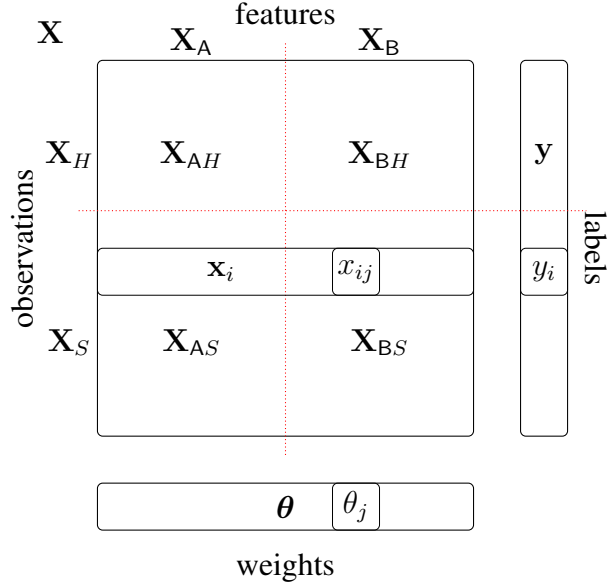

We need to adapt logistic regression and stochastic gradient descent
 to work with an additively homomorphic encryption scheme and the masks.
In this section the focus is on the
``non-distributed'' setting---all data is available in one place---, while Section \ref{sec:secure-stoch-grad} 
details the federated learning protocol.

With logistic regression we learn a linear model $\Vk \in \mathbb{R}^d$ that maps
examples $\Vx \in \mathbb{R}^d$ to a binary label $y \in \{-1, 1\}$. 
The learning sample $\trainingset$ is a set of $n$ example-label pairs $(\Vx_i, y_i)$
from $i=1,\dots, n$. Figure \ref{fig:notation} presents the notation
used in this section, showing in particular that the observation matrix is
used row-wise.
The average logistic loss computed on the training set is
\begin{equation}\label{eq:loglikelihood}
  \loglike_\trainingset(\Vk) = \frac{1}{n} \sum_{i\in\trainingset} \log(1 + e^{-y_i\dotkxi}).
\end{equation}
In turn, the stochastic gradients computed on a mini-batch ${\trainingsubset} \subseteq \trainingset$ of size $s'$ are
\begin{equation} \label{eq:logderiv}
  \nabla \loglike_{\trainingsubset}(\Vk)
  = \frac{1}{s'}\sum_{i \in {\trainingsubset}} \left(\frac{1}{1+e^{-y\dotkx}} - 1\right)y_i\Vx_i.
\end{equation}
Below we adapt equations~(\ref{eq:loglikelihood})
and~(\ref{eq:logderiv}) to accommodate the encrypted mask.
Learning requires the computation of gradients only, not of the loss itself.
Yet, to combat overfitting we monitor the loss function on a hold-out for early stopping, other than using ridge regularization. From now on, we denote $\ell_{\holdoutset}$ as the loss value on a hold-out $\holdoutset$ of size $h$.

\begin{figure}[t]
\centering
\includegraphics[width=.45\textwidth]{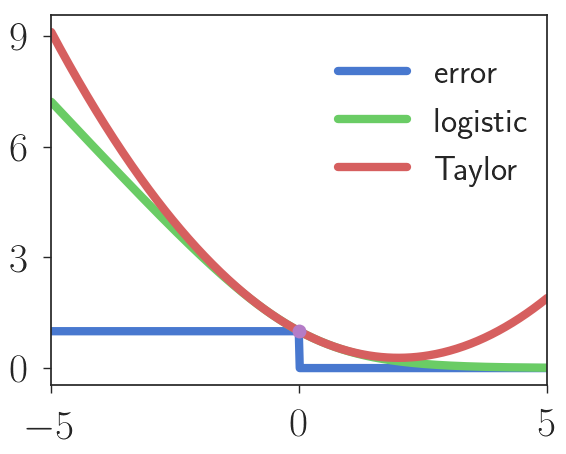}
\bignegspace
\caption{Loss profiles.\label{fig:losses}}
\negspace
\end{figure}

\noindent \textbf{Taylor loss} --- In order to operate under the constraints
imposed by an additively homomorphic encryption scheme, we need to consider approximations 
to the logistic loss and the gradient. 
To achieve this, we take a Taylor
series expansion of $\log(1 + e^{-z})$ around $z=0$:
\begin{equation}\label{calculus}
  \log(1 + e^{-z})
  = \log 2 - \tfrac{1}{2}z + \tfrac{1}{8}z^2 - \tfrac{1}{192}z^4 + O(z^6).
\end{equation}
The second order approximation of \eqref{eq:loglikelihood} evaluated on $\holdoutset$ is:
\begin{equation}~\label{eq:taylor-loss}
  \loglike_\holdoutset(\Vk)
  \approx \frac{1}{h}\sum_{i\in\holdoutset} \log 2
  - \frac{1}{2}y_i\dotprod{\Vk}{\Vx_i}
  + \frac{1}{8}(\dotprod{\Vk}{\Vx_i})^2,
\end{equation}
where we have used the fact that $y_i^2 = 1\:, \forall i$. 
We call this function the \emph{Taylor loss}. By differentiation,
we write the gradient for a mini-batch ${\trainingsubset}$ as:
\begin{equation}
  \begin{split}
    \nabla \loglike_{\trainingsubset}(\Vk)
    &\approx \frac{1}{s'}\sum_{i \in \trainingsubset}
      \left(\frac{1}{4}\dotkxi -\frac{1}{2}y_i\right) \Vx_i.
  \end{split}
\end{equation}

The second-order is a compromise between precision and computational overhead.
From (\ref{calculus}), the third-order approximation equals the second,
while fourth and fifth-orders are not fit for minimization since their images take negative 
values. Minimizing the sixth-order Taylor loss is costly when working in the encrypted space,
and some simple experiments are enough to show that the higher degree
terms do not provide significant performance gains.

Additionally, the expansion is around 0, which implies a rough
approximation of the logistic loss when $| \dotkxi | \gg 0$. 
Nevertheless, square losses are commonly used for classification \citep{nock2009bregman} and that is essentially what we have obtained.
Experiments in Section \ref{sec:experiments} show that features standardization suffices for 
good performance.
The loss is pictured in Figure \ref{fig:losses}.
The logistic loss is bounded above by the Taylor loss at every point,
so their values are not directly comparable and in
addition their \emph{minimizers} will differ.

\noindent \textbf{Applying the encrypted mask} ---
The encrypted mask can be incorporated into
\eqref{eq:loglikelihood} by multiplying each term by
$\Enc{m_i}$. Combined with the Taylor loss, the \emph{masked} gradient for a mini-batch $\trainingsubset$ is
\begin{equation}
  \label{eq:12}
  \begin{split}
    \Enc{\nabla \loglike_{\trainingsubset}(\Vk)} & \approx
    \frac{1}{s'} \sum_{i \in {\trainingsubset}} \Enc{m_i}
      \left(\frac{1}{4}\dotkxi -\frac{1}{2}y_i\right) \Vx_i, \\
  \end{split}
\end{equation}
and the \emph{masked} logistic loss on $\holdoutset$ is
\begin{equation}
  \label{eq:6}
  \Enc{\loglike_{\holdoutset}(\Vk)} \approx
  \Enc{\nu}  - \frac{1}{2}\dotprod{\Vk}{\Enc{\Vmu}} +
  \frac{1}{8h} \sum_{i\in\holdoutset} \Enc{m_i}
  (\dotprod{\Vk}{\Vx_i})^2,
\end{equation}
where $\Enc{\nu} = ((\log 2)/h) \sum_{i\in\holdoutset} \Enc{m_i}$
and $\Enc{\Vmu} = (1/h)\sum_{i\in\holdoutset} \Enc{m_i} y_i \Vx_i$.
The constant term $\Enc{\nu}$ is irrelevant for minimization since it
is model-independent; henceforth we set it to $0$.

\section{Secure federated logistic regression}
\label{sec:secure-stoch-grad}

We detail now the second phase of our pipeline which amounts 
to run the federated logistic regression with SGD.
We assume that the entity resolution protocol has been run, thus
\AB~have permuted their datasets accordingly, which now have the same number of rows $n$.
The complete dataset is a
matrix $\Data \in \mathbb{R}^{n \times d}$.
This matrix does not exist in one place but is composed of
the columns of the datasets $\DPF$ and $\DPL$ held respectively by
\FDP\ and \LDP; this gives the vertical partition:
\begin{eqnarray}
  \Data & = & \left[\!
  \begin{array}{c|c}
    \DPFALL & \DPLALL
  \end{array}
\!\right]\label{defSPLIT}
\end{eqnarray}
as seen in Figure~\ref{fig:notation}. \FDP\ also holds the label vector $\Vy$.
Let $\Vx$ be a row of $\Data$. Define
$\Vx_\DPF$ to be the restriction of $\Vx$ to the columns of $\DPF$,
and similarly for $\Vk$ in place of $\Vx$ and $\DPL$ in place of
$\DPF$. Then we can decompose $\dotkx$ as:
\begin{equation}
  \label{eq:10}
  \dotkx = \dotprod{\Vk_\DPF}{\Vx_\DPF} + \dotprod{\Vk_\DPL}{\Vx_\DPL}.
\end{equation}

\begin{algorithm}[t]
  \setstretch{1.5}
\SetKw{Break}{break}
\SetKwInOut{Input}{input}\SetKwInOut{Output}{output}
\caption{Secure logistic regression (run by \Coord) \label{alg:LogRegSGD}}
\KwData{Mask $\Vm$,
  learning rate $\eta$, regularisation $\Gamma$, hold-out size $h$, batch size $s'$}
\negspace \KwResult{Model $\Vk$}
\BlankLine
\negspace create an additively homomorphic encryption key pair\;
\negspace send the public key to \ADP\;
\negspace encrypt $\Vm$ with the public key, send $\Enc{\Vm}$ to \FDP\ and \LDP\;
\negspace run Algorithm~\ref{alg:MeanOperator}~($h$)\; \label{algo:logreg:mu}
\negspace $\Vk \leftarrow \Vzero, \loglike_{\holdoutset} \leftarrow \infty$\;

\Repeat{max iterations}{
\negspace   \For{every mini-batch $\trainingsubset$}{
\negspace     $\grad\loglike_{\trainingsubset}(\Vk) \leftarrow $
     Algorithm~\ref{alg:GradEval}~($\Vk, t$)\;\label{algo:logreg:grad}
\negspace     $\Vk \leftarrow \Vk - \learnrate \left( \grad\loglike_{\trainingsubset}(\Vk) + \Gamma\Vk \right)$;\label{line:updatek}
  }
\negspace   $\loglike_\holdoutset(\Vk) \leftarrow $ Algorithm~\ref{alg:LossEval}~($\Vk$)\;\label{algo:logreg:loss2}
\negspace   \lIf{$\loglike_\holdoutset(\Vk)$ has not decreased for a while}{\Break}
}
\negspace \Return{$\Vk$}\;
\end{algorithm}

Algorithm~\ref{alg:LogRegSGD} computes the secure logistic regression and it is 
executed by \Coord. It is a standard SGD where computations
involving raw data are replaced with their secure variants.
We describe computation of gradients in this section, 
while we defer to Appendix \ref{app:protocols} for the initialization and the loss evaluation for stopping criterion.

Prior to learning, \Coord\ generates a key pair for the chosen
cryptosystem and shares the public key with \AB.
Then, \Coord\ sends the encrypted mask $\Enc{\Vm}$ to \AB. This allows the implementation of a protocol where
\Coord\ is oblivious to the hold-out split and the mini-batch sampling of the training set.
Algorithm~\ref{alg:MeanOperator} initializes the protocol.
A common hold-out $\holdoutset$ is sampled by \AB. Additionally, they
compute and cache $\Enc{\bm{\mu}}$ in the logistic loss (\ref{eq:6}) computed on the hold-out.

Data has already been shuffled by the permutations result of entity matching (Section \ref{sec:priv-pres-entity}), hence we can access it by mini-batches sequentially.
Any stochastic gradient algorithm can be used for optimisation. We choose SAG \citep{schmidt13} for our experiments in Section \ref{sec:experiments}; in this case, 
\Coord\ keeps the memory of the previous gradients.

Early stopping is applied by monitoring the Taylor loss
on the hold-out $\holdoutset$ by Algorithm \ref{alg:LossEval} (Appendix \ref{sec:secure-logar-loss}). 
We prefer the computation of the loss over the hold-out error, which would be too costly under any additively homomorphic encryption scheme.

\begin{algorithm}[t]
  \setstretch{1.5}
\caption{Secure gradient\label{alg:GradEval}}
\KwData{Model $\Vk$, batch size $s'$}
\negspace \KwResult{$\grad\loglike_{\trainingsubset}(\Vk)$ of
  an (undisclosed) mini-batch $\trainingsubset$}

\negspace \ThdCoord{}{\negspace send $\Vk$ to \FDP\;}

\negspace \ThdFDP{}{\negspace 
  select the next batch $\trainingsubset \subset \trainingset$, $|\trainingsubset| = s'$\;
  \negspace $\Vu = \tfrac{1}{4} \DPFT \Vk_\DPF$\;
  \negspace $\Enc{\Vu'} = \Enc{\Vm}_{\trainingsubset}\circ (\Vu - \tfrac{1}{2}\Vy_{\trainingsubset})$\;
  \negspace send $\Vk$, $\trainingsubset$ and $\Enc{\Vu'}$ to \LDP\;
}

\negspace \ThdLDP{}{
  \negspace $\Vv = \tfrac{1}{4} \DPLT \Vk_\DPL$\;
   \negspace $\Enc{\Vw} = \Enc{\Vu'} + \Enc{\Vm}_{\trainingsubset}\circ \Vv$\; 
  \negspace $\Enc{\Vz} = \DPLT \Enc{\Vw}   $\;
  \negspace send $\Enc{\Vw}$ and $\Enc{\Vz}$ to \FDP\;
}
\negspace \ThdFDP{}{
  \negspace $\Enc{\Vz'} = \DPFT \Enc{\Vw}   $\;
  \negspace send $\Enc{\Vz'}$ and $\Enc{\Vz}$ to \Coord\;
}

\negspace \ThdCoord{}{
  \negspace obtain $\Enc{\grad\loglike_{\trainingsubset}(\Vk)}$ by concatenating $\Enc{\Vz'}$ and $\Enc{\Vz}$\;
  \negspace obtain $\grad\loglike_{\trainingsubset}(\Vk)$ by decrypting with the private key\;
}
\end{algorithm}

\noindent\textbf{Secure gradient} --- Algorithm~\ref{alg:GradEval} computes the secure gradient. It
is called for every batch in Algorithm~\ref{alg:LogRegSGD} and hence is the computational bottleneck.
We demonstrate the correctness of Algorithm~\ref{alg:GradEval} as follows. Fix an $i\in{\trainingsubset}$ and let
$\Vx = (\Vx_\DPF\mid\Vx_\DPL)$ be the $i$th row of $\Data$.  Note
that the $i$th component of $\Enc{\Vm_{\trainingsubset}\circ\Vu'}$ is
$\Enc{m_i(\tfrac{1}{4}\dotprod{\Vk_\DPF}{\Vx_\DPF} - \tfrac{1}{2}y_i)}$,
so the $i$th component of $\Enc{\Vw}$ is $\Enc{m_i(\tfrac{1}{4}\dotkx
  - \tfrac{1}{2}y_i)}$
by~\eqref{eq:10}. Then
\begin{eqnarray}
  \Enc{\Vz} = \DPLT \Enc{\Vw}  =
  \left[\EncB{\sum_{i\in\holdoutset}m_ix_{ij}(\tfrac{1}{4}\dotkxi -
      \tfrac{1}{2}y_i)}\right]_{j}
\end{eqnarray}
where $j$ ranges over the columns of $\DPL$ (notation $\left[u_j
\right]_j$ is a vector notation whose coordinates are defined by the
set of values $\{u_j\}_j$); similarly \FDP\
calculates the same for $j$ ranging over columns of $\DPF$ to obtain $\Enc{\Vz'}$.
\Coord\ concatenates these two and get $\Enc{\nabla\loglike_{\trainingsubset}(\Vk)}$
by~\eqref{eq:12}.
During Algorithm~\ref{alg:GradEval}, the only information sent in
clear is about the model $\Vk$ and the mini-batch $\trainingsubset$, both only shared between \ADP. 
All other messages are encrypted and \Coord\ only receives
$\grad\loglike_{\trainingsubset}(\Vk)$. Section
\ref{app:security-eval} in the Appendix provides an additional
security evaluation of our algorithms, including sources of potential
leakage of information.

\section{Theoretical assessment of the learning component}
\label{sec:theory}

As we work on encrypted data, the convergence rate of our
algorithm is an important point; since we consider entity resolution,
it is crucial to investigate the impact of its errors on learning.

\subsection{Convergence} 

The former question is in fact already answered: although computations are
performed in the encrypted domain, the underlying arithmetic is equivalent and thus has no influence on the optimization of the Taylor loss.
Our implementation of SAG is done on a
second-order Taylor loss which is ridge regularized, so we
have access to the strong convexity convergence \citep[Theorem
1]{schmidt13}. Let $S$ be a learning sample, assume we learn from
$S$ (not a holdout $H$ or subset) and let
\begin{eqnarray}
\loglike_{S}(\Vk; \gamma, \Gamma) & = & \loglike_{S}(\Vk) + \gamma
\ve{\theta}^\top \Gamma \ve{\theta}\label{ridgeregTaylor}
\end{eqnarray} 
denote the ridge regularized
Taylor loss, with $\gamma > 0$, matrix $\Gamma \succ
0$ (positive semi-definite) symmetric, then we can expect convergence rates for
$\loglike_{S}(\Vk; \gamma, \Gamma) \rightarrow \min_{\Vk'}
\loglike_{S}(\Vk'; \gamma, \Gamma)$ at a rate approaching $\rho^k$ for
some $0<\rho <1$, $k$ being the number of mini-batch updates in Algorithm
\ref{alg:LogRegSGD}.\\

\subsection{Impact of entity resolution: parameters} 

This leaves us with the second problem, that of entity resolution, and in
particular how wrong matches can affect $\min_{\Vk}
\loglike_{S}(\Vk; \gamma, \Gamma)$. To the
best of our knowledge, the state of the art on formal analyses of how
entity resolution affects learning is essentially a blank page, even in the vertical
partition setting where both parties
have access to the same set of
$n$ entities --- the case we study. 
We let
\begin{eqnarray*}
\hat{\Vk}^* & = & \arg\min_{\Vk}
\loglike_{S}(\Vk; \gamma, \Gamma)\:\:,\\
\Vk^* & = & \arg\min_{\Vk}
\loglike_{S^*}(\Vk; \gamma, \Gamma)\:\:.
\end{eqnarray*}
$S^*$ is the
ideal dataset among all shared features, that is, reconstructed
knowing the solution to entity resolution between \FDP~and \LDP. 
$S$ denotes our dataset produced via (mistake-prone) entity resolution. 
We
shall deliver a number of results on how $\hat{\Vk}^*$ and $\Vk^*$ are
related, but first focus in this section on defining and detailing the
parameters and assumptions that will be key to obtaining our results. 

\subsubsection{Modelling entity resolution mistakes}

In our setting, entity-resolution mistakes can be represented by
\textit{permutation} errors between \FDP~and \LDP. 
Precisely, there exists an \textit{unknown} permutation matrix, $\PERM_*: [n]
\rightarrow [n]$, such that instead of learning from the ideal $\Data$ as in
\eqref{defSPLIT}, we learn from some
\begin{eqnarray}
\hat{\Data} & = & [\DPFALL |
(\DPLALL^\top \PERM_*)^\top]\:\:.
\end{eqnarray} 
Without loss of
generality, we assume
that indices refer to columns in \FDP~and so permutation errors impact
the indices in \LDP. We
recall that \FDP~holds the labels as well. Several parameters and
assumptions will be key to our results.
One such key parameter is the \textit{size} $T$ of $\PERM_*$ when
factored as 
\textit{elementary} permutations, 
\begin{eqnarray}
\PERM_* & = & \prod_{j=1}^T \PERM_{j} \label{decompPSTAR}
\end{eqnarray} 
($T$ unknown), where
$\PERM_{t}$ (unknown)
acts on some index $\ua{t}, \va{t}\in [n]$ in \FDP. Such a
factorization always holds, and it is not hard to see that there
always exist a factorization with $T\leq n$. 

Another
key parameter is the number $T_+ \leq T$ of \textit{class mismatch}
permutations in the factorization, \textit{i.e.} for which $y_{\ua{t}}
\neq y_{\va{t}}$. We let
\begin{eqnarray}
\rho & \defeq & \frac{T_+}{T}\:\:\label{defrho}
\end{eqnarray}
define the proportion of elementary permutations that act between
classes.

We let $\ub{t}$
(resp. $\vb{t}$) denote the indices in $[n]$ of the rows in
$\DPLALL$ that are in observation $\ua{t}$ (resp. $\va{t}$) and that will
be permuted by $\PERM_t$. For example, if $\ub{t} = \va{t}, \vb{t} = \ua{t}$, then
$\PERM_{t}$ correctly reconstructs observations $\ua{t}$ and
$\va{t}$. 

\subsubsection{Assumptions on permutations and data}

We now proceed through our assumptions, that are covered in greater
detail in the appendix, Section \ref{app:assum}. We make two
categories of assumptions:
\begin{itemize}
\item $\PERM_*$ is bounded in \textit{magnitude} and \textit{size},
\item the data and learning problem parameters are accurately \textit{calibrated}.
\end{itemize}
\paragraph{Bounding $\PERM_*$ in terms of magnitude ---} This is what we
define as $(\epsilon,
\tau)$-\textit{accuracy}. Denote $\hat{\Vx}_{ti}$ as row $i$ in $\hat{\Data}_t$, in
which $\DPLALL$ is altered by the \textit{subsequence} $\prod_{j=1}^t \PERM_j$, and
$\Vx_{i}$ as row $i$ in $\Data$. 
\begin{definition}\label{defACCURATE}
We say that $\PERM_t$ is $(\epsilon, \tau)$-\textit{accurate}
for some $\epsilon, \tau\geq 0, \epsilon \leq 1$ iff for any $\ve{w}\in \mathbb{R}^d$,
\begin{eqnarray}
|(\hat{\Vx}_{ti} - \Vx_{i})_\shuffle^\top \ve{w}_\shuffle| &
\leq & \epsilon \cdot |\Vx_{i}^\top \ve{w}| + \tau \|\ve{w}\|_2 \:\:, \forall i \in [n]\:\:, \label{defACCURATE1}\\
|(\Vx_{\uf{t}} -
\Vx_{\vf{t}})_{\F}^\top\ve{w}_{\F}| & \leq & \epsilon \cdot \max_{i\in
  \{\uf{t}, \vf{t}\}} |\Vx_{i}^\top \ve{w}|+
\tau \|\ve{w}\|_2\:\:, \forall \F \in \{\anchor, \shuffle\}\:\:.\label{defACCURATE2}
\end{eqnarray}
We say that $\PERM_*$ is $(\epsilon, \tau)$-accurate iff each $\PERM_t$ is
$(\epsilon, \tau)$-accurate, $\forall t = 1, 2, ..., T$. 
\end{definition}
If we consider that vectors $\hat{\Vx}_{ti} - \Vx_{i}, \Vx_{\uf{t}} -
\Vx_{\vf{t}}$ quantify errors made by elementary permutation
$\PERM_t$, then $(\epsilon, \tau)$-accuracy postulates that errors along
any direction are bounded by a fraction of the norm of original observations,
plus a penalty that depends on the direction. We remark that in the
context of the inequalities,
$\tau$ is homogeneous to a norm, which is not the case for
$\epsilon$ (which can be thought ``unit-free''). For that reason, we define an important quantity that we
shall use repeatedly, aggregating $\epsilon$ and a ``unit-free'' $\tau$:
\begin{eqnarray}
\xi & \defeq & \epsilon + \frac{\tau}{X_*}\:\:,\label{defXI}
\end{eqnarray}
where $X_* \defeq
  \max_i \|\Vx_i\|_2$ is the max norm in (the columns of)
  $\X$. Section \ref{app:assum} in the appendix gives more context
  around Definition \ref{defACCURATE}. 
\paragraph{Bounding $\PERM_*$ in terms of size ---} The
$\alpha$-\textit{boundedness} condition states that the decomposition
in eq. (\ref{decompPSTAR}) has a number of terms limited as a function
of $n$.
\begin{definition}
We say that $\PERM_*$ is \textbf{$\alpha$-bounded} 
for some $0< \alpha \leq 1$
iff its size satisfies
\begin{eqnarray}
T & \leq & \left(\frac{n}{\xi}\right)^{\frac{1-\alpha}{2}}\:\:.\label{sizeT}
\end{eqnarray}
\end{definition}
The bounded permutation size assumption roughly
means that $T = o(\sqrt{n})$ in the worst case. ``Worst case'' means
that in favorable cases where we can fix $\xi$ small,
the assumption may be automatically verified even for $\alpha$ very
close to 1 since it always holds that a permutation can be decomposed
in elementary permutations with $T\leq n$. Note that to achieve a
particular level of $(\epsilon,
\tau)$-accuracy assumption, we may need more than the minimal
factorisation, but it is more than reasonable to assume that we shall
still have $T = O(n)$, which does not change the picture of the
constraint imposed by $\alpha$-boundedness.

\paragraph{Data and model calibration ---} We denote $\sigma(\mathcal{S})$
as the standard deviation of a set $\mathcal{S}$, and
$\lambda_1^\uparrow(\Gamma)$ the smallest eigenvalue of $\Gamma$,
following \cite{bMA}. Finally, we define the stretch of a vector.
\begin{definition}\label{defstretch}
The \textbf{stretch} of vector $\Vx$ along direction $\ve{w}$ with
$\|\ve{w}\|_2 = 1$ is 
\begin{eqnarray}
\vstretch(\Vx,\ve{w}) & \defeq & \|\Vx\|_2
|\cos(\Vx, \ve{w})|\:\:.
\end{eqnarray}
\end{definition}
The stretch of $\Vx$ is just the norm of the orthogonal projection along direction $\ve{w}$. 
\begin{definition}\label{def:DMC}
We say that the \textbf{data-model calibration} assumption holds iff
the following two constraints are satisfied:
\begin{itemize}
\item Maxnorm-variance regularization: ridge regularization parameters $\gamma, \Gamma$
  are chosen so that
\begin{eqnarray}
\frac{X_*^2}{ \frac{(1-\epsilon)^2}{8}
  \cdot \inf_{\ve{w}}\sigma^2(\{\vstretch(\Vx_i,\ve{w})\}_{i=1}^n) + \gamma
  \lambda_1^\uparrow(\Gamma)} & \leq & 1\:\:.\label{ineqcontr}
\end{eqnarray}
\item Minimal data size: the dataset is not too small, namely,
\begin{eqnarray}
n & \geq & 4 \xi\:\:.\label{ineqcontr2}
\end{eqnarray}
\end{itemize}
\end{definition}
Remark that the data size roughly means that $n$ is larger than a
small constant, and the maxnorm-variance regularization means that the
regularization parameters have to be homogeneous to a squared
norm. Alternatively, having fixed the regularization parameters, we
just need to recalibrate data by normalizing observations to control
$X_*$. 

\paragraph{Key parameters ---} We are now ready to deliver a series of results on various
relationships between $\hat{\Vk}^*$ and $\Vk^*$: deviations between
the classifiers, their empirical Ridge-regularized Taylor losses,
generalization abilities, etc. . Remarkably, all results depend on
three distinct key parameters that we now define.
\begin{definition}
We define $\deltamargin, \delta_\rho, \delta_{\mu}$ as follows:
\begin{itemize}
\item [] $\deltamargin \defeq \|\Vk^* \|_2 X_*$, a bound on the maximum
margin for the optimal (unknown) classifier\footnote{Follows from
  Cauchy-Schwartz inequality.};
\item [] $\delta_\rho \defeq
\sqrt{\xi} \rho / 4$, aggregates parameters $\epsilon, \tau, \rho$ of
$\PERM_*$;
\item [] $\delta_{\mu} \defeq \|\sum_i y_i\Vx_i\|_2/(nX_*)$
($\in [0,1]$), the normalized mean-operator\footnote{A sufficient statistics
for the class \citep{pnrcAN}.}.
\end{itemize}
\end{definition}
Notice that $\deltamargin$ depends on the optimal classifier $\Vk^*$,
$\delta_\rho$ depends on the permutation matrix $\PERM_*$, while
$\delta_{\mu}$ depends on the true data $S^*$.

\subsection{Relative error bounds for  $\hat{\Vk}^*$ vs $\Vk^*$} 

Our analysis gives the conditions on the unknown
permutation, the learning problem and the data, for the following inequality to
hold:
\begin{eqnarray}
\|\hat{\Vk}^* - \Vk^* \|_2 & \leq & \frac{a}{n} \cdot
\|\Vk^* \|_2 + \frac{b}{n} \:\:,
\end{eqnarray}
where $a$ and $b$ are functions of relevant parameters. In other
words, we display a parameter regime in which a reasonably robust entity
resolution will \textit{not} significantly affect the optimal
classifier. We assume without loss of generality that
$\|\Vk^* \|_2 \neq 0$, a property guaranteed to hold if the
mean operator is not the null
vector \citep{pnrcAN}.
\begin{theorem}\label{thAPPROX1}
Suppose $\PERM_*$ is $(\epsilon, \tau)$-accurate and the data-model calibration assumption holds. Then the following holds:
\begin{eqnarray}
\frac{\|\hat{\Vk}^* - \Vk^* \|_2}{\|\Vk^* \|_2} & \leq & \frac{\xi}{n} \cdot T^2 \cdot \left( 1+
\frac{\sqrt{\xi}}{4 \|\Vk^* \|_2 X_*} \cdot \rho\right) \:\:.\label{eqthAPPROX00}
\end{eqnarray}
If, furthermore, $\PERM_*$ is $\alpha$-bounded, then we get
\begin{eqnarray}
\frac{\|\hat{\Vk}^* - \Vk^* \|_2}{\|\Vk^* \|_2} & \leq & C(n) \cdot \left(
  1 + \frac{\delta_\rho}{\deltamargin}
\right)\:\:,\label{eqthAPPROX01}
\end{eqnarray}
with $C(n) \defeq (\xi / n)^{\alpha}$.
\end{theorem}
Function $C(n)$ is going to have key roles in the results to
follow. It will be useful to keep in mind that, as long as we can
sample and link more data by keeping a finite upperbound on $\xi$, we
shall observe:
\begin{eqnarray}
C(n) & \rightarrow_{+\infty} & 0 \:\:, \forall \alpha \in (0, 1]\:\:.
\end{eqnarray}
The proof of Theorem \ref{thAPPROX1} is quite long and detailed in Sections
\ref{app:notations} -- \ref{app:proof-thAPPROX1} of the appendix. It uses a helper theorem
which we think is of independent interest for the study of entity
resolution in a learning setting. Informally, it is an exact expression for $\hat{\Vk}^* - \Vk^*$ which
holds regardless of $\PERM_*$. We anticipate that this result may be
interesting to optimize entity resolution in the context of
learning. Theorem \ref{thAPPROX1} essentially shows that
$\|\hat{\Vk}^* - \Vk^* \|_2/\|\Vk^* \|_2 \rightarrow_n 0$ and is all
the faster to converge as "classification gain beats permutation penalty"
($\deltamargin / \delta_\rho$ large). Ultimately, we see that learning can withstand bigger
permutations ($\alpha \searrow$), provided it is accompanied by a sufficient decrease in the proportion of of elementary permutations that act between
  classes ($\rho \searrow$): efficient entity resolution algorithms for federated
  learning should thus focus on trade-offs of this kind. 

The key problem that remains is how does the drift
$\hat{\Vk}^* - \Vk^*$ impact learning. For that objective, we start with two results with
different flavours: first, an \textit{immunity} of optimal large margin
classification to the errors of entity resolution, and second, a bound
for the difference of the Taylor losses of the optimal classifier
($\Vk^* $) and our classifier $\hat{\Vk}^*$ \textit{on the true data},
which shows strong convergence properties for $\loglike_{S^*}(\hat{\Vk}^*;
\gamma, \Gamma)$ towards $\loglike_{S^*}(\Vk^*; \gamma, \Gamma)$.

\subsection{Immunity of optimal large margin classification to $\PERM_*$} 
We 
show that under the conditions of Theorem \ref{thAPPROX1}, large
margin classification by the optimal classifier ($\Vk^* $)
is immune to the impact of $\PERM_*$ on learning, in the sense that
the related examples will be given
the \textit{same} class by $\hat{\Vk}^*$  --- the corresponding
margin, however, may vary. We formalize the definition now.
\begin{definition}
Fix $\kappa > 0$. We say that $\hat{\Vk}^*$ is
\textbf{immune to $\PERM_*$ at margin $\kappa$} iff for any example $(\Vx,
y)$, if $y (\Vk^*)^\top \Vx > \kappa$, then $y (\hat{\Vk}^*)^\top \Vx > 0$.
\end{definition}
We can now formalize the immunity property.
\begin{theorem}\label{thIMMUNE}
Suppose $\PERM_*$ is $(\epsilon, \tau)$-accurate and
$\alpha$-bounded, and the data-model calibration assumption
holds. For any $\kappa > 0$, $\hat{\Vk}^*$ is immune to
$\PERM_*$ at margin $\kappa$ if
\begin{eqnarray}
 n & > &  \xi \cdot \left(
  \frac{\deltamargin + \delta_\rho}{\kappa} \right)^{\frac{1}{\alpha}}\:\:.\label{eqNXIKAPPA}
\end{eqnarray}
\end{theorem}
(proof in Appendix, Section \ref{app:proof-thIMMUNE})
Eq. (\ref{eqNXIKAPPA}) is interesting for the relationships between
$n$ (data), $\xi$ (permutation) and $\kappa$ (margin) to achieve
immunity. Consider a permutation $\PERM_*$ for which $\rho = 0$, that
is, $\PERM_*$ factorizes as cycles that act within one class
each. Since the maximal optimal margin within $\mathcal{S}$ is bounded by $\deltamargin$ by Cauchy-Schwartz
inequality, so we
can let $\kappa \defeq \delta \cdot \deltamargin$ where
$0<\delta<1$ is the margin immunity parameter. In this case,
(\ref{eqNXIKAPPA}) simplifies to the following constraint on
$\delta$ to satisfy immunity at margin $\kappa$:
\begin{eqnarray}
 \delta & > &  C(n)\:\:, \label{eqNXIKAPPA2}
\end{eqnarray}
where $C(n)$ is defined in Theorem \ref{thAPPROX1}.
Roughly, increasing the domain size ($n$) without degrading the
effects of $\PERM_*$ ($\xi$) can allow to significantly reduce the minimal immune
margin.

\subsection{Bound on the difference of Taylor losses for $\Vk^*
  $ vs $\hat{\Vk}^*$ on the true data} 
We essentially show that under the
assumptions of Theorem \ref{thIMMUNE}, it holds that
$\loglike_{S^*}(\hat{\Vk}^*; \gamma, \Gamma) - \loglike_{S^*}(\Vk^*; \gamma, \Gamma) = o(1)$, and the convergence is governed by
  $C(n)$. By means of words, the loss of our classifier (built over
  entity-resolved data) on the true
  data converges to that of the optimal classifier at a rate
  roughly proportional to $1/n^\alpha$.
\begin{theorem}\label{thDIFFLOSS}
Suppose $\PERM_*$ is $(\epsilon, \tau)$-accurate and
$\alpha$-bounded, and the data-model calibration assumption
holds. Then it holds that:
\begin{eqnarray}
\loglike_{S^*}(\hat{\Vk}^*; \gamma, \Gamma) - \loglike_{S^*}(\Vk^*; \gamma, \Gamma) & \leq & \deltamarginrho \left( \delta_{\mu} + 
6 \deltamarginrho\right)\cdot C(n)\:\:,\label{eqDIFFLOSS1}
\end{eqnarray}
where $\deltamarginrho \defeq (\deltamargin + \delta_\rho)/2$ and $C(n)$ is defined in Theorem \ref{thAPPROX1}.
\end{theorem}
(proof in Appendix, Section \ref{app:proof-thDIFFLOSS}) We remark the
difference with Theorem \ref{thIMMUNE} that the bound also depends on a
normalized sufficient statistics for the class of the true data
($\delta_{\mu}$). 

\subsection{Generalization abilities for classifiers learned
  from E/R'ed data} 
Suppose that ideal sample $S^*$ is obtained
i.i.d. from some unknown distribution $\mathcal{D}$, before it is
"split" between \FDP~and \LDP, and then reconstructed to form our
training sample $S$. What is the
generalization ability of classifier $\hat{\ve{\theta}}^*$ ? This question
is non-trivial because it entails the impact of entity
resolution (E/R) on generalization, and not just on training, that is, we
want to upperbound $\Pr_{(\Vx, y)\sim \mathcal{D}} [y (\hat{\Vk}^*)^\top \Vx
\leq 0 ]$ with high probability given that the data we have
access to may not exactly reflect sampling from $\mathcal{D}$. The
following Theorem provides such a bound.
\begin{theorem}\label{thGENTHETA}
With probability at least $1-\delta$ over the sampling of $S^*$
according to $\mathcal{D}^n$, if the permutation $\PERM_*$ that links $S$ and $S^*$ is $(\epsilon, \tau)$-accurate and
$\alpha$-bounded and the data-model calibration assumption
holds, then it holds that
\begin{eqnarray}
\Pr_{(\Vx, y)\sim \mathcal{D}} \left[y (\hat{\Vk}^*)^\top \Vx
\leq 0 \right] & \leq & \loglike_{S^*}(\Vk^*; \gamma, \Gamma) + \frac{2L X_{*} \theta_*}{\sqrt{n}} +
  \sqrt{\frac{\ln(2/\delta)}{2n}} + U(n)\:\:,\label{eqGENER1}
\end{eqnarray}
with $L$ the Lipschitz constant of the Ridge-regularized Taylor loss,
$\theta_*$ an upperbound on $\|\Vk^*\|_2$
and 
\begin{eqnarray}
U(n) & \defeq & \deltamarginrho \cdot \left( \delta_{\mu_0} + 
6 \deltamarginrho + (4L/\sqrt{n}) \right) \cdot
C(n)\:\:.
\end{eqnarray} 
$\deltamarginrho$ is defined in Theorem \ref{thDIFFLOSS} and $C(n)$ is defined in Theorem \ref{thAPPROX1}.
\end{theorem}
(proof in Appendix, Section \ref{app:proof-thGENTHETA}) What is quite
remarkable with that property is that we immediately have with probability at least $1-\delta$, from \cite{bmRA}:
\begin{eqnarray}
\Pr_{(\Vx, y)\sim \mathcal{D}} \left[y (\Vk^*)^\top \Vx
\leq 0 \right] & \leq & \loglike_{S^*}(\Vk^*; \gamma, \Gamma) + \frac{2L X_{*} \theta_*}{\sqrt{n}} +
  \sqrt{\frac{\ln(2/\delta)}{2n}} \:\:,
\end{eqnarray}
\textit{i.e.} the corresponding inequality in
which we remove $U(n)$ in eq. (\ref{eqGENER1}) actually holds for substituting $\hat{\Vk}^*$
by $\Vk^*$ in the left hand side. In short, entity resolution
affects generalization \textit{only} through penalty $U(n)$. In
addition, if $\PERM_*$ is "small" enough so that $\alpha \geq 1/2$,
then we keep the slow-rate convergence ($O(1/\sqrt{n})$) of the E/R-free
case. Otherwise, entity-resolution may impact generalization by slowing
down this convergence.

\renewcommand\thesubfigure{(\alph{subfigure})}

\begin{figure*}[t]
\centering
\subfigure[loss convergence]{
	\includegraphics[width=.23\textwidth]{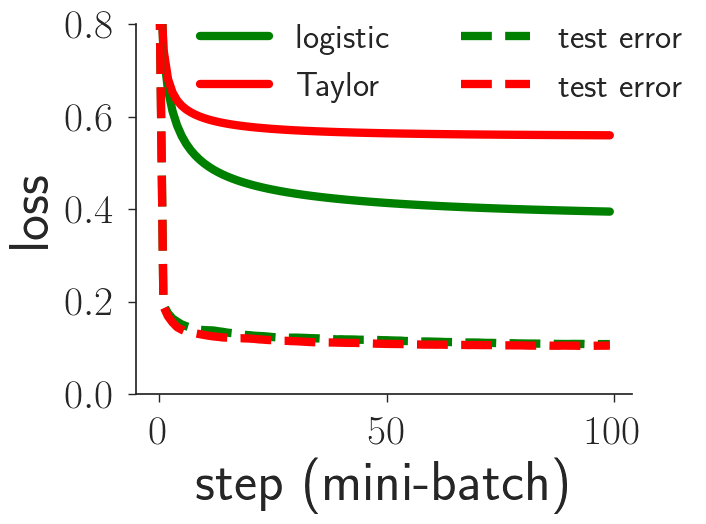}
	\label{fig:last0}
	}
\subfigure[runtime for entity matching]{
	\includegraphics[width=.24\textwidth]{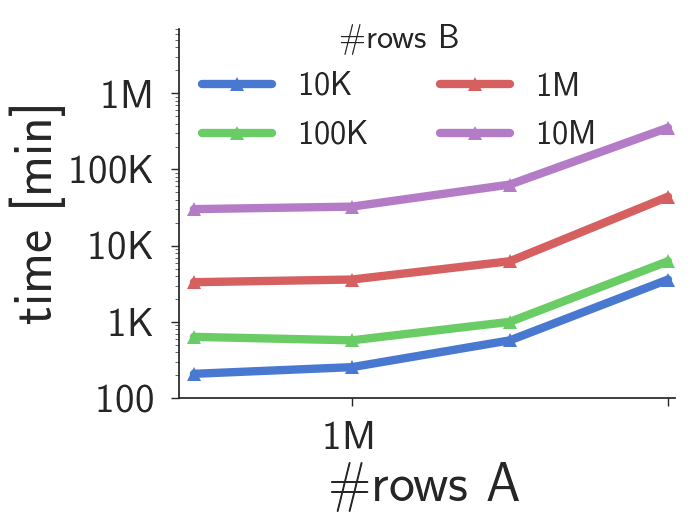}
	\label{fig:last1}
	}
\subfigure[runtime for a learning epoch]{
	\includegraphics[width=.23\textwidth]{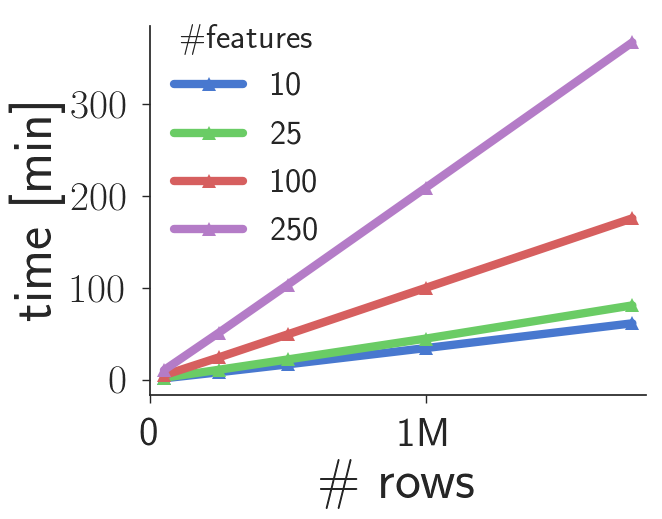}
	\label{fig:last2}
	}
\subfigure[runtime for a learning epoch]{
	\includegraphics[width=.23\textwidth]{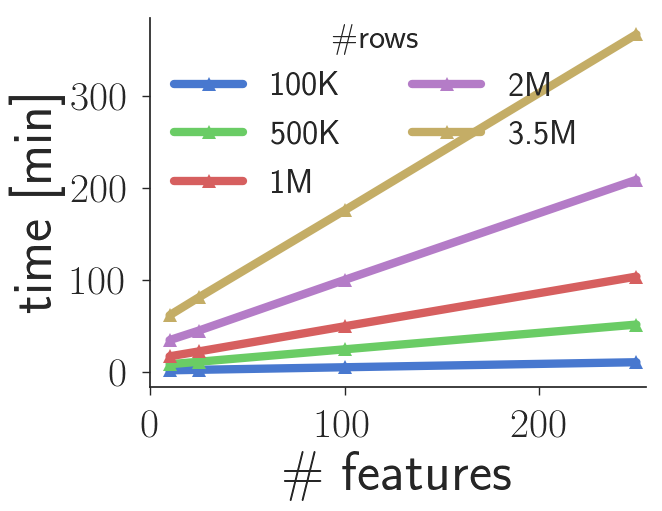}
	\label{fig:last3}
	}
\caption{\protect\subref{fig:last0}  Learning curve for Taylor vs. logistic loss (straight lines) and their test error (dotted);
	     \protect\subref{fig:last1} runtime of entity matching with respect to the size of the two datasets;
	     runtime of one learning epoch (all mini-batches + hold-out loss evaluation) with respect to number of examples \protect\subref{fig:last2} and features
	     \protect\subref{fig:last3}.
	     }
\end{figure*}

\section{Experiments}
\label{sec:experiments}

We test our system experimentally and show that: 
i) the Taylor loss converges at a similar rate to the logistic loss in practice;
ii) their accuracy is comparable at test time;
iii) privacy-preserving entity resolution scales to the tens of millions of rows per data provider in a matter of hours;
iv) federated logistic regression scales to millions of rows in the order of hours per epoch;
v) the end-to-end system shows learning performance \emph{on par} with learning on the perfectly linked datasets, the non-private setting.

\textbf{Taylor vs. logistic} ---
We have sketched the convergence rate of the Taylor loss in Section
\ref{sec:theory}. Here we are interested in its experimental behaviour; see Figure \ref{fig:last0}. We run vanilla SGD on MNIST (odd vs. even digits) with constant learning rate $\eta = 0.05$ and $\Gamma = 10^{-2} I$ ridge regularization. The Taylor loss shows a very similar rate of convergence to the logistic loss. In both cases, features are standardized.
Notice that the two losses converge to different minima \emph{and} minimizers. What is important is to verify that at test time the two minimizers result in similar accuracy, which is indeed the case in Figure~\ref{fig:last0}.
In Table \ref{table:1} we show that this is no coincidence. For several UCI datasets \citep{UCI}, we standardize features and run SGD with the same hyper-parameters as above. Quantitative performance (accuracy and AUC) of the two losses differ by at most 1.8 points.

\setlength{\tabcolsep}{0.5em}
\begin{table}[t]
\small
\centering
\begin{tabular}{|l|r|r|cc|cc|}
\toprule 
&&& \multicolumn{2}{c|}{logistic} & \multicolumn{2}{c|}{Taylor (delta)} \\
dataset               & \#rows & \#feat. & acc & AUC & acc & AUC \\
\midrule
\textit{iris}                  & 100    & 3          & 98.0 &100                                                        & +0.0 & +0.0 \\

\textit{diabetes} & 400 & 10 & 69.0 & 79.3 & +0.0 & +0.5 \\
\textit{bostonhousing} & 400 & 13 &  76.4 & 98.2 &  +0.0 & +0.1 \\
\textit{breastcancer} & 400 & 30 &  94.7 & 98.4 &  +1.7 & +0.3 \\
\textit{digits} & 1.5K  & 64         & 92.3 & 97.0                                                     & +1.0 & +0.4                                                     \\
\textit{simulated} & 8K & 100 & 68.8 & 76.0 & -0.1 & +0.0 \\
\textit{mnist}  & 60K    & 784        & 89.2 & 95.8                                                     &  +0.1 & -0.1                                                           \\
\textit{20newsgroups} & 10K & 5K & 51.7 & 66.8 & -0.2 & -1.8 \\
\textit{covertype}               & 500K   & 54         & 72.2 & 79.1                                                       & -0.1 & +0.1 \\
\hline  \bottomrule
\end{tabular}
\caption{Taylor vs. logistic loss on UCI datasets. When the original problem is multi-class or regression, we cast it into binary classification.}
\label{table:1}
\end{table}

\textbf{Scalability} ---
For the rest of this Section, empirical results will be based on a
``quasi-real scenario'', an augmented version of \cite{kaggle}'s \emph{Give me some credit} dataset. The original problem involves classifying whether a customer will default on a loan. We call this benchmark dataset PACS (PI Augmented Credit Scoring): realistic PIs, potentially missing and with typos, are generated for the entities by Febrl \citep{febrl}. The data is then split vertically in two halves of the 12 original features. The size of the dataset is about 160K examples.

In order to test the scalability of each component of our pipeline, we upsample the dataset and augment the feature sets by feature products, and record the runtime of our implementations of entity resolution and logistic regression solve by SAG.
For these experiments we run each party of the protocol on a separate AWS machine with 16 cores and 60GiB of memory. We use the Paillier encryption scheme as our additively homomorphic scheme with a key size of 1024 bits, from the implementation of \cite{phe}.

First, for benchmarking the entity resolution algorithm, we make a grid of datasets with sizes spanning on a log-scale from $1K$ to $10M$.
The entity matching only acts on the PI attributes. In Figure \ref{fig:last1} we measure the runtime of the combined construction of CLKs, network
communication to \Coord\ and matching two datasets. The coordinating party uses a cluster of 8 machines for maching.
As expected, entity resolution scales quadratically in the number of rows, due to the number of comparisons that are needed for brute-force matching.
Due to the embarrassingly parallel workload the \Coord\ component can handle matches with datasets up to $10M$ rows in size in a matter of hours.

Second, we measured the runtime of learning. We generated multiple versions of the dataset varying the number of rows and features;
while rows are simply copied multiple times, we create new features by multiplying the original ones. We are not interested in convergence
time here but only on a scalability measure on an epoch basis; the recorded runtime accounts for updating the model for all mini-batches and one
hold-out loss computation. See Figures \ref{fig:last2} and \ref{fig:last3}. Runtime grows linearly in both scales. This is the bottleneck of the
whole system, due to the expensive encrypted operations and the communication of large encrypted numbers between the parties. We estimated that
encryption amounts to a slow down of about two orders of magnitude.

The communication costs of one epoch consists of
\begin{eqnarray*} 
\text{cost}_{\grad} &\leq& (2n + 2 \lceil n/s' \rceil d) ct\\
\text{cost}_{\ell} & = & (h_s + 2) ct
\end{eqnarray*}
where $ct$ describes the size of one cipher text. Thus, for $s' = d = 100$, $n = 1$ M and a cipher text size of $256$ bytes, 
the overall communication costs for one epoch a just under 1 GB.

\textbf{Benchmark} ---
Finally, we test the whole system on three version of PACS, where we
control the percent of shared entities between \AB, in the range of \{100, 66, 33\}\%. The same test set is used for evaluating performance; labels in the test set are unbalanced ($93/7 \%$) while the training set is artificially balanced.
Table \ref{table:2} shows our system performs exactly \emph{on par} with logistic regression run on perfectly linked datasets (the non-private setting), with respect to accuracy, AUC and f1-score. We also record the percent for wrongly linked entities by our entity matching algorithm: around $1\%$ of the entities is linked as a mistake. The main take-away message is that those mistakes are not detrimental for learning; hence, we support the claim  of Theorem \ref{thAPPROX1} on a practical ground.
(Notice also how the fraction of shared entities does not seem to matter, arguably due to the fact that we learn a low dimensional linear model which cannot capitalize the large training set.)

\setlength{\tabcolsep}{0.5em}
\begin{table}[t]
\small
\centering
\begin{tabular}{|l|ccc|cccc|}
\toprule
& \multicolumn{3}{c|}{logistic regression} & \multicolumn{3}{c}{our (delta)} & matching \\
dataset               & acc & AUC & f1 & acc & AUC & f1 & errors (\%)\\
\midrule
\textit{PACS-33} & 88.5 & 80.4 & 36.9 & +0.1 & +0.0 & +0.1 &  1.0 \\
\textit{PACS-66} & 88.6 & 80.3 & 36.9 & -0.1 & +0.0 & +0.1  & 0.9 \\
\textit{PACS-100} & 88.6 & 80.3 & 37.2 & +0.1 & -0.0 & +0.0  & 0.8 \\
 \hline \bottomrule
\end{tabular}
\caption{Our system vs. logistic regression on perfectly linked data on PACS. We also record the percent of wrong matches by our entity resolution algorithm.}
\label{table:2}
\end{table}

\section{Conclusion}
\label{sec:conclusion}

We have shown how to learn a linear classifier in a privacy-preserving federated fashion when
data is vertically partitioned. We addressed the problem end-to-end, by 
pipelining entity resolution and distributed logistic regression by using Paillier encryption. 
All records, including identifiers required for linkage and
the linkage map itself, remain confidential from other parties.
To the best of our knowledge, our system is the first scalable and accurate solution to this problem. 
Importantly, we do not introduce extrinsic noise for gaining
privacy, which would hinder predictive performance; differential privacy
could, however, be applied on top of our protocol when the model itself is deemed susceptible to malicious attacks.
On the theory side, we provide the first
analysis of the impact of entity resolution on learning, with a potential for further use
for the design of entity matching methods specifically targeted to
learning. Our results show that the picture
of learning is not changed under reasonable assumptions on the
magnitude and size of mistakes due to entity resolution --- even rates
for generalization can remain of the same order. More: errors
introduced during entity resolution do not impact the examples with large
margin classification by the optimal (unknown) classifier: they are given the same
class by our classifier as well. Since federated learning can
dramatically improve the accuracy of this optimal classifier (because
it is built over a potentially much larger set of features), the message this observation
carries is that it can be extremely beneficial to carry out federated
learning in the setting where each peer's data provides a significant
uplift to the other. Our results signal the existence of non-trivial
tradeoffs for entity-resolution to be optimized with the objective of
learning from linked data. We hope such results will contribute to spur
related research in the active and broad field fo entity resolution.

\bibliographystyle{abbrvnat}
\bibliography{references,bibgen}

\begin{thebibliography}{44}
\providecommand{\natexlab}[1]{#1}
\providecommand{\url}[1]{\texttt{#1}}
\expandafter\ifx\csname urlstyle\endcsname\relax
  \providecommand{\doi}[1]{doi: #1}\else
  \providecommand{\doi}{doi: \begingroup \urlstyle{rm}\Url}\fi

\bibitem[Aono et~al.(2016)Aono, Hayashi, Trieu~Phong, and
  Wang]{Aono:2016:SSL:2857705.2857731}
Y.~Aono, T.~Hayashi, L.~Trieu~Phong, and L.~Wang.
\newblock Scalable and secure logistic regression via homomorphic encryption.
\newblock In \emph{CODASPY}, 2016.

\bibitem[Bache and Lichman(2013)]{UCI}
K.~Bache and M.~Lichman.
\newblock Uci machine learning repository, 2013.
\newblock \url{http://archive.ics.uci.edu/ml}.

\bibitem[Bartlett and Mendelson(2002)]{bmRA}
P.-L. Bartlett and S.~Mendelson.
\newblock Rademacher and gaussian complexities: Risk bounds and structural
  results.
\newblock \emph{JMLR}, 3:\penalty0 463--482, 2002.

\bibitem[Ben-Or et~al.(1988)Ben-Or, Goldwasser, and Wigderson]{benor88}
M.~Ben-Or, S.~Goldwasser, and A.~Wigderson.
\newblock Completeness theorems for non-cryptographic fault-tolerant
  distributed computation.
\newblock In \emph{Proceedings of the Twentieth Annual ACM Symposium on Theory
  of Computing}, 1988.

\bibitem[Bhatia(1997)]{bMA}
R.~Bhatia.
\newblock \emph{Matrix Analysis}.
\newblock Springer, 1997.

\bibitem[BlueKrypt(2017)]{keylength}
BlueKrypt, 2017.
\newblock v 30.4 - February 23, 2017 \url{www.keylength.com}.

\bibitem[Chaudhuri and Monteleoni(2009)]{chaudhuri2009privacy}
K.~Chaudhuri and C.~Monteleoni.
\newblock Privacy-preserving logistic regression.
\newblock In \emph{NIPS}, 2009.

\bibitem[Christen(2008)]{febrl}
P.~Christen.
\newblock Febrl: a freely available record linkage system with a graphical user
  interface.
\newblock In \emph{Proceedings of the second Australasian workshop on Health
  data and knowledge management-Volume 80}. Australian Computer Society, Inc.,
  2008.

\bibitem[Christen(2012)]{christen12}
P.~Christen.
\newblock \emph{Data matching: concepts and techniques for record linkage,
  entity resolution, and duplicate detection}.
\newblock Springer Science \& Business Media, 2012.

\bibitem[Damg{\aa}rd et~al.(2008)Damg{\aa}rd, Geisler, and
  Kr{\o}igaard]{cryptoeprint:2008:321}
I.~Damg{\aa}rd, M.~Geisler, and M.~Kr{\o}igaard.
\newblock A correction to ``{Efficient and Secure Comparison for On-Line
  Auctions}''.
\newblock Cryptology ePrint Archive, Report 2008/321, 2008.
\newblock \url{http://eprint.iacr.org/2008/321}.

\bibitem[Duchi et~al.(2013)Duchi, Jordan, and Wainwright]{duchi2013local}
J.~C. Duchi, M.~I. Jordan, and M.~J. Wainwright.
\newblock Local privacy and statistical minimax rates.
\newblock In \emph{FOCS}, 2013.

\bibitem[Duverle et~al.(2015)Duverle, Kawasaki, Yamada, Sakuma, and
  Tsuda]{duverle2015spw}
D.~Duverle, S.~Kawasaki, Y.~Yamada, J.~Sakuma, and K.~Tsuda.
\newblock Privacy-preserving statistical analysis by exact logistic regression.
\newblock In \emph{IEEE Security and Privacy Workshops (SPW)}, 2015.

\bibitem[Dwork(2008)]{dwork2008differential}
C.~Dwork.
\newblock Differential privacy: A survey of results.
\newblock In \emph{International Conference on Theory and Applications of
  Models of Computation}, 2008.

\bibitem[Dwork and Roth(2014)]{drTA}
C.~Dwork and A.~Roth.
\newblock The algorithmic foudations of differential privacy.
\newblock \emph{Foundations and Trends in Theoretical Computer Science},
  9:\penalty0 211--407, 2014.

\bibitem[Esperan{\c{c}}a et~al.(2017)Esperan{\c{c}}a, Aslett, and
  Holmes]{eahEA}
P.~M. Esperan{\c{c}}a, L.~J.~M. Aslett, and C.~C. Holmes.
\newblock Encrypted accelerated least squares regression.
\newblock In \emph{AISTATS'17}, pages 334--343, 2017.

\bibitem[Gasc{\'o}n et~al.(2017)Gasc{\'o}n, Schoppmann, Balle, Raykova,
  Doerner, Zahur, and Evans]{gascon2017privacy}
A.~Gasc{\'o}n, P.~Schoppmann, B.~Balle, M.~Raykova, J.~Doerner, S.~Zahur, and
  D.~Evans.
\newblock Privacy-preserving distributed linear regression on high-dimensional
  data.
\newblock \emph{PoPET}, 2017.

\bibitem[Gentry(2009)]{fhe09}
C.~Gentry.
\newblock Fully homomorphic encryption using ideal lattices.
\newblock In \emph{Proceedings of the Forty-first Annual ACM Symposium on
  Theory of Computing}, 2009.

\bibitem[Hall and Fienberg(2010)]{hall2010privacy}
R.~Hall and S.~E. Fienberg.
\newblock Privacy-preserving record linkage.
\newblock In \emph{International conference on privacy in statistical
  databases}, 2010.

\bibitem[Kaggle(2011)]{kaggle}
Kaggle.
\newblock {Give me some credit}, 2011.
\newblock \url{www.kaggle.com/c/GiveMeSomeCredit}.

\bibitem[Kakade et~al.(2008)Kakade, Sridharan, and Tewari]{kstOT}
S.~Kakade, K.~Sridharan, and A.~Tewari.
\newblock On the complexity of linear prediction: Risk bounds, margin bounds,
  and regularization.
\newblock In \emph{NIPS*21}, pages 793--800, 2008.

\bibitem[Konecn{\'{y}} et~al.(2016)Konecn{\'{y}}, McMahan, Ramage, and
  Richt{\'{a}}rik]{konecny16}
J.~Konecn{\'{y}}, H.~B. McMahan, D.~Ramage, and P.~Richt{\'{a}}rik.
\newblock Federated optimization: Distributed machine learning for on-device
  intelligence.
\newblock \emph{CoRR}, abs/1610.02527, 2016.

\bibitem[Kuzu et~al.(2011)Kuzu, Kantarcioglu, Durham, and Malin]{Kuzu2011}
M.~Kuzu, M.~Kantarcioglu, E.~Durham, and B.~Malin.
\newblock A constraint satisfaction cryptanalysis of bloom filters in private
  record linkage.
\newblock In \emph{PETS}, Berlin, Heidelberg, 2011.

\bibitem[Kuzu et~al.(2013)Kuzu, Kantarcioglu, Durham, Toth, and
  Malin]{Kuzu2013}
M.~Kuzu, M.~Kantarcioglu, E.~Durham, C.~Toth, and B.~Malin.
\newblock A practical approach to achieve private medical record linkage in
  light of public resources.
\newblock \emph{Journal of the American Medical Information Association}, 2013.

\bibitem[McMahan et~al.(2017)McMahan, Moore, Ramage, Hampson, and Agu\"uera~y
  Arcas]{mcmahan2016communication}
H.~B. McMahan, E.~Moore, D.~Ramage, S.~Hampson, and B.~Agu\"uera~y Arcas.
\newblock Communication-efficient learning of deep networks from decentralized
  data.
\newblock In \emph{AISTATS}, 2017.

\bibitem[Mohassel and Zhang(2017)]{mohassel17}
P.~Mohassel and Y.~Zhang.
\newblock Secureml: A system for scalable privacy-preserving machine learning.
\newblock In \emph{2017 IEEE Symposium on Security and Privacy (SP)}, pages
  19--38, May 2017.
\newblock \doi{10.1109/SP.2017.12}.

\bibitem[Niedermeyer et~al.(2014)Niedermeyer, Steinmetzer, Kroll, and
  Schnell]{Niedermeyer2014}
F.~Niedermeyer, S.~Steinmetzer, M.~Kroll, and R.~Schnell.
\newblock Cryptanalysis of basic bloom filters used for privacy preserving
  record linkage.
\newblock \emph{Journal of Privacy and Confidentiality}, 2014.

\bibitem[Nock(2016)]{nockOR}
R.~Nock.
\newblock On regularizing rademacher observation losses.
\newblock In \emph{NIPS}, 2016.

\bibitem[Nock and Nielsen(2009)]{nock2009bregman}
R.~Nock and F.~Nielsen.
\newblock Bregman divergences and surrogates for learning.
\newblock \emph{IEEE Transactions on Pattern Analysis and Machine
  Intelligence}, 2009.

\bibitem[Nock et~al.(2015)Nock, Patrini, and Friedman]{npaRO}
R.~Nock, G.~Patrini, and A.~Friedman.
\newblock Rademacher observations, private data, and boosting.
\newblock In \emph{ICML}, 2015.

\bibitem[Paillier(1999)]{paillier99}
P.~Paillier.
\newblock Public-key cryptosystems based on composite degree residuosity
  classes.
\newblock In \emph{International Conference on the Theory and Applications of
  Cryptographic Techniques}, 1999.

\bibitem[Patrini(2016)]{patrini-thesis}
G.~Patrini.
\newblock \emph{Weakly supervised learning via statistical sufficiency}.
\newblock PhD thesis, The Australian National University, 2016.

\bibitem[Patrini et~al.(2014)Patrini, Nock, Rivera, and Caetano]{pnrcAN}
G.~Patrini, R.~Nock, P.~Rivera, and T.~Caetano.
\newblock ({A}lmost) no label no cry.
\newblock In \emph{NIPS*27}, 2014.

\bibitem[Patrini et~al.(2016{\natexlab{a}})Patrini, Nielsen, Nock, and
  Carioni]{pnncLF}
G.~Patrini, F.~Nielsen, R.~Nock, and M.~Carioni.
\newblock Loss factorization, weakly supervised learning and label noise
  robustness.
\newblock In \emph{ICML}, 2016{\natexlab{a}}.

\bibitem[Patrini et~al.(2016{\natexlab{b}})Patrini, Nock, Hardy, and
  Caetano]{pnhcFL}
G.~Patrini, R.~Nock, S.~Hardy, and T.~Caetano.
\newblock Fast learning from distributed datasets without entity matching.
\newblock In \emph{IJCAI}, 2016{\natexlab{b}}.

\bibitem[{python-paillier}(2017)]{phe}
{python-paillier}, 2017.
\newblock release v1.3, \url{github.com/n1analytics/python-paillier}.

\bibitem[Schmidt et~al.(2013)Schmidt, Roux, and Bach]{schmidt13}
M.~Schmidt, N.~L. Roux, and F.~Bach.
\newblock Minimizing finite sums with the stochastic average gradient.
\newblock \emph{Mathematical Programming}, 2013.

\bibitem[Schnell(2013)]{sEP}
R.~Schnell.
\newblock Efficient private record linkage of very large datasets.
\newblock In \emph{59$^{th}$ World Statistics Congress}, 2013.

\bibitem[Schnell et~al.(2011)Schnell, Bachteler, and Reiher]{schnell11}
R.~Schnell, T.~Bachteler, and J.~Reiher.
\newblock A novel error-tolerant anonymous linking code.
\newblock Technical report, Paper No. WP-GRLC-2011-02, German Record Linkage
  Center Working Paper Series, 2011.

\bibitem[Vatsalan et~al.(2013{\natexlab{a}})Vatsalan, Christen, and
  Verykios]{vatsalan13}
D.~Vatsalan, P.~Christen, and V.~S. Verykios.
\newblock A taxonomy of privacy-preserving record linkage techniques.
\newblock \emph{Information Systems}, 2013{\natexlab{a}}.

\bibitem[Vatsalan et~al.(2013{\natexlab{b}})Vatsalan, Christen, and
  Verykios]{vatsalan13b}
D.~Vatsalan, P.~Christen, and V.~S. Verykios.
\newblock Efficient two-party private blocking based on sorted nearest
  neighborhood clustering.
\newblock In \emph{CIKM '13 Proceedings of the 22nd ACM international
  Conference on Information and Knowledge Management}, 2013{\natexlab{b}}.

\bibitem[Winkler(2009)]{wRL}
W.~E. Winkler.
\newblock Record linkage.
\newblock In \emph{Handbook of Statistics}, pages 351--380. Elsevier, 2009.

\bibitem[Wu et~al.(2013)Wu, Teruya, Kawamoto, Sakuma, and
  Kikuchi]{wu2013privacy}
S.~Wu, T.~Teruya, J.~Kawamoto, J.~Sakuma, and H.~Kikuchi.
\newblock Privacy-preservation for stochastic gradient descent application to
  secure logistic regression.
\newblock In \emph{27th Annual Conference of the Japanese Society for
  Artificial Intelligence}, 2013.

\bibitem[Xie et~al.(2016)Xie, Wang, Boker, and
  Brown]{DBLP:journals/corr/XieWBB16}
W.~Xie, Y.~Wang, S.~M. Boker, and D.~E. Brown.
\newblock Privlogit: Efficient privacy-preserving logistic regression by
  tailoring numerical optimizers.
\newblock \emph{CoRR}, 2016.

\bibitem[Yao(1986)]{yao86}
A.~Yao.
\newblock How to generate and exchange secrets.
\newblock \emph{27th Annual Symposium on Foundations of Computer Science
  (SFCS)}, 1986.

\end{thebibliography}

\clearpage

\renewcommand\thesection{\Roman{section}}
\renewcommand\thesubsection{\thesection.\arabic{subsection}}
\renewcommand\thesubsubsection{\thesection.\thesubsection.\arabic{subsubsection}}

\renewcommand*{\thetheorem}{\Alph{theorem}}
\renewcommand*{\thelemma}{\Alph{lemma}}
\renewcommand*{\thecorollary}{\Alph{corollary}}

\renewcommand{\thetable}{A\arabic{table}}

\setcounter{section}{0}

\section*{Appendix}\label{sec:appendix}
Theorems
and Lemmata are numbered with letters (A, B, ...) to make a clear difference with
the main file numbering.

\section*{Table of contents}

\noindent \textbf{Paillier encryption scheme} \hrulefill Pg \pageref{app:part-homom-encrypt}\\
\noindent \textbf{Encoding} \hrulefill Pg \pageref{app:encod-float-point}\\
\noindent \textbf{Secure initialization and loss computation} \hrulefill Pg \pageref{app:protocols}\\
\noindent \textbf{Security evaluation} \hrulefill Pg
\pageref{app:security-eval}\\
\noindent \textbf{Proofs} \hrulefill Pg
\pageref{app:proofs}\\
\noindent $\hookrightarrow$ Notations for proofs\hrulefill Pg
\pageref{app:notations}\\
\noindent $\hookrightarrow$ Helper Theorem\hrulefill Pg
\pageref{app:proof-thEXACT1}\\
\noindent $\hookrightarrow$ Assumptions: details and discussion\hrulefill Pg
\pageref{app:assum}\\
\noindent $\hookrightarrow$ Proof of Theorem \ref{thAPPROX1}\hrulefill Pg
\pageref{app:proof-thAPPROX1}\\
\noindent $\hookrightarrow$ Proof of Theorem \ref{thIMMUNE}\hrulefill Pg
\pageref{app:proof-thIMMUNE}\\
\noindent $\hookrightarrow$ Proof of Theorem \ref{thDIFFLOSS}\hrulefill Pg
\pageref{app:proof-thDIFFLOSS}\\
\noindent $\hookrightarrow$ Proof of Theorem \ref{thGENTHETA}\hrulefill Pg
\pageref{app:proof-thGENTHETA}\\
\noindent \textbf{Cryptographic longterm keys, entity resolution and
  learning} \hrulefill Pg \pageref{app:clk}

\newpage
\section{Paillier encryption scheme}
\label{app:part-homom-encrypt}

We describe in more detail the Paillier
encryption scheme~\cite{paillier99}, an asymmetric additively homomorphic
encryption scheme.   In what follows,
$\ZZ_t$ will denote the ring of integers $\{0,1,...,t-1\}$, with
addition and multiplication performed modulo the positive integer $t$,
and $\ZZ_t^*$ will denote the set of invertible elements of $\ZZ_t$
(equivalently, elements that do not have a common divisor with $t$).

In the Paillier system the private key is a pair of randomly selected
large prime numbers $p$ and $q$, and the public key is their product
$m = pq$.  The plaintext space is the set $\ZZ_m$, the ciphertext
space is $\ZZ_{m^2}^*$, and the encryption of an element $x\in\ZZ_m$
is given by:
\begin{equation}
  \label{eq:8}
  x \mapsto g^x r_x^m
\end{equation}
where $g$ is a generator of $\ZZ_{m^2}^*$ (which can be taken to be
$m + 1$) and $r_x$ is an element of $\ZZ_m^*$ selected uniformly
at random (the subscript $x$ is to indicate that a new value should be
selected for every input). Note that any $r_x$ produces a valid
ciphertext, so in particular many valid ciphertexts correspond to a
given plaintext.  Let the encryption of a number $x$ be
$\Enc{x}$.  It follows immediately from \eqref{eq:8} that, for
elements $x, y \in \ZZ_m$,
\begin{equation}
  \Enc{x}\Enc{y} = g^{x+y}(r_xr_y)^m = \Enc{x + y}\:,
\end{equation}
whence we define the operator `$\oplus$' by
$\Enc{x} \oplus \Enc{y} = \Enc{x + y}$. (In Section \ref{app:part-homom-encrypt} we 
denoted this operator simply as '$+$'.) Then, by linearity,
\begin{equation}
  \label{eq:2B}
  \bigoplus_{i=1}^y\Enc{x} = \EncB{\sum_{i=1}^y x} = \Enc{xy}.
\end{equation}
Note that in this latter case, $y$ is not encrypted. Suppose a party
is sent $\Enc{x}$ and computes $\Enc{xy}$ from $\Enc{x}$ and $y$
using~\eqref{eq:2B}. If the party who originally computed $\Enc{x}$
subsequently obtains $\Enc{xy}$, then they can verify whether a guess
$y'$ for $y$ is correct by calculating $\bigoplus_{i=1}^{y'}\Enc{x}$
and comparing it with $\Enc{xy}$. In general recovering $y$ here is
equivalent to the discrete logarithm problem, but if $y$ is known to
come from a small set then this becomes simple to solve. We can
protect $y$ when ``multiplying by $y$'' by adding $\Enc{0}$, forcing
the random multiplier $r$ to change.  Hence we define multiplication by:
\begin{equation}
  \Enc{x} y = \Enc{x y} \oplus \Enc{0}.
\end{equation}

It is worth reiterating that all operations that occur \emph{inside}
the $\Enc{\,\cdot\,}$ occur in the plaintext space $\ZZ_m$, hence
modulo the public key $m$, while all operations that occur
\emph{outside} the $\Enc{\,\cdot\,}$ occur in the ciphertext space
$\ZZ_{m^2}$, hence modulo $m^2$.

For a vector $\Vx = (x_i)_i$, the notation
$\Enc{\Vx}$ should be interpreted component-wise, that is
$\Enc{\Vx} = \big(\Enc{x_i}\big)_i$. We also extend the definition of
$\oplus$ to operate component-wise on vectors, and of multiplication
to perform scalar multiplication:
$s \Enc{\Vx} = \Enc{s\Vx} = \Enc{s} \Vx$.  Additionally, we define:
\begin{equation}
  \Enc{\Vx}\circ\Vy = \big(\Enc{x_i}  y_i\big)_i = \Enc{\Vx\circ\Vy}
\end{equation}
to be the component-wise (or ``Hadamard'') product with a plaintext
vector and:
\begin{equation}
  \Enc{\Vx}\odot\Vy = \bigoplus_i\Enc{x_i}  y_i = \Enc{\dotprod{\Vx}{\Vy}}
\end{equation}
to be the ``dot product'' of an encrypted vector with a plaintext
vector. (In Section \ref{app:part-homom-encrypt} we 
omitted the symbol for this operator, as it would for a standard product.)
We can also extend the definition of `$ $' to express
matrix multiplication: For matrices $\MA$ and $\MB$ with compatible
dimensions,
\begin{align}
  \Enc{\MA\MB} &= \MA  \Enc{\MB} \\
  &= \big(\MA_i \odot \Enc{\MB^j} \big)_{i,j}
  = \big(\Enc{\MA_i} \odot \MB^j \big)_{i,j}
  = \Enc{\MA}  \MB  \nonumber
\end{align}
is the usual product of two matrices. In particular, for a vector
$\Vx$, we can compute:
$\Enc{\MA\Vx} = \Enc{\MA}  \Vx = \MA  \Enc{\Vx}$ or
$\Enc{\Vx\MA} = \Enc{\Vx}  \MA = \Vx  \Enc{\MA}$ depending
on when the dimensions are compatible.

There is, of course, a decryption process corresponding to the
encryption process of \eqref{eq:8}, but the details are out of scope here.
Suffice it to say that, computationally
speaking, the cost of a decryption is in the order of a single modular
exponentiation modulo $m$.
\section{Encoding}
\label{app:encod-float-point}

The Paillier cryptosystem only provides arithmetic for elements of
$\mathbb{Z}_m$. We give details on our encoding of floating-point numbers.

Fix an integer $\beta > 1$, which will serve as a \emph{base}, and let $q$
be any fraction.  Any pair of integers $(s, e)$ satisfying
\(q = s\beta^e\) is called a \emph{base-$\beta$ exponential representation} of
$q$.  The values $s$ and $e$ are called the \emph{significand} and
\emph{exponent} respectively. Let $q = (s, e)$ and $r = (t, f)$ be the
base-$\beta$ exponential representations of two fractions and without loss
of generality assume that $f \ge e$.  Then $(s + t\beta^{f - e}, e)$ is a
representation of $q + r$ which follows by using the equivalent
representation $(t\beta^{f - e}, e)$ for $(t, f)$.  A representation for
$qr$ is simply given by $(st, e + f)$.

In order to be compatible with the Paillier encryption system we will
require the significand $s$ to satisfy $0\le s < m$ where $m$ is the
public key. This limits the precision of the encoding to $\log_2m$
bits, which is not a significant impediment as a reasonable choice of
$m$ is at least 1024 bits \citep{keylength}.
Let:
\begin{equation}
  q = \sigma \phi 2^\epsilon,
\end{equation}
be a unbiased IEEE 754 floating-point number, where
\begin{equation}
  \sigma = \pm 1, \quad \phi = 1 + \sum_{i=1}^d b_{d-i}2^{-i},
\end{equation}
and $b_j \in \{0, 1\}$ for all $j = 0, \ldots, d - 1$; for single
precision $d = 23$ and $-127 < \epsilon < 128$, while for double
precision $d = 52$ and $-1023 < \epsilon < 1024$. (We do not treat
subnormal numbers, so we can treat $\epsilon = 0$ specially as
representing $q = 0$.) An \emph{encoding} of $q$ relative to a
Paillier key $m$ is a base-$\beta$ exponential representation computed
as follows. If $q = 0$, set $s = e = 0$. Otherwise, let
$e = \lfloor (\epsilon - d + 1) \log_\beta2 \rfloor$ and
$s' = \phi \beta^{-e}$; one can verify that this choice of $e$ is the
smallest such that $s' \ge 1$. Finally, set $s = s'$ if $\sigma = 1$
and $s = m - s'$ if $\sigma = -1$.  Then $(s, e)$ is a base-$\beta$
exponential representation for $q$ with $0 \le s < m/2$ when $q \ge 0$
and $m/2 \le s < m$ when $q < 0$. This consciously mimics the
  familiar two's compliment encoding of integers. The advantage of
this encoding (over, say, a fixed-point encoding scheme) is that we
can encode a very large range of values.

Since the significand must be less than $m$, it can happen that the
sum or product of two encodings is too big to represent correctly and
\emph{overflow} occurs. Specifically, for a fraction $q = (s, e)$, if
$|q|/\beta^e \ge m/2$ then overflow occurs.  In practice we start with
fractions with 53 bits of precision (for a double precision IEEE 754
float) and the Paillier public key $m$ allows for $\log_2(m/2) = 1023$
bits of precision so we can accumulate $\lfloor 1023/53 \rfloor = 19$
multiplications by 53-bit numbers before overflow might happen.
Overflow cannot be detected in the scheme described here. One partial
solution to detecting overflow is to reserve an interval between the
positive and negative regions and to consider a number to have
overflowed if its value is in that region. This method does not detect
numbers that overflow beyond the reserved interval, and has the
additional drawback of reducing the available range of encoded
numbers.

We \emph{define} the encryption of an encoded number $q = (s, e)$ to
be $\Enc{q} = \bigl(\Enc{s}, e\bigr)$.  Note in particular that we
encrypt the significand but \emph{we do not encrypt the exponent}, so
some information about the order of magnitude of an encrypted number
is public.  There are protocols that calculate the maximum of two
encrypted numbers (see, for example, \citep{cryptoeprint:2008:321})
which could be used to calculate the maximum of encrypted exponents
(needed during addition as seen above). However it would be
prohibitively expensive to compute this maximum for every arithmetic
operation, so instead we leave the exponents public and leak the order
of magnitude of each number.

To see more precisely how much information is leaked, consider an
encrypted number $\Enc{q} = (\Enc{s}, e)$ encoded with respect to the
base $\beta$. The exponent satisfies
$e = \lfloor (\epsilon - d + 1) \log_\beta2 \rfloor$ (with $d = 23$ or
$52$ as above), hence:
\begin{equation}
  c \le \epsilon < c + \log_2\beta
\end{equation}
where $c = e\log_2\beta + d - 1$ is public information. Thus the original
exponent $\epsilon$ is known to be in a range of length $\log_2\beta$,
and hence we learn that:
\begin{equation}\label{eq:11}
  2^c \le q < 2^c\beta,
\end{equation}
that is, $q$ is within a factor of $\beta$ of $2^c$. (Note that
$e = 0$ includes the special case of $q = 0$.)  The extent of the
leakage can be mitigated in several ways, for example by choosing a
large enough base so that no significant detail can be recovered, or by fixing the exponent for all
encoded numbers as in fixed-point encodings.

Using the arithmetic of encoded numbers defined above, and assuming
$f \ge e$, we can extend the definitions of `$\oplus$' and
multiplication by:
\begin{equation}\label{eq:7}
  (\Enc{s}, e) \oplus (\Enc{t}, f) = (\Enc{s} \oplus \Enc{t}
  \beta^{f - e}, e)
\end{equation}
and
\begin{equation}
  (\Enc{s}, e) (t, f) = (\Enc{s} t, e + f).
\end{equation}
Note that the definition of `$\oplus$' in~\eqref{eq:7} has the
unfortunate consequence of requiring a multiplication of $\Enc{t}$
whenever $f \ne e$. It is clear that addition does not change the
leakage range (it is still $\beta$), while multiplication increases
the range to $\beta^2$, mitigating the leakage by a factor of $\beta$.

\section{Secure initialization and loss computation}
\label{app:protocols}

\subsection{Secure loss initialization}
\label{sec:secure-mean-operator}

\setcounter{algocf}{3}

\begin{algorithm}
  \setstretch{1.5}
\caption{Secure loss initialization\label{alg:MeanOperator}}
\KwData{Hold-out size $h$}
\KwResult{Caching of $\Enc{\Vmu}$ of the (undisclosed) hold-out $H$}

\BlankLine
\ThdCoord{}{send $h$ to \FDP\;}

\ThdFDP{} {
  sample the hold-out row $\holdoutset \subset
  \{1,\ldots,n\}$, $|H| = h$\;
  $\Enc{\Vm\circ\Vy}_\holdoutset \leftarrow \Enc{\Vm}_\holdoutset \circ \Vy_\holdoutset$\;
  $\Enc{\Vu} \leftarrow \tfrac{1}{h} \Enc{\Vm \circ \Vy}_\holdoutset^\top~\DPFH$\;
  send $\holdoutset$, $\Enc{\Vu}$ and $\Enc{\Vm\circ\Vy}_\holdoutset$ to \LDP\;
}

\ThdLDP{}{
  $\Enc{\Vv} \leftarrow \tfrac{1}{h} \Enc{\Vm \circ \Vy}_\holdoutset^\top~\DPL_\holdoutset$\;
  obtain $\Enc{\Vmu_\holdoutset}$ by concatenating $\Enc{\Vu}$ and $\Enc{\Vv}$\;
}
\end{algorithm}

Algorithm~\ref{alg:MeanOperator} performs the secure loss initialization
evaluation. Since $\Vmu_\holdoutset$ is independent from $\Vk$, we can cache
for future computation of the logistic loss.
To see why Algorithm~\ref{alg:MeanOperator} is correct, note that:
\begin{equation}
  \Enc{\Vu} = \tfrac{1}{h} \Enc{\Vm \circ \Vy}_\holdoutset^\top
  \DPFH
  = \tfrac{1}{h} \Encb{(\Vm\circ \Vy)_\holdoutset^\top \DPFH}\:.
\end{equation}
Similarly, $\Enc{\Vv} =\tfrac{1}{h} \Enc{(\Vm\circ \Vy)_\holdoutset^\top \DPLH}$,
and so, after concatenating,  \LDP\ obtains
$  \Enc{\Vmu_\holdoutset} = \tfrac{1}{h}\Encb{(\Vm\circ \Vy)_\holdoutset^\top \Data_\holdoutset}$.
\LDP\ only receives the encrypted forms of
$(\Vm\circ\Vy)_\holdoutset$ and
$\tfrac{1}{h}(\Vm\circ\Vy)_\holdoutset^\top\DPFH$ from \FDP\ and
only holds the final result $\Enc{\Vmu_\holdoutset}$ in encrypted
form which is not sent to the \Coord\ .

\subsection{Secure logistic loss}
\label{sec:secure-logar-loss}

Algorithm~\ref{alg:LossEval} computes the secure logistic loss.
It is called in Algorithm~\ref{alg:LogRegSGD}, once per iteration of the outer loop.

\begin{algorithm}[t]
  \setstretch{1.5}
\caption{Secure logistic loss\label{alg:LossEval}}
\KwData{Model $\Vk$, requires $\Enc{\Vmu_\holdoutset}$ and $\holdoutset$ cached by Algorithm \ref{alg:MeanOperator}}
\KwResult{$\loglike_\holdoutset(\Vk)$ of the (undisclosed) hold-out}
\BlankLine

\ThdCoord{}{send $\Vk$ to \FDP\;}

\ThdFDP{}{
  $\Vu \leftarrow \DPFH \Vk_\DPF$\;
  $\Enc{\Vm_\holdoutset\circ\Vu} \leftarrow \Enc{\Vm}_\holdoutset \circ \Vu$\;
  $\Enc{u'} \leftarrow \tfrac{1}{8h} (\Vu\circ\Vu)^\top \Enc{\Vm}$\;

  send $\Vk$, $\Enc{\Vm_\holdoutset\circ\Vu}$ and $\Enc{u'}$ to \LDP\;
}

\ThdLDP{}{
  $\Vv = \DPLH \Vk_\DPL$\;
  $\Enc{v'} \leftarrow \tfrac{1}{8h}(\Vv\circ\Vv)^\top \Enc{\Vm}_\holdoutset$\;
  $\Enc{w} \leftarrow \Enc{u'} + \Enc{v'} + \tfrac{1}{4h}\Vv^\top \Enc{\Vm_\holdoutset \circ \Vu}$\;
  $\Enc{\loglike_\holdoutset(\Vk)} \leftarrow \Enc{w} - \tfrac{1}{4h}\Vk^\top \Enc{\Vmu_\holdoutset}$\;
  send $\Enc{\loglike_\holdoutset(\Vk)}$ to \Coord\;
}
\ThdCoord{}{
  obtain $\loglike_\holdoutset(\Vk)$ by decryption with the private key\;
}
\end{algorithm}

To show the correctness of Algorithm~\ref{alg:LossEval} first note
that, for $i\in\holdoutset$, the $i$th component of
$\Enc{\Vm_\holdoutset\circ\Vu}$ is
$\Enc{m_i\dotprod{\Vk_\DPF}{\Vx_{i\DPF}}}$ and
$u' = \tfrac{1}{8h}\sum_im_i(\dotprod{\Vk_\DPF}{\Vx_{i\DPF}})^2$;
similarly for $\Vv$ and $v'$ \emph{mutatis mutandis}. Hence
\begin{equation}
  \begin{split}
    w_0 &= u' + v' + \tfrac{1}{4h} \bigl(\dotprod{(\Vm_\holdoutset \circ \Vu)}{\Vv}\bigr)\\
    &= \tfrac{1}{8h}\sum_{i\in\holdoutset} m_i(\dotkxi)^2\:,
  \end{split}
\end{equation}
since
\begin{equation}
  (\dotkxi)^2 = (\dotprod{\Vk_\DPF}{\Vx_{i\DPF}})^2 + (\dotprod{\Vk_\DPL}{\Vx_{i\DPL}})^2
      + 2 (\dotprod{\Vk_\DPF}{\Vx_{i\DPF}})(\dotprod{\Vk_\DPL}{\Vx_{i\DPL}})
\end{equation}
for every row $\Vx$ of $\Data$ by~\eqref{eq:10}. Then
$w = w_0 - \tfrac{1}{4}\dotprod{\Vk}{\Vmu_\holdoutset} \approx
\loglike(\Vk)$ by~\eqref{eq:6}, which can be computed since
$\Enc{\Vmu_\holdoutset}$ was previously computed by \LDP\ in
Algorithm~\ref{alg:MeanOperator}.

\Coord\ only sees the final result $\loglike_\holdoutset(\Vk)$.  \ADP{}~
see the model $\Vk$, and \LDP\ receives only the encrypted
forms of the values $m_i\dotprod{\Vk_\DPF}{\Vx_\DPF}$ and
$\sum_im_i(\dotprod{\Vk_\DPF}{\Vx_\DPF})^2$ (for $i\in\holdoutset$) from
\FDP.
\section{Security evaluation}\label{app:security-eval}

\begin{table}[t]
\small
\center
  \begin{tabular}[h]{c|c|c|c|c|c|c|c|c|c|c|c}
    \emph{agent} & $\DPF$ & $\DPL$ & $\Vm$ & $n$ & $\Vk$
    & $\holdoutset$ & $\trainingset$ & $\trainingsubset$
    & $\Vmu_\holdoutset$
    & $\loglike_\holdoutset$ & $\nabla\loglike_\trainingsubset$
    \\
    \hline
    \FDP   & \tickYes & & & \tickYes & \tickYes & \tickYes & \tickYes & \tickYes & & & \tickYes \\
    \LDP   & & \tickYes & & \tickYes & \tickYes & \tickYes & \tickYes & \tickYes & & & \tickYes \\
    \Coord & & & \tickYes & \tickYes & \tickYes & & & & & \tickYes & \tickYes \\
  \end{tabular}
  \caption{Summary of data visibility for each agent.}
  \label{tbl:accesstable}
\end{table}

It is important to clarify how the data is distributed and what is
visible to the three parties.
Table~\ref{tbl:accesstable} gives a summary of which values each party
will see during an execution of Algorithm~\ref{alg:LogRegSGD}. Most
entries will be obvious, with the possible exception of the column for
$\nabla\loglike_\trainingsubset(\Vk)$. The values for $\nabla\loglike_\trainingsubset(\Vk)$ can be
derived from the successive values of $\Vk$, which are seen by all
parties. If $\Vk$ and $\Vk'$ are successive values of the weight
vector, then $\nabla\loglike_\trainingsubset(\Vk) = \eta^{-1}(\Vk - \Vk') - \gamma \Vk$
from line \ref{line:updatek}; the values of $\eta$ and $\gamma$
are not private. Notice how $\Vmu_\holdoutset$ is known by nobody,
and it is accessible by $\LDP$ only in encrypted form.

Different protection mechanisms are in place to secure the
communication and protect against information leakage among the three
parties. 

The matched entities are protected from \FDP\ and \LDP\ by the implementation
of a private entity matching algorithm which creates a random permutation
for \FDP\ and for \LDP, and a mask received only by \Coord .
During the private entity matching algorithm, the information from \FDP\ and \LDP\ are
protected from \Coord\ by the construction of the CLKs.

Additively homomorphic encryption protects the values in \FDP's data from \LDP, and
vice-versa. This does not hold for \Coord, who owns the private
key. Protection from \Coord\ is achieved by sending only computed values which are not
considered private (such as the gradient which is shared with all participants).

We finally note that, while the data is protected via different mechanisms,
our work does not protect against the leakage of information through the model
and its iterative computation.

\paragraph{Potential leakage of information --- } We analyze two
sources of leakage:

\begin{itemize}
\item \textbf{Number of matches in a mini batch:} After the entity resolution, the data providers only receive 
an encrypted mask and a permutation. They do not learn the total number of matches between them. Thus,
when the first data provider chooses the mini batches at random, the number of matches in each of the batches
is unknown. However, this can lead to a potential leak of information, as the corresponding gradient will be revealed
in the clear. If, for instance, this gradient is zero, then both data provider know, that they do not have any entities in 
common within this mini batch. Or a mini batch only contains one match, then the gradient reveals the corresponding 
label.
This kind of leakage can be mitigated by carefully choosing the batch size. 
Let $X$ be a random variable describing the number of matches in a mini batch, then $X$ follows the hypergeometric 
distribution. Let $\text{CDF}_{\text{hyper}}(N, K, n, k)$ be its cumulative density function with $N$ is the population size,
$K$ is the number of success states in the population, $n$ is the number of draws, and $k$ is the number of observed 
successes. Then the probability that a mini batch 
contains not more than $k$ matches is:
\begin{equation*}
P[X \leq k] = \text{CDF}_\text{hyper}(n, M, s', k).
\end{equation*}
$M$ stands of the total number of matches. Note that the coordinator has all the necessary information available, and 
could therefore test if the probability for the chosen parameters is acceptably low or otherwise abort the protocol.

\item \textbf{Public exponent in encoding of encrypted numbers:} The floating point inspired encoding scheme for 
encrypted numbers leaks information about the plaintext through a public exponent. We quantify the leakage in 
Appendix \ref{app:encod-float-point} and want to point out that the leakage can be reduced by increasing the value 
for the base. However we do admit that this form of leakage is unnecessary, as we are confident that the algorithm 
will also work with a fixed-point encoding scheme. We plan to address this issue in a future version.

\end{itemize}

\section{Proofs}\label{app:proofs}

\subsection{Notations}\label{app:notations}


The proofs being quite heavy in notations, we define two new notations
(on the observation matrix and training sample)
that will be more convenient than in the main file.
\begin{mdframed}[style=MyFrame]
(\textbf{observation matrix}) the
observation matrix is now transposed --- it has observations in
columns. Thus, instead of writing as in eq. (\ref{defSPLIT}) (main file):
\begin{eqnarray}
  \Data & = & \left[\!
  \begin{array}{c|c}
    \DPFALL & \DPLALL
  \end{array}
\!\right]\label{defSPLIT2}\:\:,
\end{eqnarray}
We let from now on, with a matrix notation,
\begin{eqnarray}
\X & \defeq & \left[
\begin{array}{c}
\X_\anchor\\\cline{1-1}
\X_\shuffle
\end{array}
\right]\label{notmat}
\end{eqnarray}
denote a block partition of $\X$, where $\X_\anchor \in \mathbb{R}^{d_\anchor
  \times m}$ and $\X_\shuffle \in \mathbb{R}^{d_\shuffle
  \times m}$ (so, $\X_\anchor = \DPFALL^\top$, $\X_\shuffle =
\DPLALL^\top$). 
\end{mdframed}
\begin{mdframed}[style=MyFrame]
(\textbf{training and ideal sample}) Instead of $S^*$ to denote the
ideal sample, we let $S$ denote the ideal sample, and the training
sample we have access to (produced via error-prone entity resolution)
is now $\hat{S}$ instead of $S$.
\end{mdframed}
We do not observe $\X$ but rather an estimate through
\textit{approximately accurate} linkage,
\begin{eqnarray}
\hat{\X} & \defeq & \left[
\begin{array}{c}
\X_\anchor\\\cline{1-1}
\hat{\X}_\shuffle
\end{array}
\right]\:\:,
\end{eqnarray}
where
\begin{eqnarray} 
\hat{\X}_\shuffle & \defeq & \X_\shuffle \PERM_*\:\:,
\end{eqnarray}
where $\PERM_* \in
\{0,1\}^{n\times n}$ is a
permutation matrix. To refer to the features
of $\FDPA$ and $\LDPB$ with words we shall call them the \textit{anchor}
($\FDPA$) and \textit{shuffle} ($\LDPB$) feature spaces.

Any permutation matrix can be factored as a product of elementary
permutation matrices. So suppose that
\begin{eqnarray}
\PERM_* & = & \prod_{t=1}^T \PERM_{t}\:\:,\label{productperm}
\end{eqnarray}
where $\PERM_{t}$ is an elementary permutation matrix, where $T<n$ is
unknown. $\hat{\X}$ can be progressively constructed from a
sequence $\hat{\X}_0, \hat{\X}_1, ..., \hat{\X}_T$ where $\hat{\X}_0 =
\X$ and for $t\geq 1$,
\begin{eqnarray}
\hat{\X}_t & \defeq & \left[
\begin{array}{c}
\X_\anchor\\\cline{1-1}
\hat{\X}_{t \shuffle}
\end{array}
\right]\:\:, \hat{\X}_{t \shuffle} \defeq \X_\shuffle \prod_{j=1}^t \PERM_{j}\:\:.
\end{eqnarray}
We recall that by convention, permutations act on the shuffle set $\shuffle$ of
features, and labels are associated with $\anchor$ and therefore are
not affected by the permutation (without loss of generality). Let
\begin{eqnarray}
\hat{\X}_t & \defeq & [\hat{\ve{x}}_{t1} \:\: \hat{\ve{x}}_{t2} \:\: \cdots
\:\: \hat{\ve{x}}_{tn}] \:\:,
\end{eqnarray}
denote the column vector decomposition of $\hat{\X}_t$ (with
$\hat{\ve{x}}_{0i} \defeq \ve{x}_i$) and let $\hat{S}_t$ be the training sample obtained from the $t$
first permutations in the sequence ($\hat{S}_0 \defeq
S$). Hence, $\hat{S}_t \defeq \{(\hat{\ve{x}}_{ti}, y_i), i
= 1, 2, ..., n\}$.

\begin{definition}\label{defmeano}
The mean operator associated to $\hat{S}_t$ is $\mathbb{R}^d \ni \ve{\mu}_t \defeq
\ve{\mu}(\hat{S}_t) \defeq \sum_i y_i \cdot \hat{\ve{x}}_{ti}$.
\end{definition}
The mean operator is a sufficient statistics for the class in linear
models \citep{pnrcAN}. We can make at this point a remark that is
going to be crucial in our results, and obvious from its definition:
the mean operator is \textit{invariant} to permutations made within
classes, \textit{i.e.} $\ve{\mu}_T = \ve{\mu}_0$ if $\PERM_{*}$ factorizes as two permutations, one affecting the positive
class only, and the other one affecting the negative class only. Since
the optimal classifier for the Taylor loss is a linear mapping of the
mean operator, we see that the drift in its optimal classifier due to
permutations will be
due \textit{only} to mistakes in the linear mapping when $\PERM_{*}$
factorizes in such a convenient way.
\begin{figure}[t]
\centering
\includegraphics[trim=25bp 280bp 330bp
10bp,clip,width=.70\linewidth]{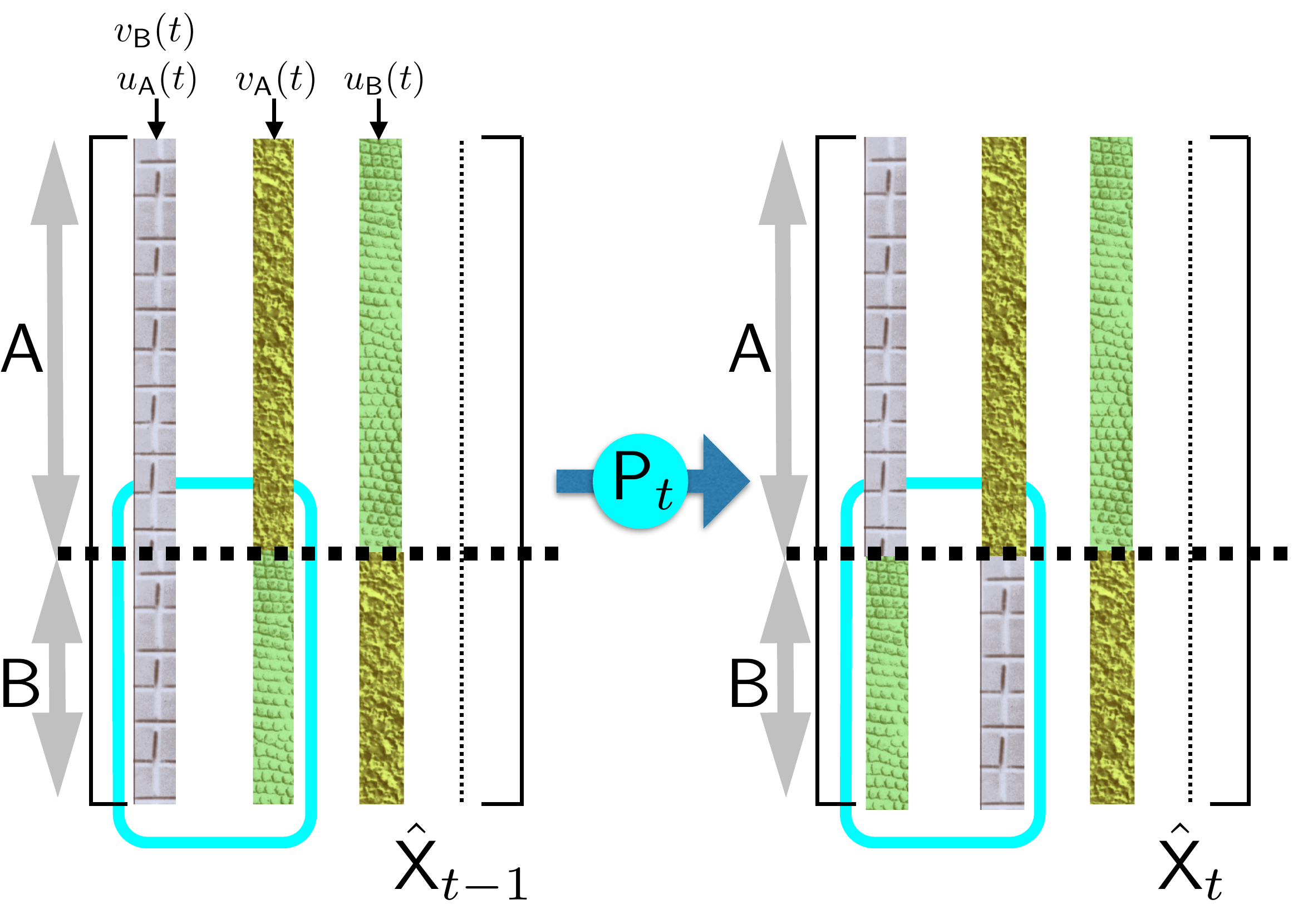} 
\caption{Permutation $\PERM_t$ applied to observation matrix
  $\hat{\X}_{t-1}$ and subsequent matrix $\hat{\X}_t$, using notations $\ua{t}$, $\ub{t}$, $\va{t}$ and
  $\vb{t}$. Anchor textures denote references for the textures in
  $\shuffle$ (best viewed in color).\label{fig:perm}}
\end{figure}
Now, consider elementary permutation $\PERM_t$. We call \textit{pairs}
for $\PERM_t$ the two column indexes affected by $\PERM_t$. Namely, 
\begin{enumerate}
\item [(a)] we denote
$\ua{t}$ and $\va{t}$ the two column indexes of the
observations in $\X$ affected by elementary permutation $\PERM_t$;
\item [(b)] we denote $\ub{t}$ and $\vb{t}$ the two column indexes of the
observations in $\X$ such that \textit{after} $\PERM_t$, the shuffle
part of column $\ua{t}$ (resp. $\va{t}$) is the shuffle part of
$\ve{x}_{\ub{t}}$ (resp. $\ve{x}_{\vb{t}}$).
\end{enumerate}
To summarize (b), we have the following Lemma.
\begin{lemma}\label{lemUAUB}
The following holds for any $t \geq 1$:
\begin{eqnarray}
(\hat{\ve{x}}_{t\ua{t}})_\shuffle & = &
(\ve{x}_{\ub{t}})_\shuffle\:\:,\label{eqm1EX1}\\
(\hat{\ve{x}}_{t\va{t}})_\shuffle & = & (\ve{x}_{\vb{t}})_\shuffle\:\:. \label{eqm1EX2}
\end{eqnarray}
\end{lemma}
Figure \ref{fig:perm} shows an example of our notations. 
\begin{example}\label{exampleEX1}
Denote for short $\{0,1\}^{n\times n} \ni \Theta_{u,v} \defeq \ve{1}_{u}\ve{1}^\top_{v} +
\ve{1}_{v}\ve{1}^\top_{u} - \ve{1}_{v}\ve{1}^\top_{v} -
\ve{1}_{u}\ve{1}^\top_{u}$ (symmetric) such that $\ve{1}_u$ is the $u^{th}$
canonical basis vector of $\mathbb{R}^n$. For $t=1$, it follows
\begin{eqnarray}
\ub{1} & = & \va{1}\:\:,\label{eq0EX1}\\
\vb{1} & = & \ua{1}\:\:.\label{eq0EX2}
\end{eqnarray} 
Thus, it follows:
\begin{eqnarray}
\lefteqn{\X_\anchor \Theta_{\ua{1},\va{1}}
  \hat{\X}^\top_{1\shuffle}}\nonumber\\
 & = &
(\ve{x}_{\ua{1}})_\anchor (\ve{x}_{1\va{1}})^\top_\shuffle +
(\ve{x}_{\va{1}})_\anchor (\ve{x}_{1\ua{1}})^\top_\shuffle -
(\ve{x}_{\va{1}})_\anchor (\ve{x}_{1\va{1}})^\top_\shuffle -
(\ve{x}_{\ua{1}})_\anchor (\ve{x}_{1\ua{1}})^\top_\shuffle\nonumber\\
 & = & (\ve{x}_{\ua{1}})_\anchor (\ve{x}_{\vb{1}})^\top_\shuffle +
(\ve{x}_{\va{1}})_\anchor (\ve{x}_{\ub{1}})^\top_\shuffle -
(\ve{x}_{\va{1}})_\anchor (\ve{x}_{\vb{1}})^\top_\shuffle -
(\ve{x}_{\ua{1}})_\anchor
(\ve{x}_{\ub{1}})^\top_\shuffle \label{eq1EX1}\\
 & = & (\ve{x}_{\ua{1}})_\anchor (\ve{x}_{\ua{1}})^\top_\shuffle +
(\ve{x}_{\va{1}})_\anchor (\ve{x}_{\va{1}})^\top_\shuffle -
(\ve{x}_{\va{1}})_\anchor (\ve{x}_{\ua{1}})^\top_\shuffle -
(\ve{x}_{\ua{1}})_\anchor
(\ve{x}_{\va{1}})^\top_\shuffle\label{eq1EX2}\\
 & = & (\ve{x}_{\ua{1}}-\ve{x}_{\va{1}})_\anchor(\ve{x}_{\ua{1}}-\ve{x}_{\va{1}})_\shuffle^\top\:\:.
\end{eqnarray}
In eq. (\ref{eq1EX1}), we have used eqs (\ref{eqm1EX1}, \ref{eqm1EX2})
and in eq. (\ref{eq1EX2}), we have used eqs (\ref{eq0EX1}, \ref{eq0EX2}).
\end{example}

\noindent \textbf{Key matrices} --- The proof of our main helper Theorem is relatively heavy in linear
algebra notations: for example, it involves $T$ double applications
of Sherman-Morrison's inversion Lemma. We now define a series of
matrices and vectors that will be most useful to simplify notations
and proofs. 
Letting $b \defeq 8n\gamma$ (where $\gamma$ is the parameter of the
Ridge regularization in our Taylor loss, see Lemma \ref{lem11} below), we first define the matrix we will
use most often:
\begin{eqnarray}
\matrice{v}_t &\defeq & \left( \hat{\X}_t \hat{\X}_t^\top + b\cdot \Gamma
  \right)^{-1} \:\:, t = 0, 1, ..., T\:\:,\label{defVT}
\end{eqnarray}
where $\Gamma$ is the 
Ridge regularization parameter matrix in our Taylor loss, see Lemma \ref{lem11} below.
Another matrix $\matrice{u}_t$, quantifies precisely the local
mistake made by each elementary permutation. To define it, we first
let (for $t = 1, 2, ..., T$):
\begin{eqnarray}
\ve{a}_t & \defeq & (\ve{x}_{\ua{t}}-\ve{x}_{\va{t}})_\anchor\:\:,\label{defAT}\\
\ve{b}_t & \defeq &
(\ve{x}_{\ub{t}}-\ve{x}_{\vb{t}})_\shuffle\:\:.\label{defST}
\end{eqnarray}
Also, let (for $t = 1, 2, ..., T$)
\begin{eqnarray}
\ve{a}^+_t & \defeq& \left[
\begin{array}{c}
(\ve{x}_{\ua{t}}-\ve{x}_{\va{t}})_\anchor\\\cline{1-1}
\ve{0}
\end{array}
\right] \in \mathbb{R}^d \:\:,\label{defAPLUST}\\
\ve{b}^+_t & \defeq &  \left[
\begin{array}{c}
\ve{0} \\\cline{1-1}
(\ve{x}_{\ub{t}} -\ve{x}_{\vb{t}})_\shuffle
\end{array}
\right] \in \mathbb{R}^d \:\:, \label{defBPLUST}
\end{eqnarray}
and finally (for $t = 1, 2, ..., T$),
\begin{eqnarray}
c_{0,t}
& \defeq & {\ve{a}^+_t}^\top  \matrice{v}_{t-1}
    \ve{a}^+_t \:\:,\\
c_{1,t} & \defeq & {\ve{a}^+_t}^\top \matrice{v}_{t-1}
    \ve{b}^+_t \:\:,\\
c_{2,t} & \defeq & {\ve{b}^+_t}^\top  \matrice{v}_{t-1}
    \ve{b}^+_t \:\:.
\end{eqnarray}
We now define $\matrice{u}_t$ as the following block matrix:
\begin{eqnarray}
\matrice{u}_t & \defeq & \frac{1}{(1-c_{1,t})^2-c_{0,t}c_{2,t}}\cdot \left[\begin{array}{c|c}
c_{2,t} \cdot \ve{a}_t \ve{a}_t^\top & (1-c_{1,t}) \cdot \ve{a}_t
\ve{b}_t^\top \\ \cline{1-2}
(1-c_{1,t}) \cdot \ve{b}_t
\ve{a}_t^\top & c_{0,t} \cdot \ve{b}_t\ve{b}_t^\top 
\end{array}\right]\:\:, t = 1, 2, ..., T\:\:.\label{defUT}
\end{eqnarray}
$\matrice{u}_t$ can be computed only when
$(1-c_{1,t})^2 \neq c_{0,t}c_{2,t}$. This shall be the subject of the
\textit{invertibility} assumption below.
Hereafter, we suppose without loss of generality that $\ve{b}_t \neq
\ve{0}$, since otherwise permutations would make no mistakes on the
shuffle part.

There is one important thing to remark on $\matrice{u}_t$: it is
defined from the indices $\ua{t}$ and $\va{t}$ in $\anchor$ that are affected by
$\PERM_{t}$. Hence, $\matrice{u}_1$ collects the two first such
indices (see Figure \ref{fig:perm}).
We also define matrix $\Lambda_t$ as follows:
\begin{eqnarray}
\Lambda_t & \defeq & 2 \matrice{v}_{t} \matrice{u}_{t+1} \:\:, t = 0,
1, ..., T-1\:\:.\label{defL1}
\end{eqnarray}
To finishup with matrices, we define a doubly indexed matrices that
shall be crucial to our proofs, $\matrice{h}_{i,j}$ for $0\leq
j\leq i \leq T$:
\begin{eqnarray}
\matrice{h}_{i,j} & \defeq & \left\{
\begin{array}{ccl}
\prod_{k=j}^{i-1}(\matrice{i}_d + \Lambda_k) & \mbox{ if }
& 0\leq j < i\\
\matrice{i}_d & \mbox{ if }
& j = i
\end{array}
\right.\:\:.\label{defHIJ}
\end{eqnarray}

\noindent \textbf{Key vectors} --- we let
\begin{eqnarray}
\ve{\epsilon}_{t} & \defeq & \ve{\mu}_{t+1} -
  \ve{\mu}_{t} \:\:, t = 0, 1, ..., T-1\:\:,\label{defEPSILONT}
\end{eqnarray}
which is the difference between two successive mean operators, and
\begin{eqnarray}
\ve{\lambda}_t & \defeq & 2 \matrice{v}_{t+1} \ve{\epsilon}_{t}\:\:, t
= 0, 1, ..., T-1\:\:.\label{defL2}
\end{eqnarray}
\begin{figure}[t]
\centering
\includegraphics[trim=30bp 540bp 630bp
30bp,clip,width=.60\linewidth]{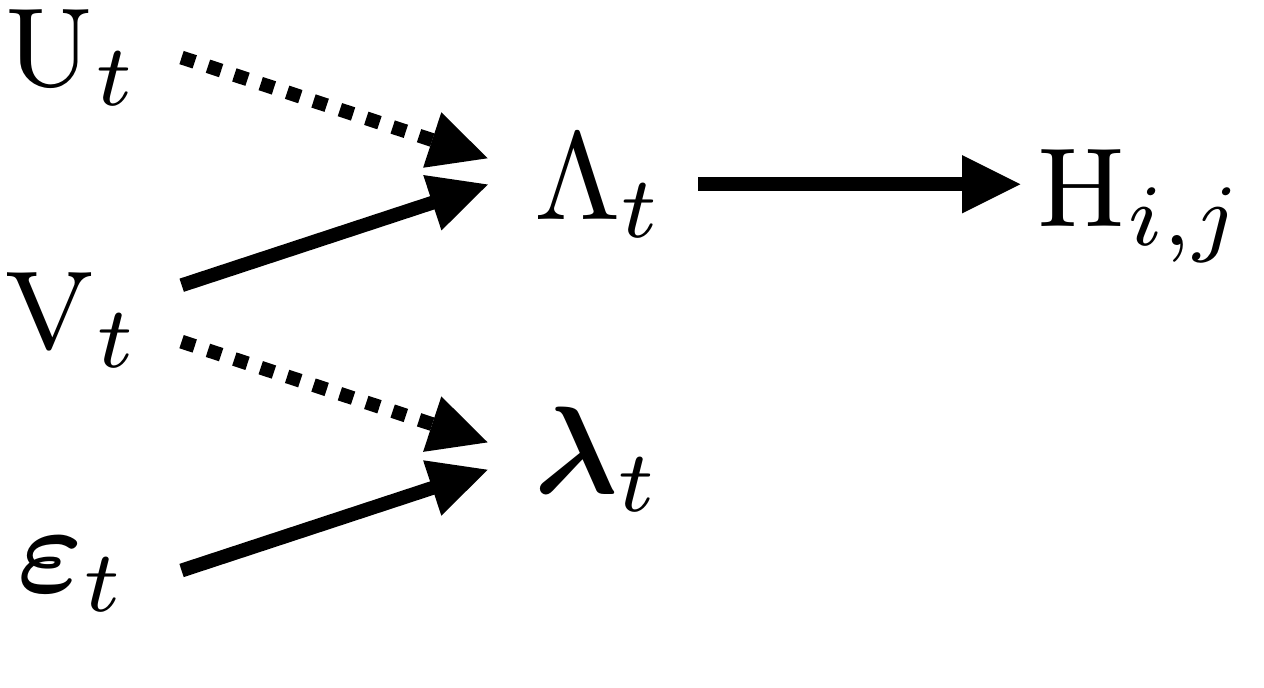} 
\caption{Summary of our key notations on matrices and vectors, and
  dependencies. The dashed arrow means indexes do not match (eq. (\ref{defL2})).\label{fig:notations}}
\end{figure}
Figure \ref{fig:notations} summarizes our key notations in this
Section. We are now ready to proceed through the proof of our key helper Theorem.

\subsection{Helper Theorem}\label{app:proof-thEXACT1}

In this Section, we first show (Theorem \ref{thmeq1} below) that under
lightweight assumptions to ensure the existence of $\matrice{v}_t$,
the difference between two successive optimal classifiers in the
progressive computation of the overall permutation matrix that
generates the errors is \textit{exactly} given by:
\begin{eqnarray}
\ve{\theta}^*_{t+1} - \ve{\theta}^*_t & = & 2\cdot \matrice{v}_t
\matrice{u}_{t+1} \ve{\theta}^*_t + 2\cdot \matrice{v}_{t+1}
\ve{\epsilon}_{t}\nonumber\\
 & = & \Lambda_t \ve{\theta}^*_t
+ \ve{\lambda}_t\:\:, \forall t\geq 0\:\:,\label{sumeq1}
\end{eqnarray}
where $\Lambda_t, \ve{\epsilon}_{t}, \ve{\lambda}_t$ are defined in
eqs (\ref{defEPSILONT}, \ref{defL1},
\ref{defL2}). This holds regardless of the permutation matrices in the
sequence.

We start by the trivial solutions to the minimization of the Taylor loss for
all $t = 1, 2, ..., T$.
\begin{lemma}\label{lem11}
Let 
\begin{eqnarray}
\ell_{S, \gamma}(\ve{\theta}) & \defeq & \log 2 - \frac{1}{n} \sum_i \left\{\frac{1}{2}\cdot
y_i \ve{\theta}^\top \ve{x}_i - \frac{1}{8}\cdot (\ve{\theta}^\top
\ve{x}_i)^2\right\} + \gamma \ve{\theta}^\top \Gamma \ve{\theta}
\end{eqnarray}
denote the $\gamma$-Ridge regularized Taylor loss for set $S$ on classifier $\ve{\theta}$. Then
\begin{eqnarray}
\ve{\theta}^*(\hat{S}) & \defeq & \arg\min_{\ve{\theta}} \ell_{\hat{S},
  \gamma}(\ve{\theta})\nonumber\\
 & = & 2 \cdot \left( \hat{\X} \hat{\X}^\top + b\cdot \Gamma
  \right)^{-1} \ve{\mu}(\hat{S})\:\:,
\end{eqnarray}
with $b \defeq
8n\gamma$. More generally, if we let $\ve{\mu}_t$ denote the
mean operator for set $S_t$, then the optimal classifier for
$\ell_{\hat{S}_t, \gamma}$ is given by
\begin{eqnarray}
\ve{\theta}^*_t  \defeq \ve{\theta}^*(\hat{S}_t) & = & 2 \cdot \matrice{v}_t \ve{\mu}_t\:\:.
\end{eqnarray}
\end{lemma}
(proof straightforward)
\begin{lemma}\label{lemCONDVT}
Suppose $\matrice{v}_{t-1}$ exists. Then $\matrice{v}_{t}$ exists iff
the following holds true:
\begin{eqnarray}
\left\{
\begin{array}{rcl}
c_{1,t} & \neq & 1 \:\:,\label{cond1}\\
(1-c_{1,t})^2 & \neq & c_{0,t}c_{2,t}\:\:.\label{cond2}
\end{array}
\right.
\end{eqnarray}
\end{lemma}
\begin{proof}
We know that $\hat{\X}_{t}$ is obtained from $\hat{\X}_{t-1}$
  after permuting the shuffle part of observations at indexes $\ua{t}$ and $\va{t}$ in
  $\hat{\X}_{(t-1) {\shuffle}}$ by $\PERM_{t}$ (see Figure \ref{fig:perm}). So,
\begin{eqnarray}
\hat{\X}_{t{\shuffle}} & = & \hat{\X}_{(t-1){\shuffle}} + \hat{\X}_{(t-1){\shuffle}} (\PERM_t -
\matrice{i}_n)\nonumber\\
 & = & \hat{\X}_{(t-1){\shuffle}} + \hat{\X}_{(t-1){\shuffle}}
 (\ve{1}_{\ua{t}}\ve{1}^\top_{\va{t}} + \ve{1}_{\va{t}}\ve{1}^\top_{\ua{t}} - \ve{1}_{\va{t}}\ve{1}^\top_{\va{t}} - \ve{1}_{\ua{t}}\ve{1}^\top_{\ua{t}})\:\:,\label{eqth1PR1}
\end{eqnarray}
where $\ve{1}_u \in \mathbb{R}^n$ is the $u^{th}$ canonical basis
vector. We also have 
\begin{eqnarray}
\hat{\X}_{t}\hat{\X}_{t}^\top & = & \left[
\begin{array}{c|c}
\X_\anchor \X_\anchor^\top  & \X_\anchor \hat{\X}_{t\shuffle}^\top \\\cline{1-2}
\hat{\X}_{t\shuffle}\X_\anchor^\top & \hat{\X}_{t\shuffle} \hat{\X}_{t\shuffle}^\top
\end{array}
\right]\nonumber\\
 & = & \left[
\begin{array}{c|c}
\X_\anchor \X_\anchor^\top  & \X_\anchor \hat{\X}_{t\shuffle}^\top \\\cline{1-2}
\hat{\X}_{t\shuffle}\X_\anchor^\top &\hat{\X}_{(t-1)\shuffle} \PERM_{t} \PERM^\top_{t} \hat{\X}_{(t-1)\shuffle}^\top
\end{array}
\right]\nonumber\\
 & = & \left[
\begin{array}{c|c}
\X_\anchor \X_\anchor^\top  & \X_\anchor \hat{\X}_{t\shuffle}^\top \\\cline{1-2}
\hat{\X}_{t\shuffle}\X_\anchor^\top &\hat{\X}_{(t-1)\shuffle} \hat{\X}_{(t-1)\shuffle}^\top
\end{array}
\right]\label{eqth1PR2}\:\:,
\end{eqnarray}
because the
inverse of a permutation matrix is its transpose. We recall that
$\X_\anchor$ does not change throughout permutations, only
$\X_\shuffle$ does. Hence,
\begin{eqnarray}
\hat{\X}_{t}\hat{\X}_{t}^\top & = & \hat{\X}_{t-1}\hat{\X}_{t-1}^\top + \left[
\begin{array}{c|c}
0  & \X_\anchor (\hat{\X}_{t\shuffle}-\X _{(t-1)\shuffle})^\top \\\cline{1-2}
(\hat{\X}_{t\shuffle}-\X _{(t-1)\shuffle}) \X_\anchor^\top & 0
\end{array}
\right]\nonumber\\
 & = & \hat{\X}_{t-1}\hat{\X}_{t-1}^\top + \left[
\begin{array}{c|c}
0  & \X_\anchor \Theta_{\ua{t},\va{t}} \hat{\X}^\top_{(t-1)\shuffle}\\\cline{1-2}
\hat{\X}_{(t-1)\shuffle} \Theta_{\ua{t},\va{t}} \X_\anchor^\top & 0
\end{array}
\right]\:\:,
\end{eqnarray}
with $\Theta_{\ua{t},\va{t}} \defeq
\ve{1}_{\ua{t}}\ve{1}^\top_{\va{t}} +
\ve{1}_{\va{t}}\ve{1}^\top_{\ua{t}} -
\ve{1}_{\va{t}}\ve{1}^\top_{\va{t}} -
\ve{1}_{\ua{t}}\ve{1}^\top_{\ua{t}}$ (symmetric, see
eq. (\ref{eqth1PR1}) and example \ref{exampleEX1}). 
Now, remark that
\begin{eqnarray}
\lefteqn{\X_\anchor \Theta_{\ua{t},\va{t}}
  \hat{\X}^\top_{(t-1)\shuffle}}\nonumber\\
 & = & \X_\anchor (\ve{1}_{\ua{t}}\ve{1}^\top_{\va{t}} +
\ve{1}_{\va{t}}\ve{1}^\top_{\ua{t}} -
\ve{1}_{\va{t}}\ve{1}^\top_{\va{t}} -
\ve{1}_{\ua{t}}\ve{1}^\top_{\ua{t}}) \hat{\X}^\top_{t\shuffle} \nonumber\\
 & = & (\ve{x}_{\ua{t}})_\anchor (\ve{x}_{t\va{t}})_\shuffle^\top +  (\ve{x}_{\va{t}})_\anchor
 (\ve{x}_{t\ua{t}})_\shuffle^\top -  (\ve{x}_{\va{t}})_\anchor
 (\ve{x}_{t\va{t}})_\shuffle^\top -  (\ve{x}_{\ua{t}})_\anchor
 (\ve{x}_{t\ua{t}})_\shuffle^\top\nonumber\\
 & = & (\ve{x}_{\ua{t}})_\anchor (\ve{x}_{\vb{t}})_\shuffle^\top +  (\ve{x}_{\va{t}})_\anchor
 (\ve{x}_{\ub{t}})_\shuffle^\top -  (\ve{x}_{\va{t}})_\anchor
 (\ve{x}_{\vb{t}})_\shuffle^\top -  (\ve{x}_{\ua{t}})_\anchor
 (\ve{x}_{\ub{t}})_\shuffle^\top\label{eqSIMPL1}\\
 & =& -((\ve{x}_{\ua{t}})_\anchor-(\ve{x}_{\va{t}})_\anchor)((\ve{x}_{\ub{t}})_\shuffle-(\ve{x}_{\vb{t}})_\shuffle)^\top\nonumber\\
 & =&
 -(\ve{x}_{\ua{t}}-\ve{x}_{\va{t}})_\anchor(\ve{x}_{\ub{t}}-\ve{x}_{\vb{t}})_\shuffle^\top
 = -\ve{a}_{t} \ve{b}_{t}^\top\:\:.
\end{eqnarray}
Eq. (\ref{eqSIMPL1}) holds because of Lemma \ref{lemUAUB}. We finally get 
\begin{eqnarray}
\hat{\X}_{t}\hat{\X}_{t}^\top & = & \hat{\X}_{t-1}\hat{\X}_{t-1}^\top - \ve{a}^+_t
{\ve{b}^+_t}^\top - {\ve{b}^+_t}
{{\ve{a}^+_t}}^\top\:\:,\label{eqXTXTM1}
\end{eqnarray}
and so we have
\begin{eqnarray}
\matrice{v}_t & = & \left(\matrice{v}^{-1}_{t-1} - \ve{a}^+_t
{\ve{b}^+_t}^\top - {\ve{b}^+_t}
{{\ve{a}^+_t}}^\top\right)^{-1}\label{eqDEFVT}\:\:.
\end{eqnarray}
We analyze when $\matrice{v}_t$ can be computed. First notice that
assuming $\matrice{v}_{t-1}$ exists implies its inverse also exists, and so
\begin{eqnarray}
\mathrm{det}(\matrice{v}^{-1}_{t-1} - \ve{a}^+_t
{\ve{b}^+_t}^\top) & = & \mathrm{det}(\matrice{v}^{-1}_{t-1})\mathrm{det}(\matrice{i}_d - \matrice{v}_{t-1}\ve{a}^+_t
{\ve{b}^+_t}^\top) \nonumber\\
 & = &  \mathrm{det}(\matrice{v}^{-1}_{t-1}) (1 -
 {\ve{b}^+_t}^\top\matrice{v}_{t-1}\ve{a}^+_t) \nonumber\\
 & = &
 \mathrm{det}(\matrice{v}^{-1}_{t-1}) (1-c_{1,t})\:\:,\label{eqSYL0}
\end{eqnarray}
where the last identity comes from Sylvester's determinant
formula. So, if in addition $1-c_{1,t}\neq 0$, then
\begin{eqnarray}
\lefteqn{\mathrm{det}\left(\matrice{v}^{-1}_{t-1} - \ve{a}^+_t
{\ve{b}^+_t}^\top - {\ve{b}^+_t}
{{\ve{a}^+_t}}^\top\right)}\nonumber\\
 & = & \mathrm{det}(\matrice{v}^{-1}_{t-1} - \ve{a}^+_t
{\ve{b}^+_t}^\top) \mathrm{det}\left(\matrice{i}_d - \left(\matrice{v}^{-1}_{t-1} - \ve{a}^+_t
{\ve{b}^+_t}^\top\right) {\ve{b}^+_t}
{{\ve{a}^+_t}}^\top\right) \nonumber\\
 &= & \mathrm{det}(\matrice{v}^{-1}_{t-1}) (1-c_{1,t}) \mathrm{det}\left(\matrice{i}_d - \left(\matrice{v}^{-1}_{t-1} - \ve{a}^+_t
{\ve{b}^+_t}^\top\right)^{-1} {\ve{b}^+_t}
{{\ve{a}^+_t}}^\top\right) \label{eqSYL2}\\
 & = & \mathrm{det}(\matrice{v}^{-1}_{t-1}) (1-c_{1,t}) \left(1 - {{\ve{a}^+_t}}^\top\left(\matrice{v}^{-1}_{t-1} - \ve{a}^+_t
{\ve{b}^+_t}^\top\right)^{-1} {\ve{b}^+_t}\right)  \label{eqSYL3}\\
 & = & \mathrm{det}(\matrice{v}^{-1}_{t-1}) (1-c_{1,t}) \left(1 -
   {{\ve{a}^+_t}}^\top\left(\matrice{v}_{t-1} + \frac{1}{1 -
       {\ve{b}^+_t}^\top \matrice{v}_{t-1}\ve{a}^+_t}\cdot
     \matrice{v}_{t-1}\ve{a}^+_t{\ve{b}^+_t}^\top\matrice{v}_{t-1}\right)
   {\ve{b}^+_t}\right) \label{eqSYL4}\\
 & = & \mathrm{det}(\matrice{v}^{-1}_{t-1}) (1-c_{1,t}) \left(1 -
   c_{1,t} - \frac{c_{0,t} c_{2,t}}{1 -
       c_{1,t}}\right) \nonumber\\
 & = & \frac{ (1 -
   c_{1,t})^2 - c_{0,t} c_{2,t}}{\mathrm{det}(\matrice{v}_{t-1}) }\:\:.
\end{eqnarray}
Here, eq. (\ref{eqSYL2}) comes from
eq. (\ref{eqSYL0}). Eq. (\ref{eqSYL3}) is another application of
Sylvester's determinant formula. Eq. (\ref{eqSYL3}) is
Sherman-Morrison formula. We immediately conclude on Lemma \ref{lemCONDVT}.
\end{proof}
If we now assume without loss of generality that $\matrice{v}_0$
exists --- which boils down to taking $\gamma >0, \Gamma \succ 0$ ---,
then we get the existence of the complete sequence of matrices
$\matrice{v}_{t}$ (and thus the existence of the sequence of optimal classifiers
$\ve{\theta}^*_0, \ve{\theta}^*_1, ...$) provided the following \textbf{invertibility}
condition is satisfied.
\begin{mdframed}[style=MyFrame]
(\textbf{invertibility}) For any $t\geq 1$, $(1-c_{1,t})^2 \not\in \{0, c_{0,t}c_{2,t}\}$.
\end{mdframed}

\begin{theorem}\label{thmeq1}
Suppose the invertibility assumption holds. Then we have:
\begin{eqnarray}
\frac{1}{2}\cdot( \ve{\theta}^*_{t+1} - \ve{\theta}^*_t) & = & \matrice{v}_t \matrice{u}_{t+1} \ve{\theta}^*_t + \matrice{v}_{t+1} \ve{\epsilon}_{t}\nonumber\:\:, \forall t\geq 0\:\:,
\end{eqnarray}
where $\ve{\epsilon}_{t}$ is defined in eq. (\ref{defEPSILONT}).
\end{theorem}
\begin{proof}
We have from Lemma \ref{lem11}, for any $t\geq 1$,
\begin{eqnarray}
\frac{1}{2}\cdot( \ve{\theta}^*_{t} - \ve{\theta}^*_{t-1}) & = & \matrice{v}_{t} \ve{\mu}_{t} - \matrice{v}_{t-1}\ve{\mu}_{t-1}\nonumber\\
 & = & \Delta_{t-1} \ve{\mu}_{t-1} + \matrice{v}_{t}\ve{\epsilon}_{t-1}\:\:,
\end{eqnarray}
with $\Delta_t \defeq \matrice{v}_{t+1} - \matrice{v}_{t}$. It comes
from eq. (\ref{eqXTXTM1}),
\begin{eqnarray}
\Delta_{t-1} & = & \left( \hat{\X}_{t-1} \hat{\X}_{t-1}^\top + b\cdot \Gamma
  - \ve{a}^+_t 
{\ve{b}^+_t}^\top - \ve{b}^+_t
{\ve{a}^+_t}^\top\right)^{-1} -\matrice{v}_{t}\:\:.
\end{eqnarray}
To simplify this expression, we need two consecutive applications of
Sherman-Morrison's inversion formula:
\begin{eqnarray}
\lefteqn{\left(\hat{\X}_{t-1}\hat{\X}_{t-1}^\top + b\cdot \Gamma - \ve{a}^+_t 
{\ve{b}^+_t}^\top - \ve{b}^+_t
{\ve{a}^+_t}^\top\right)^{-1}}\nonumber\\
 & = & \left(\hat{\X}_{t-1}\hat{\X}_{t-1}^\top + b\cdot \Gamma - \ve{a}^+_t 
{\ve{b}^+_t}^\top \right)^{-1} +
  \frac{1}{1-{\ve{a}^+_t}^\top\left( \hat{\X}_{t-1}\hat{\X}_{t-1}^\top + b\cdot \Gamma - \ve{a}^+_t 
{\ve{b}^+_t}^\top \right)^{-1}\ve{b}^+_t} \cdot \Q_t\:\:,\label{sm1}
\end{eqnarray}
with
\begin{eqnarray}
\Q_t & \defeq & \left( \hat{\X}_{t-1}\hat{\X}_{t-1}^\top + b\cdot \Gamma  - \ve{a}^+_t {\ve{b}^+_t}^\top \right)^{-1} \ve{b}^+_t {\ve{a}^+_t}^\top \left( \hat{\X}_{t-1}\hat{\X}_{t-1}^\top + b\cdot \Gamma  - {\ve{a}^+_t} 
{\ve{b}^+_t}^\top \right)^{-1}\:\:,\nonumber\\
\end{eqnarray}
and
\begin{eqnarray}
\left(\hat{\X}_{t-1}\hat{\X}_{t-1}^\top + b\cdot \Gamma - {\ve{a}^+_t} 
{\ve{b}^+_t}^\top \right)^{-1}   & = &
\matrice{v}_{t-1} + \frac{1}{1-{\ve{b}^+_t}^\top  \matrice{v}_{t-1}   {{\ve{a}^+_t}}}\cdot  \matrice{v}_{t-1}{\ve{a}^+_t} {\ve{b}^+_t}^\top  \matrice{v}_{t-1} \:\:.\label{sm2}
\end{eqnarray}
Let us define the
following shorthand:
\begin{eqnarray}
\Sigma_t & \defeq & \matrice{v}_{t-1} + \frac{1}{1-{\ve{b}^+_t}^\top  \matrice{v}_{t-1}  {{\ve{a}^+_t}}}\cdot
  \matrice{v}_{t-1} {\ve{a}^+_t} {\ve{b}^+_t}^\top  \matrice{v}_{t-1}\:\:.
\end{eqnarray}
Then, plugging together eqs. (\ref{sm1}) and
    (\ref{sm2}), we get:
\begin{eqnarray}
\lefteqn{\left(\hat{\X}_{t-1}\hat{\X}_{t-1}^\top + b\cdot \Gamma - {\ve{a}^+_t} 
{\ve{b}^+_t}^\top - \ve{b}^+_t
{{\ve{a}^+_t}}^\top\right)^{-1}}\nonumber\\
  & = & \matrice{v}_{t-1}  + \frac{1}{1-{\ve{b}^+_t}^\top  \matrice{v}_{t-1}
    {{\ve{a}^+_t}}}\cdot  \matrice{v}_{t-1} {\ve{a}^+_t} {\ve{b}^+_t}^\top
  \matrice{v}_{t-1} \nonumber\\
 &  & +
  \frac{1}{1-{{\ve{a}^+_t}}^\top \matrice{v}_{t-1}\ve{b}^+_t -
    \frac{
    {\ve{a}^+_t}^\top \matrice{v}_{t-1} {\ve{a}^+_t} \cdot {\ve{b}^+_t}^\top
    \matrice{v}_{t-1} \ve{b}^+_t}{1-{\ve{b}^+_t}^\top  \matrice{v}_{t-1}
    {{\ve{a}^+_t}}}} \cdot \Sigma_t  \ve{b}^+_t {\ve{a}^+_t}^\top \Sigma_t \nonumber\\
  & = & \matrice{v}_{t-1}  + \frac{1}{1-c_{1,t}}\cdot  \matrice{v}_{t-1} {\ve{a}^+_t} {\ve{b}^+_t}^\top
  \matrice{v}_{t-1} \nonumber\\
 &  & +
  \frac{1}{1- c_{1,t} -
    \frac{c_{0,t} c_{2,t}}{1-c_{1,t}}} \cdot \left( 
\begin{array}{c}
\matrice{v}_{t-1}  \\
+\\
  \frac{1}{1- c_{1,t}}\cdot
  \matrice{v}_{t-1} {\ve{a}^+_t} {\ve{b}^+_t}^\top  \matrice{v}_{t-1}
\end{array}\right) \ve{b}^+_t {\ve{a}^+_t}^\top \left( \begin{array}{c}
\matrice{v}_{t-1}  \\
+\\
  \frac{1}{1- c_{1,t}}\cdot
  \matrice{v}_{t-1} {\ve{a}^+_t} {\ve{b}^+_t}^\top  \matrice{v}_{t-1}
\end{array}
\right) \nonumber\\
 & = & \matrice{v}_{t-1}  + \frac{1}{1-c_{1,t}}\cdot  \matrice{v}_{t-1} {\ve{a}^+_t} {\ve{b}^+_t}^\top
  \matrice{v}_{t-1} + \frac{1}{1- c_{1,t} -
    \frac{c_{0,t} c_{2,t}}{1-c_{1,t}}} \cdot \matrice{v}_{t-1} \ve{b}^+_t {{\ve{a}^+_t}}^\top
  \matrice{v}_{t-1} \nonumber\\
 && + \frac{c_{0,t}}{(1- c_{1,t})^2 -
    c_{0,t} c_{2,t}} \cdot \matrice{v}_{t-1} \ve{b}^+_t {\ve{b}^+_t}^\top
  \matrice{v}_{t-1} + \frac{c_{2,t}}{(1- c_{1,t})^2 -
    c_{0,t} c_{2,t}} \cdot \matrice{v}_{t-1} {\ve{a}^+_t} {{\ve{a}^+_t}}^\top
  \matrice{v}_{t-1} \nonumber\\
 & &+ \frac{c_{0,t} c_{2,t}}{(1- c_{1,t}) ((1- c_{1,t})^2 -
    c_{0,t} c_{2,t})} \cdot \matrice{v}_{t-1} {\ve{a}^+_t} {\ve{b}^+_t}^\top
  \matrice{v}_{t-1}\nonumber\\
 & = & \matrice{v}_{t-1}  + \frac{1-c_{1,t}}{(1-c_{1,t})^2-c_{0,t}c_{2,t}}\cdot  \left(\matrice{v}_{t-1} {\ve{a}^+_t} {\ve{b}^+_t}^\top
  \matrice{v}_{t-1} + \matrice{v}_{t-1} \ve{b}^+_t {{\ve{a}^+_t}}^\top
  \matrice{v}_{t-1} \right) \nonumber\\
 & & + \frac{c_{0,t}}{(1- c_{1,t})^2 -
    c_{0,t} c_{2,t}} \cdot \matrice{v}_{t-1} \ve{b}^+_t {\ve{b}^+_t}^\top
  \matrice{v}_{t-1} + \frac{c_{2,t}}{(1- c_{1,t})^2 -
    c_{0,t} c_{2,t}} \cdot \matrice{v}_{t-1} {\ve{a}^+_t} {{\ve{a}^+_t}}^\top
  \matrice{v}_{t-1} \nonumber\\
 & = & \matrice{v}_{t-1}  + \frac{1}{(1-c_{1,t})^2-c_{0,t}c_{2,t}}\cdot \left\{
\begin{array}{c}
   (1-c_{1,t})\cdot (\matrice{v}_{t-1} {\ve{a}^+_t} {\ve{b}^+_t}^\top
  \matrice{v}_{t-1} + \matrice{v}_{t-1} \ve{b}^+_t {{\ve{a}^+_t}}^\top
  \matrice{v}_{t-1}) \\
+ c_{0,t} \cdot \matrice{v}_{t-1} \ve{b}^+_t {\ve{b}^+_t}^\top
  \matrice{v}_{t-1} \\
+ c_{2,t} \cdot \matrice{v}_{t-1} {\ve{a}^+_t} {{\ve{a}^+_t}}^\top
  \matrice{v}_{t-1}
\end{array}\right\}\nonumber\\
  & = & \matrice{v}_{t-1}  + \matrice{v}_{t-1}\matrice{u}_t \matrice{v}_{t-1}\:\:.
\end{eqnarray}
So,
\begin{eqnarray}
\frac{1}{2}\cdot( \ve{\theta}^*_{t} - \ve{\theta}^*_{t-1}) & = &
\Delta_{t-1} \ve{\mu}_{t-1} +
\matrice{v}_{t}\ve{\epsilon}_{t-1}\nonumber\\
 & = & \matrice{v}_{t-1} \matrice{u}_t \matrice{v}_{t-1}\ve{\mu}_{t-1} + \matrice{v}_{t} \ve{\epsilon}_{t-1}\nonumber\\
 & = &  \matrice{v}_{t-1} \matrice{u}_t \ve{\theta}^*_{t-1} + \matrice{v}_{t} \ve{\epsilon}_{t-1} \:\:, \label{eq22}
\end{eqnarray}
as claimed (end of the proof of Lemma \ref{thmeq1}). 
\end{proof}
All that remains to do now is to unravel the relationship in Theorem \ref{thmeq1}
and quantify the exact variation $\ve{\theta}^*_{T} -
\ve{\theta}^*_0$ as a function of $\ve{\theta}^*_0$ (which is the
error-free optimal classifier), holding for any permutation
$\PERM_*$. We therefore suppose that the invertibility assumption holds.
\begin{theorem}\label{thEXACT}
Suppose the invertibility assumption holds. For any $T\geq 1$,
\begin{eqnarray}
\ve{\theta}^*_{T} - \ve{\theta}^*_{0} & = & (\matrice{h}_{T,0} - \matrice{i}_d)
\ve{\theta}^*_0 + \sum_{t=0}^{T-1} \matrice{h}_{T,t+1} \ve{\lambda}_{t}\:\:.
\end{eqnarray}
\end{theorem}
\begin{proof}
We recall first that we have from Theorem \ref{thmeq1},
$\ve{\theta}^*_{t+1} - \ve{\theta}^*_t =  \Lambda_t \ve{\theta}^*_t
+ \ve{\lambda}_t$, $\forall t\geq 0$. Equivalently,
\begin{eqnarray}
\ve{\theta}^*_{t+1} & = & (\matrice{i}_d + \Lambda_t) \ve{\theta}^*_t + \ve{\lambda}_t\:\:.
\end{eqnarray}
Unravelling, we easily get $\forall T \geq 1$,
\begin{eqnarray}
\ve{\theta}^*_{T} & = & \prod_{t=0}^{T-1}(\matrice{i}_d + \Lambda_t)
\ve{\theta}^*_0  + \ve{\lambda}_{T-1} + \sum_{j=0}^{T-2} \prod_{t=j+1}^{T-1} (\matrice{i}_d
+ \Lambda_t) \ve{\lambda}_{j}\nonumber\\
 & = & \matrice{h}_{T,0}
\ve{\theta}^*_0 + \sum_{t=0}^{T-1} \matrice{h}_{T,t+1} \ve{\lambda}_{t}\:\:,
\end{eqnarray}
which yields the statement of Theorem \ref{thEXACT}.
\end{proof}
Since it applies to every permutation matrix, it applies to every
entity resolution algorithm. Theorem \ref{thEXACT} gives us a
interesting expression for the deviation $\ve{\theta}^*_{T} -
\ve{\theta}^*_0 $ which can be used to derive bounds on the distance
between the two classifiers, even outside our privacy framework. We apply it
below to derive one such bound.

\subsection{Assumptions: details and discussion}\label{app:assum}

Our result relies on several assumptions that concern the permutation
$\PERM_*$ and its decomposition in eq. (\ref{productperm}), as well as
on the data
size $n$ and regularization parameters $\gamma, \Gamma$. We make in
this Section more extensive comments on the assumptions we use.

\begin{mdframed}[style=MyFrame]
(\textbf{$(\epsilon, \tau)$-accuracy}) Equivalently, $\PERM_t$ is
$(\epsilon, \tau)$-accurate iff both conditions below are satisfied:
\begin{enumerate}
\item the stretch in the \textit{shuffle} space of errors on an observation due
  to permutations is bounded by the \textit{total} stretch of the observation:
\begin{eqnarray}
\vstretch((\hat{\ve{x}}_{ti} - \ve{x}_{i})_\shuffle,\ve{w}_\shuffle) &
\leq & \epsilon \cdot \vstretch(\ve{x}_{i},\ve{w})+
\tau \:\:, \forall i \in [n]\:\:, \forall \ve{w}\in \mathbb{R}^d : \|\ve{w}\|_2 = 1\:\:.\label{eqconst1}
\end{eqnarray}
\item recall that $\ve{x}_{\ua{t}}, \ve{x}_{\va{t}}$ are the
  observations in $\X$ whose shuffle parts are \textit{affected} by $\PERM_t$,
  and $\ve{x}_{\ub{t}}, \ve{x}_{\vb{t}}$ are the observations in $\X$
  whose shuffle parts are \textit{permuted} by $\PERM_t$. Then the
  stretch of the errors due to $\PERM_t$, $(\ve{x}_{\ua{t}} -
\ve{x}_{\va{t}})_\anchor$ and $(\ve{x}_{\ub{t}} -
\ve{x}_{\vb{t}})_\shuffle$, 
is bounded by the \textit{maximal} stretch of the related
observations:
\begin{eqnarray}
\vstretch((\ve{x}_{\uf{t}} -
\ve{x}_{\vf{t}})_{\F},\ve{w}_{\F}) & \leq & \epsilon \cdot \max_{i\in
  \{\uf{t}, \vf{t}\}} \vstretch(\ve{x}_{i},\ve{w})+
\tau \:\:, \nonumber\\
 & & \forall \F \in \{\anchor, \shuffle\}, \forall \ve{w}\in \mathbb{R}^d :
\|\ve{w}\|_2 = 1 \:\:.
\end{eqnarray}
\end{enumerate}
\end{mdframed}
\begin{figure}[t]
\centering
\includegraphics[trim=40bp 400bp 560bp
70bp,clip,width=.60\linewidth]{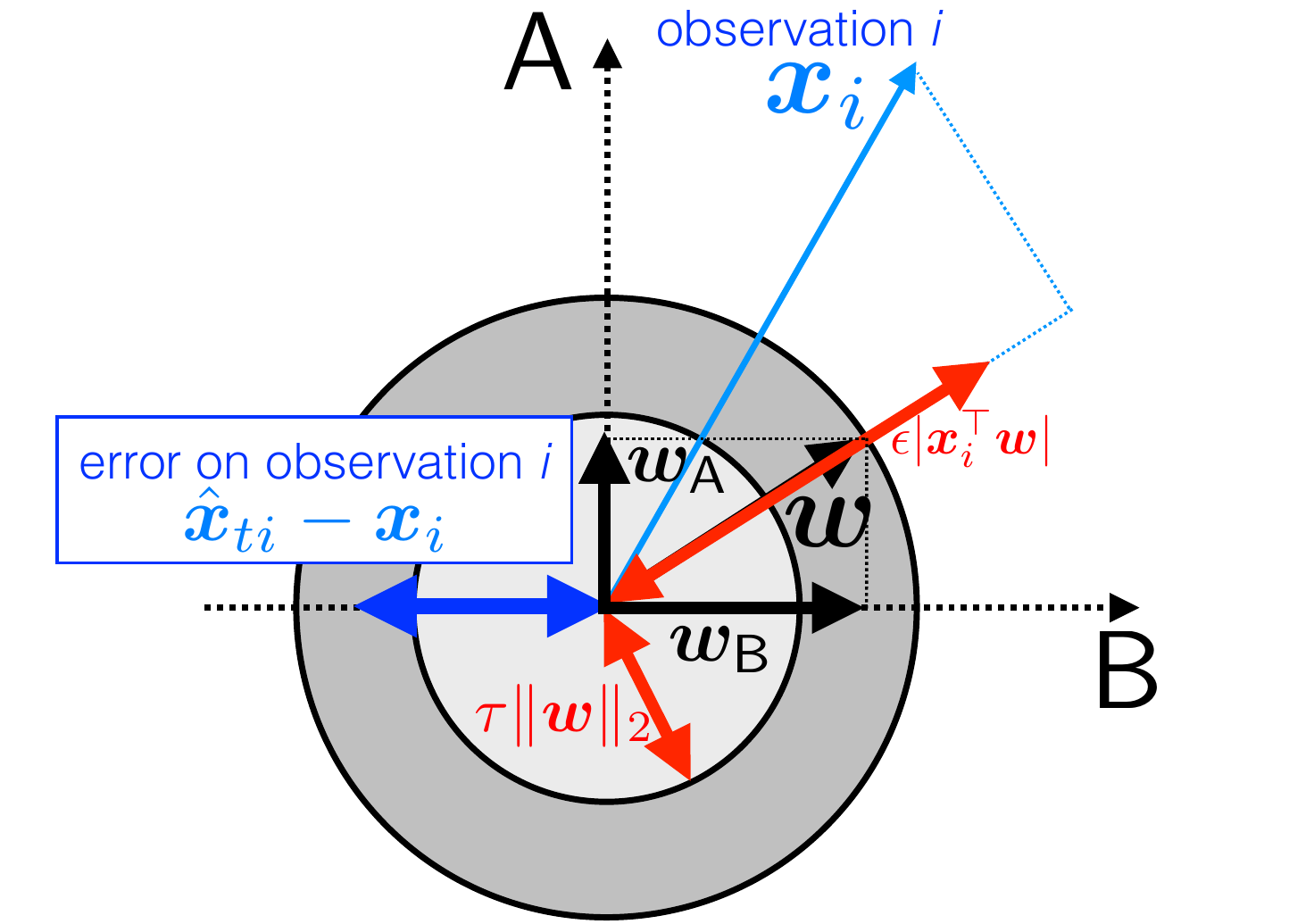} 
\caption{If $\PERM_t$ is $(\epsilon, \tau)$-accurate, eq. (\ref{defACCURATE1})
  in Definition \ref{defACCURATE} reads as follows: for any vector
  $\ve{w}$, the norm of the error projected on $\ve{w}$, \textit{in
    the shuffle space}, is no more than a fraction of the norm of
  $\ve{w}$ (red) plus a fraction of the norm of the projection of the
  observation along $\ve{w}$ \textit{in the complete feature space}.\label{fig:Hyp}
}
\end{figure}

\noindent $\hookrightarrow$ Remarks on Definition \ref{defACCURATE}:\\
\noindent \textbf{Remark 1}: eq. (\ref{defACCURATE1})
in the $(\epsilon, \tau)$-accuracy assumption imposes that the stretch of all errors in $\hat{\X}_{t}$, $\hat{\ve{x}}_{ti} -
\ve{x}_{i}$, along direction $\ve{w}$, be bounded by the stretch of the corresponding
observations in $\X$ along the same direction. Since the anchor parts of
$\hat{\ve{x}}_{ti}$ and $\ve{x}_{i}$ coincide by
convention, the stretch of the error is in fact that measured on
the shuffle set of features. Figure \ref{fig:Hyp} gives a visual for
that.

\begin{figure}[t]
\centering
\includegraphics[trim=25bp 280bp 270bp
10bp,clip,width=.70\linewidth]{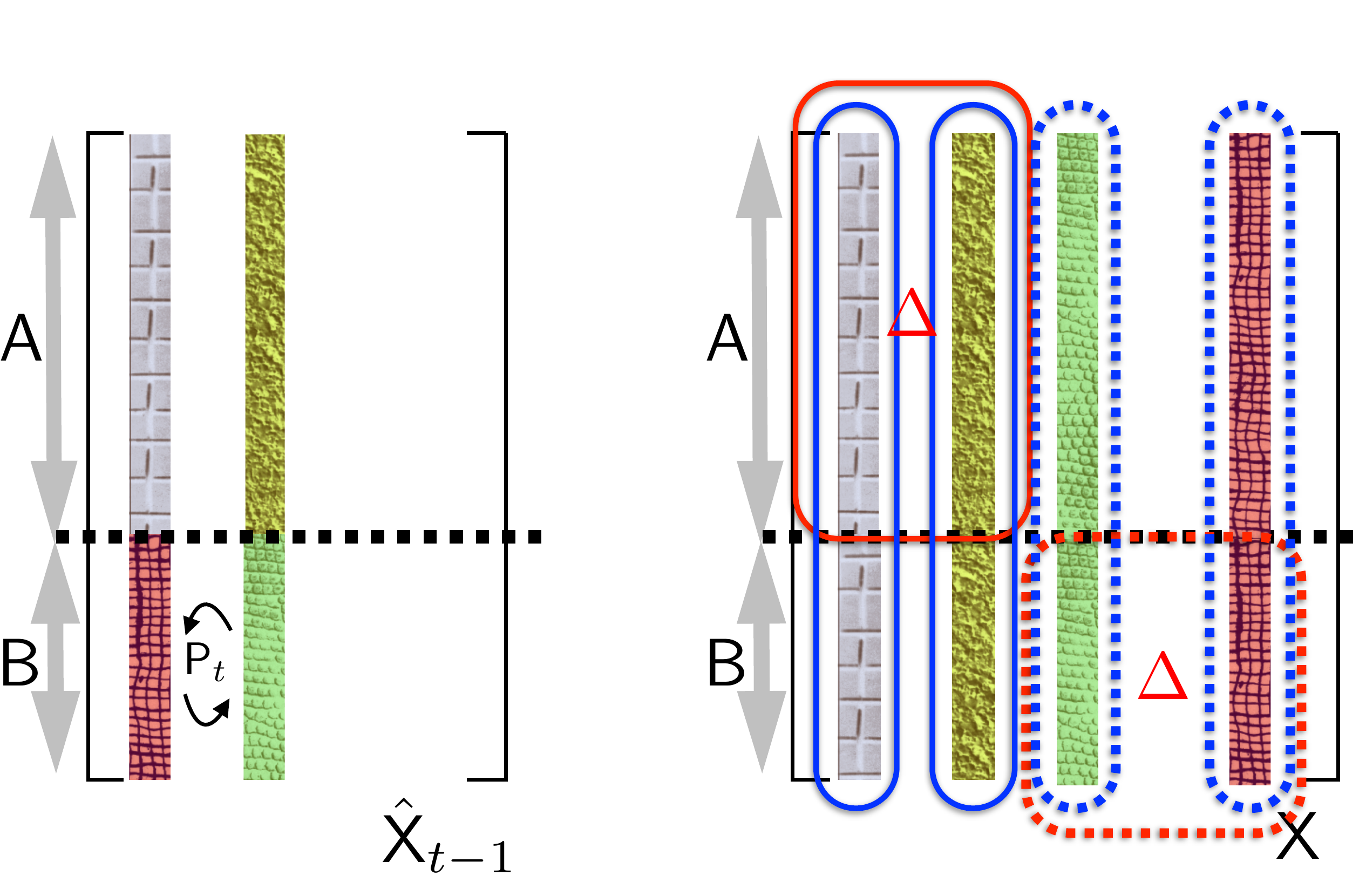} 
\caption{Overview of eq. (\ref{defACCURATE2}) for the $(\epsilon, \tau)$-accuracy assumption for elementary
  permutation $\PERM_t$. \textit{Left}: effect of
  $\PERM_t$ on $\hat{\X}_{t-1}$; \textit{right}: corresponding vectors
  whose stretch is used in the assumption. Informally, the assumption
  says that the stretch of the (dashed, resp. plain) red
  vectors is bounded by the maximal stretch among the (dashed, resp. plain) blue vectors (best viewed in
  color). \label{fig:assum1}}
\end{figure}
\noindent \textbf{Remark 2}:  Figure \ref{fig:assum1} illustrates the
second part of the $(\epsilon, \tau)$-accuracy
assumption. Infomally, stretch of errors due to $\PERM_t$ is
bounded by the maximal stretch of the related observations affected by $\PERM_t$.\\
\noindent \textbf{Remark 3}: parameter $\tau$ is necessary in some
way: if one picks $\ve{w}$ orthogonal to $\ve{x}_{i}$, then
$\vstretch(\ve{x}_{i},\ve{w}) = \|\ve{x}_i\|_2 |\cos(\ve{x}_i, \ve{w})| = 0$ while it may be the
case that
that $\vstretch((\hat{\ve{x}}_{ti} - \ve{x}_{i})_\shuffle,\ve{w}_\shuffle) = \|(\hat{\ve{x}}_{ti} - \ve{x}_{i})_\shuffle\|_2
|\cos((\hat{\ve{x}}_{ti} - \ve{x}_{i})_\shuffle, \ve{w}_\shuffle)|
\neq 0$.\\

\noindent $\hookrightarrow$ Remarks on Definition \ref{def:DMC}:\\
\textbf{Remark 1}: both conditions are in fact weak. They barely
impose that the strength of regularization ($\gamma
  \lambda_1^\uparrow(\Gamma)$) is proportional to a squared norm, while data
size is no less than a potentially small constant.\\
\noindent \textbf{Remark 2}:  $X_*^2$ can be of the same order as
$\inf_{\ve{w}} \sigma^2(\{\vstretch(\ve{x}_i,\ve{w})\}_{i=1}^n)$: indeed, if all observations have the same
norm in $\mathbb{R}^d$, then $\sigma^2(\{\vstretch(\ve{x}_i,\ve{w})\}_{i=1}^n) = \beta(\ve{w})
X_*^2$ for some $0\leq \beta \leq 1$. So the problem can be reduced to
roughly having 
\begin{eqnarray}
\gamma
  \lambda_1^\uparrow(\Gamma) & \geq & \left(1-
    \frac{(1-\epsilon)^2}{8}\cdot \beta(\ve{w})\right)
\cdot X_*^2\:\:, \forall \ve{w}\in \mathbb{R}^d\:\:,
\end{eqnarray}
which also shows that $\gamma
  \lambda_1^\uparrow(\Gamma)$ has to be homogeneous to a
square norm (\textit{Cf} Remark 1).

\subsection{Proof of Theorem \ref{thAPPROX1}}\label{app:proof-thAPPROX1}

First, we shall see (Corollary \ref{corBOUNDCT} below) that the assumptions
made guarantee the invertibility condition in Theorem \ref{thEXACT}. 
From Theorem \ref{thEXACT}, we now investigate a general bound of the
type
\begin{eqnarray}
\|\ve{\theta}^*_{T} - \ve{\theta}^*_0 \|_2 & \leq & a\cdot
\|\ve{\theta}^*_0 \|_2 + b\:\:,\forall T\geq 1\:\:.
\end{eqnarray}
where $a\geq 0$ and $b\geq 0$ are two reals that we want as small as
possible.
We first need an intermediate technical Lemma. Let $\mu(\{a_i\}) \defeq (1/n) \cdot
\sum_i a_i$ denote for short the average of $\{a_i\}$ with $a_i\geq 0, \forall i$. Let
$\gamma'\geq 0$ be \textit{any} real such that:
\begin{eqnarray}
\frac{\mu^2(\{a_i\})}{\mu(\{a^2_i\})} & \leq & (1-\gamma')\:\:.\label{eq0001}
\end{eqnarray}
Remark that the result is true for $\gamma' = 0$ since
$\mu(\{a^2_i\})-\mu^2(\{a_i\})$ is just the variance of $\{a_i\}$,
which is non-negative.
Remark also that we must have $\gamma'\leq 1$.
\begin{lemma}\label{lemsim}
$\sum_i \left( (1-\epsilon) a_i - q \right)^2 \geq \gamma'
(1-\epsilon)^2 \sum_i a_i^2$, $\forall \epsilon \leq 1, q\in \mathbb{R}$.
\end{lemma}
\begin{proof}
Remark that 
\begin{eqnarray}
\sqrt{(1-\gamma') \mu(\{a^2_i\})} & = & \inf_{k\geq 0} \frac{1}{2}\cdot \left(k
+ \frac{1}{k}\cdot (1-\gamma') \mu(\{a^2_i\})\right)\:\:,\label{eqLG3}
\end{eqnarray}
so we have:
\begin{eqnarray}
\mu(\{a_i\}) & \leq & \sqrt{(1-\gamma')
  \mu(\{a^2_i\})} \label{eqLG1}\\
 & \leq & \frac{q}{2(1-\epsilon)} +
\frac{(1-\gamma')(1-\epsilon)}{2q} \cdot \mu(\{a^2_i\})\:\:.\label{eqLG2}
\end{eqnarray}
Ineq. (\ref{eqLG1}) holds because of ineq. (\ref{eq0001}) and
ineq. (\ref{eqLG1}) holds because of eq. (\ref{eqLG3}) and
substituting $k \defeq q / (1-\epsilon) \geq 0$. After
reorganising, we obtain:
\begin{eqnarray}
n q^2 - 2(1-\epsilon) q \sum_i a_i +(1-\gamma')
(1-\epsilon)^2 \sum_i a_i^2 & \geq & 0\:\:,
\end{eqnarray}
and so we obtain the inequality of:
\begin{eqnarray}
\sum_i \left( (1-\epsilon) a_i - q \right)^2 & = & n q^2 -
2(1-\epsilon) q \sum_i a_i + (1-\epsilon)^2 \sum_i a_i^2
\nonumber\\
 & \geq & \gamma' (1-\epsilon)^2 \sum_i a_i^2\:\:,\label{eqa01}
\end{eqnarray}
which allows to conclude the proof of Lemma \ref{lemsim}.
\end{proof}
Let
\begin{eqnarray}
\gamma'(\X,\ve{w}) & \defeq & 1 - \frac{\mu^2(\{\vstretch(\ve{x}_i,\ve{w})\}_{i=1}^n) }{\mu(\{\vstretch^2(\ve{x}_i,\ve{w})\}_{i=1}^n)}\:\:,
\end{eqnarray}
where we recall that $\mu(\{a_i\})$ is the average in set
$\{a_i\}$. It is easy to remark that $\gamma'(\X,\ve{w}) \in [0,1]$
and it can be used in Lemma \ref{lemsim} for the choice
\begin{eqnarray}
\{a_i\} & \defeq & \{\vstretch(\ve{x}_i,\ve{w})\}_{i=1}^n\:\:.
\end{eqnarray}
It is also not hard to see that as long as there exists two $\ve{x}_i$
in $\X$
with a different \textit{direction}, we shall have $\gamma'(\X,\ve{w})
> 0, \forall \ve{w}$.

Following \cite{bMA}, for any symmetric matrix $\matrice{m}$, we let
$\ve{\lambda}^\downarrow(\matrice{m})$ (resp. $\ve{\lambda}^\uparrow(\matrice{m})$) denote the vector of eigenvalues
arranged in decreasing (resp. increasing) order. So, $\lambda_1^\downarrow(\matrice{m})$
(resp. $\lambda_1^\uparrow(\matrice{m})$)
denotes the maximal (resp. minimal) eigenvalue of $\matrice{m}$.
\begin{lemma}\label{lemV2}
For any set $\mathcal{S} \defeq \{a_i\}_{i=1}^n$, let
$\mu(\mathcal{S})$ and $\sigma(\mathcal{S})$
denote the mean and standard deviation of $\mathcal{S}$. If
$\PERM_t$ is $(\epsilon,
\tau)$-accurate, then the eigenspectrum of $\matrice{v}_t$ is
bounded as below:
\begin{eqnarray}
\lambda_1^\downarrow(\matrice{v}_t) & \leq & \frac{1}{n}\cdot \frac{1}{ (1-\epsilon)^2
  \cdot \inf_{\ve{w}}\sigma^2(\{\vstretch(\ve{x}_i,\ve{w})\}_{i=1}^n)
  + 8\gamma \lambda_1^\uparrow(\Gamma)} \label{eq002F}\:\:,\\
\lambda_1^\uparrow(\matrice{v}_t) & \geq & \frac{1}{2n}\cdot \frac{1}{ (1+\epsilon)^2
  \cdot \sup_{\ve{w}}\mu(\{\vstretch^2(\ve{x}_i,\ve{w})\}_{i=1}^n)
  + \tau^2 + 4\gamma \lambda_1^\downarrow(\Gamma)} \label{eq002G}\:\:.
\end{eqnarray}
\end{lemma}
\begin{proof}
We first show the upperbound on $\lambda_1^\downarrow(\matrice{v}_t)$,
and we start by showing that
\begin{eqnarray}
\lambda_1^\downarrow(\matrice{v}_t) & \leq & \frac{1}{n}\cdot \frac{1}{ (1-\epsilon)^2
  \cdot \inf_{\ve{w}} \gamma'(\X, \ve{w}) \varsigma(\X, \ve{w})  + 8\gamma \lambda_1^\uparrow(\Gamma)} \label{eq002}\:\:,
\end{eqnarray}
with 
\begin{eqnarray}
\varsigma(\X, \ve{w}) & \defeq & \frac{1}{n} \cdot \sum_i \vstretch^2(\ve{x}_i,\ve{w})\:\:.
\end{eqnarray}
If $\PERM_t$ is $(\epsilon, \tau)$-accurate, it comes from the
triangle inequality 
\begin{eqnarray}
|\hat{\ve{x}}_{ti}^\top
\ve{w}| &  = & |\ve{x}_{i}^\top
\ve{w} + (\ve{x}_{ti_\shuffle} - \ve{x}_{i_\shuffle})^\top
\ve{w}_\shuffle| \nonumber\\
 & \geq & |\ve{x}_{i}^\top
\ve{w}| - |(\ve{x}_{ti_\shuffle} - \ve{x}_{i_\shuffle})^\top
\ve{w}_\shuffle| \nonumber\\
 & \geq & (1-\epsilon) |\ve{x}_i^\top
\ve{w}| - \tau\|\ve{w}\|_2 \:\:,
\end{eqnarray}
and so, using Lemma \ref{lemsim} with $q \defeq \tau$ and 
$a_i \defeq \|\ve{x}_i\|_2 |\cos(\ve{x}_{i}, \ve{w})|$, we obtain the
last inequality of:
\begin{eqnarray}
\| \hat{\X}_{t}^\top \ve{w} \|_2^2 & = &  \sum_i (\hat{\ve{x}}_i^\top
\ve{w})^2 \nonumber\\
 & \geq & \sum_i ((1-\epsilon) |\ve{x}_i^\top
\ve{w}| - \tau\|\ve{w}\|_2)^2\nonumber\\
 & & = \|\ve{w}\|_2^2 \cdot \sum_i ((1-\epsilon) \vstretch(\ve{x}_i,\ve{w}) - \tau)^2\nonumber\\
& \geq & \|\ve{w}\|_2^2 (1-\epsilon)^2 \cdot \gamma'(\X, \ve{w})\sum_i \vstretch^2(\ve{x}_i,\ve{w}) \:\:.\label{eqb1}
\end{eqnarray}
Therefore, if $\PERM_t$ is $(\epsilon, \tau)$-accurate, we have
\begin{eqnarray}
\lambda_1^\downarrow(\matrice{v}_t) & \defeq & \left(\inf_{\ve{w}} \frac{\ve{w}^\top\left( \hat{\X}_{t} \hat{\X}_{t}^\top + b\cdot \Gamma
  \right) \ve{w}}{\|\ve{w}\|_2^2}\right)^{-1}\nonumber\\
 & \leq & \frac{1}{ ((1-\epsilon)^2 \cdot \inf_{\ve{w}} \gamma'(\X, \ve{w})\sum_i \vstretch^2(\ve{x}_i,\ve{w}) + b\lambda_1^\uparrow(\Gamma)}\nonumber\\
 & & = \frac{1}{n}\cdot \frac{1}{(1-\epsilon)^2 \inf_{\ve{w}}
   \sigma^2(\X, \ve{w})  + 8\gamma \lambda_1^\uparrow(\Gamma)}\:\:,
\end{eqnarray}
where the last identity follows from
\begin{eqnarray}
\inf_{\ve{w}} \sigma^2(\X, \ve{w}) & \defeq & \inf_{\ve{w}}  \mu(\{\vstretch^2(\ve{x}_i,\ve{w})\}) - \mu^2(\{\vstretch(\ve{x}_i,\ve{w})\}) \nonumber\\
 & = & \inf_{\ve{w}} \left( 1 - \frac{\mu^2(\{\vstretch(\ve{x}_i,\ve{w})\}) }{\mu(\{\vstretch^2(\ve{x}_i,\ve{w})\})}\right)\cdot \mu(\{\vstretch^2(\ve{x}_i,\ve{w})\}) \nonumber\\
 & = & \inf_{\ve{w}} \left( 1 - \frac{\mu^2(\{\vstretch(\ve{x}_i,\ve{w})\}) }{\mu(\{\vstretch^2(\ve{x}_i,\ve{w})\})}\right)\cdot \frac{1}{n}\cdot\sum_i \vstretch^2(\ve{x}_i,\ve{w}) \nonumber\\
 & = & \frac{1}{n}\cdot\inf_{\ve{w}} \gamma'(\X, \ve{w}) \sum_i \vstretch^2(\ve{x}_i,\ve{w})\:\:.
\end{eqnarray}
This finishes the proof for ineq. (\ref{eq002F}). To show
ineq. (\ref{eq002G}), we remark that if $\PERM_t$ is $(\epsilon,
\tau)$-accurate, it also comes from the
triangle inequality 
\begin{eqnarray}
|\hat{\ve{x}}_{ti}^\top
\ve{w}| &  = & |\ve{x}_{i}^\top
\ve{w} + (\ve{x}_{ti_\shuffle} - \ve{x}_{i_\shuffle})^\top
\ve{w}_\shuffle| \nonumber\\
 & \leq & |\ve{x}_{i}^\top
\ve{w}| + |(\ve{x}_{ti_\shuffle} - \ve{x}_{i_\shuffle})^\top
\ve{w}_\shuffle| \nonumber\\
 & \leq & (1+\epsilon) |\ve{x}_i^\top
\ve{w}| + \tau\|\ve{w}\|_2 \:\:,
\end{eqnarray}
and so,
\begin{eqnarray}
\| \hat{\X}_{t}^\top \ve{w} \|_2^2 & = &  \sum_i (\hat{\ve{x}}_i^\top
\ve{w})^2 \nonumber\\
 & \leq & \|\ve{w}\|_2^2 \cdot \sum_i ((1+\epsilon) \vstretch(\ve{x}_i,\ve{w}) + \tau)^2\nonumber\\
& \leq & \|\ve{w}\|_2^2 \cdot \left(2 (1+\epsilon)^2\sum_i
  \vstretch^2(\ve{x}_i,\ve{w}) + 2n \tau^2\right)\:\:,\label{eqb22}
\end{eqnarray}
because $(a+b)^2 \leq 2a^2 + 2b^2$. Therefore, if $\PERM_t$ is $(\epsilon, \tau)$-accurate, we have
\begin{eqnarray}
\lambda_1^\uparrow(\matrice{v}_t) & \defeq & \left(\sup_{\ve{w}} \frac{\ve{w}^\top\left( \hat{\X}_{t} \hat{\X}_{t}^\top + b\cdot \Gamma
  \right) \ve{w}}{\|\ve{w}\|_2^2}\right)^{-1}\nonumber\\
 & \geq & \frac{1}{ 2 (1+\epsilon)^2 \cdot \inf_{\ve{w}} \sum_i
   \vstretch^2(\ve{x}_i,\ve{w}) + 2 n \tau^2 + b\lambda_1^\uparrow(\Gamma)}\nonumber\\
 & & = \frac{1}{2n}\cdot \frac{1}{(1+\epsilon)^2 
   \sup_{\ve{w}}\mu(\{\vstretch^2(\ve{x}_i,\ve{w})\}_{i=1}^n)  + \tau^2 + 4\gamma \lambda_1^\uparrow(\Gamma)}\:\:,
\end{eqnarray}
This ends the proof of Lemma \ref{lemV2}.
\end{proof}

\begin{lemma}\label{lemLAMBDAUT}
Suppose $(1-c_{1,t})^2-c_{0,t}c_{2,t} \neq 0$\footnote{This is implied
by the invertibility assumption.} and $\ve{a}_t
\neq \ve{0}$. Then $\matrice{u}_t$ is negative semi-definite iff
$(1-c_{1,t})^2-c_{0,t}c_{2,t} < 0$. Otherwise, $\matrice{u}_t$ is
indefinite. In all cases, for any $z\in
\{\lambda_1^\downarrow(\matrice{u}_t),
|\lambda_1^\uparrow(\matrice{u}_t)|\}$, we have
\begin{eqnarray}
z & \leq & \frac{2+ 3 (c_{0,t} + c_{2,t})}{2|(1-c_{1,t})^2 - c_{0,t}c_{2,t}|} \cdot \max\{\|\ve{a}_t\|^2_2,
\|\ve{b}_t\|^2_2\}\:\:.\label{ineqLAMBDAGENGEN}
\end{eqnarray}
\end{lemma}
\begin{proof}
Consider a block-vector following the column-block partition of $\matrice{u}_t$,
\begin{eqnarray}
\tilde{\ve{x}} & \defeq & \left[\begin{array}{c}
\ve{x}\\ \hline
\ve{y}
\end{array}\right]\:\:.\label{defVB}
\end{eqnarray}
Denote for short $\rho \defeq (1-c_{1,t})^2-c_{0,t}c_{2,t}$. We have
\begin{eqnarray}
\matrice{u}_t \tilde{\ve{x}} & = & \frac{1}{\rho}\cdot \left[\begin{array}{c}
(c_{2,t} (\ve{a}_t^\top \ve{x}) + (1-c_{1,t})(\ve{b}_t^\top \ve{y}))\cdot \ve{a}_t\\ \hline
((1 - c_{1,t}) (\ve{a}_t^\top \ve{x}) + c_{0,t}(\ve{b}_t^\top \ve{y}))\cdot \ve{b}_t
\end{array}\right]\:\:.\label{eqEIGV}
\end{eqnarray}
We see that the only possibility for $\tilde{\ve{x}}$ to be an
eigenvector is that $\ve{x} \propto \ve{a}_t$ and $\ve{y} \propto
\ve{b}_t$ (including the null vector for at most one vector). We now distinguish two cases.\\

\noindent \textbf{Case 1.} $c_{1,t} = 1$. In this case,
$\matrice{u}_t$ is block diagonal and so we get two eigenvectors:
\begin{eqnarray}
\matrice{u}_t \left[\begin{array}{c}
\ve{a}_t\\ \hline
\ve{0}
\end{array}\right] & = & -\frac{1}{c_{0,t}c_{2,t}}\cdot \left[\begin{array}{c|c}
c_{2,t} \cdot \ve{a}_t \ve{a}_t^\top & \matrice{0} \\ \cline{1-2}
\matrice{0} & c_{0,t} \cdot \ve{b}_t\ve{b}_t^\top 
\end{array}\right]\left[\begin{array}{c}
\ve{a}\\ \hline
\ve{0}
\end{array}\right]\nonumber\\
 & = & -\frac{1}{\lambda(\ve{a}^+_t)}\cdot \left[\begin{array}{c}
\ve{a}_t\\ \hline
\ve{0}
\end{array}\right]\:\:,
\end{eqnarray}
with (since $\|\ve{a}^+_t\|_2^2 = \|\ve{a}_t\|_2^2$):
\begin{eqnarray}
\lambda(\ve{a}^+_t) & \defeq & \frac{{\ve{a}^+_t}^\top  \matrice{v}_{t-1}
    \ve{a}^+_t}{\|\ve{a}^+_t\|_2^2}\:\:,
\end{eqnarray}
and
\begin{eqnarray}
\matrice{u}_t \left[\begin{array}{c}
\ve{0}\\ \hline
\ve{b}_t
\end{array}\right] & = & -\frac{1}{\lambda(\ve{b}^+_t)}\cdot \left[\begin{array}{c}
\ve{0}\\ \hline
\ve{b}_t
\end{array}\right] \:\:, \lambda(\ve{b}^+_t) \defeq \frac{{\ve{b}^+_t}^\top  \matrice{v}_{t-1}
    \ve{b}^+_t}{\|\ve{b}^+_t\|_2^2}\:\:.
\end{eqnarray}
We also remark that $\matrice{u}_t$ is negative semi-definite.

\noindent \textbf{Case 2.} $c_{1,t} \neq 1$. In this case, let us assume without loss of generality that for some
$\alpha \in \mathbb{R}_{*}$,
\begin{eqnarray}
\ve{x} & = & \alpha\cdot \ve{a}_t \:\:, \nonumber\\
\ve{y} & = & \ve{b}_t \:\:.\nonumber
\end{eqnarray}
In this case, we obtain 
\begin{eqnarray}
\matrice{u}_t \tilde{\ve{x}} & = & \frac{(1 - c_{1,t}) (\ve{a}_t^\top \ve{x}) + c_{0,t}(\ve{b}_t^\top \ve{y})}{(1-c_{1,t})^2-c_{0,t}c_{2,t}}\cdot \left[\begin{array}{c}
\frac{c_{2,t} (\ve{a}_t^\top \ve{x}) + (1-c_{1,t})(\ve{b}_t^\top
  \ve{y})}{(1 - c_{1,t}) (\ve{a}_t^\top \ve{x}) +
  c_{0,t}(\ve{b}_t^\top \ve{y})} \cdot \ve{a}_t\\ \hline
\ve{b}_t
\end{array}\right] \nonumber\\
 & = & \frac{\alpha (1 - c_{1,t})\|\ve{a}_t\|_2^2 + c_{0,t}\|\ve{b}_t\|_2^2}{(1-c_{1,t})^2-c_{0,t}c_{2,t}}\cdot \left[\begin{array}{c}
\frac{\alpha  c_{2,t} \|\ve{a}_t\|_2^2 + (1-c_{1,t})\|\ve{b}_t\|_2^2}{\alpha  (1-c_{1,t}) \|\ve{a}_t\|_2^2 + c_{0,t}\|\ve{b}_t\|_2^2}\cdot \ve{a}_t\\ \hline
\ve{b}_t
\end{array}\right]  \defeq \lambda \cdot \tilde{\ve{x}}\:\:,
\end{eqnarray}
and so we obtain the eigenvalue 
\begin{eqnarray}
\lambda & = & \frac{\alpha (1 - c_{1,t})\|\ve{a}_t\|_2^2 + c_{0,t}\|\ve{b}_t\|_2^2}{(1-c_{1,t})^2-c_{0,t}c_{2,t}}\:\:,
\end{eqnarray}
and we get from the eigenvector that $\alpha$ satisfies
\begin{eqnarray}
\alpha & = & \frac{\alpha  c_{2,t} \|\ve{a}_t\|_2^2 + (1-c_{1,t})\|\ve{b}_t\|_2^2}{\alpha (1-c_{1,t}) \|\ve{a}_t\|_2^2 + c_{0,t}\|\ve{b}_t\|_2^2}\:\:,
\end{eqnarray}
and so
\begin{eqnarray}
(1-c_{1,t}) \|\ve{a}_t\|_2^2 \alpha^2 +
(c_{0,t}\|\ve{b}_t\|_2^2-c_{2,t}\|\ve{a}_t\|_2^2) \alpha -
(1-c_{1,t})\|\ve{b}_t\|_2^2 & = & 0\:\:.
\end{eqnarray}
We note that the discriminant is
\begin{eqnarray}
\tau & = & 
(c_{0,t}\|\ve{b}_t\|_2^2-c_{2,t}\|\ve{a}_t\|_2^2)^2 + 4 (1-c_{1,t})^2
\|\ve{a}_t\|_2^2\|\ve{b}_t\|_2^2\:\:,
\end{eqnarray}
which is always $>0$.  Therefore we always have
two roots,
\begin{eqnarray}
\alpha_{\pm} & = & \frac{c_{2,t}\|\ve{a}_t\|_2^2 - c_{0,t}\|\ve{b}_t\|_2^2 \pm\sqrt{(c_{0,t}\|\ve{b}_t\|_2^2-c_{2,t}\|\ve{a}_t\|_2^2)^2 + 4 (1-c_{1,t})^2
\|\ve{a}_t\|_2^2\|\ve{b}_t\|_2^2}}{2 (1-c_{1,t}) \|\ve{a}_t\|_2^2 }\:\:.
\end{eqnarray}
yielding two non-zero eigenvalues,
\begin{eqnarray}
\lambda_{\pm}(\matrice{u}_t) & = & \frac{1}{2\rho}\cdot\left(c_{2,t}\|\ve{a}_t\|_2^2 + c_{0,t}\|\ve{b}_t\|_2^2 \pm\sqrt{(c_{0,t}\|\ve{b}_t\|_2^2-c_{2,t}\|\ve{a}_t\|_2^2)^2 + 4 (1-c_{1,t})^2
\|\ve{a}_t\|_2^2\|\ve{b}_t\|_2^2}\right)\:\:.
\end{eqnarray}
Let us analyze the sign of both eigenvalues. For the
numerator of $\lambda_-$ to be negative, we have equivalently after simplification
\begin{eqnarray}
(c_{2,t}\|\ve{a}_t\|_2^2 + c_{0,t}\|\ve{b}_t\|_2^2)^2 & < & (c_{0,t}\|\ve{b}_t\|_2^2-c_{2,t}\|\ve{a}_t\|_2^2)^2 + 4 (1-c_{1,t})^2
\|\ve{a}_t\|_2^2\|\ve{b}_t\|_2^2\:\:,
\end{eqnarray}
which simplifies in $c_{0,t}c_{2,t} < (1-c_{1,t})^2$, \textit{i.e.} $\rho
> 0$. Hence, $\lambda_- < 0$.

Now, for $\lambda_+$, it is easy to check that its sign is that of
$\rho$. When $\rho>0$, we have $\lambda_+ \geq |\lambda_-|$, and
because $a^2 + b^2 \leq (|a|+|b|)^2$, we get
\begin{eqnarray}
\lambda_1^\downarrow(\matrice{u}_t) = \lambda_+ & \leq & \frac{1}{2}\cdot\left(c_{2,t}\|\ve{a}_t\|_2^2 +
  c_{0,t}\|\ve{b}_t\|_2^2 + |c_{0,t}\|\ve{b}_t\|_2^2-c_{2,t}\|\ve{a}_t\|_2^2| + 2 (1-c_{1,t})
\|\ve{a}_t\|_2\|\ve{b}_t\|_2\right)\nonumber\\
 & \leq & c_{2,t}\|\ve{a}_t\|_2^2 +
  c_{0,t}\|\ve{b}_t\|_2^2 + (1-c_{1,t})
\|\ve{a}_t\|_2\|\ve{b}_t\|_2\:\:.\label{defLAMBDA}
\end{eqnarray}
Now, remark that because $\matrice{v}_t$ is positive definite,
\begin{eqnarray}
c_{0,t} - 2 c_{1,t} + c_{2,t} & \defeq & {{\ve{a}^+_t}}^\top  \matrice{v}_{t}
    {\ve{a}^+_t} - 2 {{\ve{a}^+_t}}^\top  \matrice{v}_{t}
    {\ve{b}^+_t} + {{\ve{b}^+_t}}^\top  \matrice{v}_{t}
    {\ve{b}^+_t}\nonumber\\
 & = & (\ve{a}^+_t-\ve{b}^+_t)^\top\matrice{v}_t
 (\ve{a}^+_t-\ve{b}^+_t) \nonumber\\
 & \geq & 0\:\:,
\end{eqnarray}
showing that $c_{1,t}  \leq (c_{0,t}+ c_{2,t})/2$. So we get from
ineq. (\ref{defLAMBDA}),
\begin{eqnarray}
\lambda_1^\downarrow(\matrice{u}_t) & \leq &  \frac{1}{\rho}\cdot\left( c_{2,t}\|\ve{a}_t\|_2^2 +
  c_{0,t}\|\ve{b}_t\|_2^2 + \left(1+\frac{c_{0,t} + c_{2,t}}{2}\right)
\|\ve{a}_t\|_2\|\ve{b}_t\|_2\right)\nonumber\\
 & \leq & \frac{1}{\rho}\cdot\left(1 +
  \frac{3}{2}\cdot (c_{0,t} + c_{2,t}) \right) \cdot \max\{\|\ve{a}_t\|^2_2,
\|\ve{b}_t\|^2_2\}\nonumber\\
 & \leq & \frac{2+ 3 (c_{0,t} + c_{2,t})}{2((1-c_{1,t})^2 - c_{0,t}c_{2,t})} \cdot \max\{\|\ve{a}_t\|^2_2,
\|\ve{b}_t\|^2_2\}\:\:.\label{ineqLAMBDAGEN}
\end{eqnarray}
When $\rho<0$, we remark that $\lambda_+ < \lambda_-$ and so
$\matrice{u}_t$ is negative semi-definite.\\

Whenever $c_{1,t} \neq 1$, it is then easy to check that for any $z
\in \{|\lambda_+|, |\lambda_-|\}$,
ineq. (\ref{ineqLAMBDAGEN}) brings
\begin{eqnarray}
z & \leq &  \frac{2+ 3 (c_{0,t} + c_{2,t})}{2|(1-c_{1,t})^2 - c_{0,t}c_{2,t}|} \cdot \max\{\|\ve{a}_t\|^2_2,
\|\ve{b}_t\|^2_2\}\:\:.\label{ineqLAMBDAGEN2}
\end{eqnarray}
Whenever $c_{1,t} = 1$ (Case 1.), it is also immediate to check that
for any $z
\in \{|-1/\lambda(\ve{a}^+_t)|, |-1/\lambda(\ve{b}^+_t)|\}$,
\begin{eqnarray}
z & \leq & \max\left\{\frac{1}{c_{0,t}}, \frac{1}{c_{2,t}}\right\}\cdot \max\{\|\ve{a}_t\|^2_2,
\|\ve{b}_t\|^2_2\} \nonumber\\
 & < & \left(1+\frac{3}{c_{0,t}}+ \frac{3}{c_{2,t}}\right) \cdot \max\{\|\ve{a}_t\|^2_2,
\|\ve{b}_t\|^2_2\}\nonumber\\
 & & = \frac{2+ 3 (c_{0,t} + c_{2,t})}{2|(1-c_{1,t})^2 - c_{0,t}c_{2,t}|} \cdot \max\{\|\ve{a}_t\|^2_2,
\|\ve{b}_t\|^2_2\}\:\:.
\end{eqnarray}
Once we remark that $c_{1,t} = 1$ implies $\rho < 0$, we obtain the statement of Lemma \ref{lemLAMBDAUT}.
\end{proof}

\begin{lemma}\label{lemV}
If $\PERM_t$ is $(\epsilon, \tau)$-accurate, then the following
holds true:
\begin{eqnarray}
\|\ve{b}^+_t\|_2^2 = \|\ve{b}_t\|_2^2 & \leq & 2\xi \cdot X_*^2\:\:,\label{b11}\\
\|\ve{a}^+_t\|_2^2 = \|\ve{a}_t\|_2^2 & \leq & 2\xi \cdot X_*^2\:\:,\label{b12}
\end{eqnarray}
where $\xi$ is defined in eq. (\ref{defXI}).
\end{lemma}
\begin{proof}
To prove ineq. (\ref{b11}), we make two applications of point 2. in the $(\epsilon, \tau)$-accuracy
assumption with $\F \defeq \shuffle$:
\begin{eqnarray}
\vstretch((\ve{x}_{\ub{t}} -
\ve{x}_{\vb{t}})_{\F},\ve{w}_{\F}) & \leq & \epsilon \cdot \max_{i\in
  \{\ub{t}, \vb{t}\}} \vstretch(\ve{x}_{i},\ve{w})+
\tau \:\:, \nonumber\\
 & & \forall \ve{w}\in \mathbb{R}^d :
\|\ve{w}\|_2 = 1 \:\:.
\end{eqnarray}
Fix $\ve{w} \defeq (1/\|\ve{x}_{\vb{t}}\|_2)\cdot \ve{x}_{\vb{t}}$. We get:
\begin{eqnarray}
|(\ve{x}_{\vb{t}} - \ve{x}_{\ub{t}})_\shuffle^\top \ve{x}_{\vb{t}_\shuffle}|
& \leq & \epsilon \cdot \max\{|\ve{x}_{\ub{t}}^\top \ve{x}_{\vb{t}}|,
\|\ve{x}_{\vb{t}} \|_2^2\} + \tau \cdot \|\ve{x}_{\vb{t}} \|_2
\nonumber\\
 & \leq & \epsilon \cdot X_*^2 + \tau \cdot X_*  = \xi \cdot X_*^2\:\:.\label{feqq1}
\end{eqnarray} 
Fix $\ve{w} \defeq (1/\|\ve{x}_{\ub{t}}\|_2)\cdot \ve{x}_{\ub{t}}$. We get:
\begin{eqnarray}
|(\ve{x}_{\ub{t}} - \ve{x}_{\vb{t}})_\shuffle^\top \ve{x}_{\ub{t}_\shuffle}|
& \leq & \epsilon \cdot \max\{|\ve{x}_{\ub{t}}^\top \ve{x}_{\vb{t}}|,
\|\ve{x}_{\ub{t}} \|_2^2\} + \tau \cdot \|\ve{x}_{\ub{t}} \|_2
\nonumber\\
 & \leq & \epsilon \cdot X_*^2 + \tau \cdot X_* = \xi \cdot X_*^2 \:\:. \label{feqq2}
\end{eqnarray} 
Folding together ineqs. (\ref{feqq1}) and (\ref{feqq2}) yields
\begin{eqnarray}
\lefteqn{\|(\ve{x}_{\vb{t}} - \ve{x}_{\ub{t}})_\shuffle\|_2^2 = (\ve{x}_{\vb{t}} -
\ve{x}_{\ub{t}})^\top_\shuffle (\ve{x}_{\vb{t}} - \ve{x}_{\ub{t}})_\shuffle}
\nonumber\\
 & \leq & |(\ve{x}_{\vb{t}} -
\ve{x}_{\ub{t}})^\top_\shuffle \ve{x}_{\vb{t} _\shuffle}| + |(\ve{x}_{\vb{t}} -
\ve{x}_{\ub{t}})^\top_\shuffle \ve{x}_{\ub{t} _\shuffle}| \nonumber\\
 & \leq & 2  \xi \cdot X_*^2\:\:.
\end{eqnarray}
We get 
\begin{eqnarray}
\|\ve{b}^+_t\|_2^2 = \|\ve{b}_t\|_2^2 = \|(\ve{x}_{\vb{t}} - \ve{x}_{\ub{t}})_\shuffle\|_2^2 & \leq & 2  \xi \cdot X_*^2\:\:,\label{b11F}
\end{eqnarray}
which yields ineq. (\ref{b11}).
To get ineq. (\ref{b12}),
we switch $\F \defeq \shuffle$ by $\F \defeq \anchor$ in our
application of point 2. in the $(\epsilon, \tau)$-accuracy
assumption. 
\end{proof} 
\begin{lemma}\label{lemBOUNDCT}
If $\PERM_t$ is $(\epsilon, \tau)$-accurate and the data-model calibration assumption holds,
\begin{eqnarray}
c_{i,t} & \leq & \frac{1}{12}\:\:, \forall i \in \{0, 1, 2\}\:\:.
\end{eqnarray}
\end{lemma}
\begin{proof}
We remark that
\begin{eqnarray}
c_{0,t} & \defeq & {\ve{a}^+_t}^\top  \matrice{v}_{t}
    {\ve{a}^+_t}\nonumber\\
 & \leq & \lambda_1^\downarrow(\matrice{v}_t) \|\ve{a}^+_t\|_2^2\nonumber\\
 & \leq & 2 \lambda_1^\downarrow(\matrice{v}_t) \xi \cdot X_*^2\:\:,\nonumber
\end{eqnarray}
and for the same reasons, $c_{2,t} \leq 2 \lambda_1^\downarrow(\matrice{v}_t) \xi \cdot X_*^2$. Hence, it comes from the proof of Lemma
 \ref{lemLAMBDAUT} that we also have $c_{2,t} \leq 2 \lambda_1^\downarrow(\matrice{v}_t) \xi
 \cdot X_*^2$. Using ineq. (\ref{eq002F}) in Lemma \ref{lemV2}, we thus obtain for any $i \in \{0,
 1, 2\}$:
\begin{eqnarray}
c_{i,t} & \leq & \frac{1}{n}\cdot \frac{2 \xi \cdot X_*^2}{ (1-\epsilon)^2
  \cdot \inf_{\ve{w}}\sigma^2(\{\vstretch(\ve{x}_i,\ve{w})\}_{i=1}^n)
  + 8\gamma \lambda_1^\uparrow(\Gamma)} \nonumber\\
 & & = \frac{\xi}{n}\cdot \frac{1}{4} \cdot \frac{X_*^2}{ \frac{(1-\epsilon)^2}{8}
  \cdot \inf_{\ve{w}}\sigma^2(\{\vstretch(\ve{x}_i,\ve{w})\}_{i=1}^n)
  + \gamma \lambda_1^\uparrow(\Gamma)}\nonumber\\
 & \leq & \frac{1}{4} \cdot \frac{1}{4} \cdot 1 < \frac{1}{12}\:\:,
\end{eqnarray}
as claimed. The last inequality uses the data-model calibration assumption.
\end{proof}

\begin{corollary}\label{corBOUNDCT}
Suppose $\PERM_t$ is $(\epsilon, \tau)$-accurate for any $t\geq 1$ and the data-model
calibration assumption holds. Then the invertibility assumption holds.
\end{corollary}
\begin{proof}
From Lemma \ref{lemBOUNDCT}, we conclude that $(1-c_{1,t})^2 > 121/144
> 1/144 > c_{0,t}c_{2,t} > 0$, hence the invertibility assumption holds.
\end{proof}

\begin{lemma}\label{boundLAMBDAT}
If $\PERM_t$ is $(\epsilon, \tau)$-accurate and the data-model
calibration assumption holds, the following
holds true: $\matrice{i}_d + \Lambda_t\succ 0$ and 
\begin{eqnarray}
\lambda_1^\downarrow\left(\Lambda_t\right) & \leq & \frac{\xi}{n} \:\:.\label{condLAMBDAT}
\end{eqnarray}
\end{lemma}
\begin{proof}
First note that $\lambda_1^\uparrow(\matrice{v}_{t}) \geq 1/(\gamma
\lambda_1^\downarrow(\Gamma)) > 0$ and so $\matrice{v}_{t} \succ 0$,
which implies that $\Lambda_t \defeq 2\matrice{v}_{t}\matrice{u}_t =
2\matrice{v}^{1/2}_{t}(\matrice{v}^{1/2}_{t} \matrice{u}_t
\matrice{v}^{1/2}_{t}) \matrice{v}^{-1/2}_{t}$, \textit{i.e.}
$\Lambda_t$ is similar to a symmetric matrix ($\matrice{v}^{1/2}_{t} \matrice{u}_t
\matrice{v}^{1/2}_{t} $) and therefore has only
real eigenvalues. 
We get
\begin{eqnarray}
\lambda_1^\downarrow\left(\Lambda_t\right) & = & \lambda_1^\downarrow\left(2\matrice{v}_{t}\matrice{u}_t\right) \nonumber\\
 & \leq & 2 \cdot \lambda_1^\downarrow(\matrice{v}_t)  \cdot \left(1 +
  \frac{3}{2}\cdot (c_{0,t} + c_{2,t}) \right) \cdot \max\{\|\ve{a}_t\|^2_2,
\|\ve{b}_t\|^2_2\} \label{eq111}\\
 & \leq & \frac{2 +
  3(c_{0,t} + c_{2,t})}{|(1-c_{1,t})^2 - c_{0,t}c_{2,t}|} \cdot  2 \lambda_1^\downarrow(\matrice{v}_t)  \xi \cdot X_*^2 \label{eq112}\:\:.
\end{eqnarray}
Ineq. (\ref{eq111}) is due to Lemma \ref{lemLAMBDAUT} and
ineq. (\ref{eq112}) is due to Lemma \ref{lemV}. We now use Lemma
\ref{lemBOUNDCT} and its proof, which shows that 
\begin{eqnarray}
(1-c_{1,t})^2 -
c_{0,t}c_{2,t} & \geq & \left(1-\frac{1}{12}\right)^2 - \frac{1}{144}
\nonumber\\
 & & = \frac{5}{6}\:\:. 
\end{eqnarray}
Letting $U\defeq 2 \lambda_1^\downarrow(\matrice{v}_t)\xi \cdot X_*^2$ for short, we thus get from the proof of Lemma
\ref{lemBOUNDCT}:
\begin{eqnarray}
\lambda_1^\downarrow\left(\Lambda_t\right) & \leq & \frac{6}{5} \cdot (2+3(U +
U))U\nonumber\\
 & & = \frac{6}{5} \cdot (2U+6U^2)\:\:.
\end{eqnarray}
Now we want $\lambda_1^\downarrow\left(\Lambda_t\right) \leq \xi /n$, which translates into a second-order inequality for $U$,
whose solution imposes the following upperbound on $U$:
\begin{eqnarray}
6 U & \leq & -1 + \sqrt{1+\frac{5\xi}{n}}\:\:.\label{condUU}
\end{eqnarray}
We can indeed forget the lowerbound for $U$, whose sign is negative while
$U\geq 0$. 

Since $\sqrt{1+x}\geq 1 + (x/2) - (x^2/8)$ for $x\geq 0$
(and $\xi/n \geq 0$), we get the
sufficient condition for ineq. (\ref{condUU}) to be satisfied:
\begin{eqnarray}
12 \lambda_1^\downarrow(\matrice{v}_t) \xi
 \cdot X_*^2 & \leq & \frac{5 \xi}{2n} -\frac{25}{8}\cdot \left(\frac{\xi}{n}\right)^2\:\:.\label{constLAMBDA2}
\end{eqnarray}
Now, it comes from Lemma \ref{lemV2} that a sufficient condition for
ineq. (\ref{constLAMBDA2}) is that
\begin{eqnarray}
\frac{\xi}{n}\cdot \frac{12 X_*^2}{ (1-\epsilon)^2
  \cdot \inf_{\ve{w}}\sigma^2(\X, \ve{w}) + 8\gamma \lambda_1^\uparrow(\Gamma)}  & \leq & \frac{5\xi}{2n} -\frac{25}{8}\cdot \left(\frac{\xi}{n}\right)^2\:\:,\label{constLAMBDA3}
\end{eqnarray}
which, after simplification, is equivalent to
\begin{eqnarray}
\frac{3}{5}\cdot \frac{X_*^2}{ \frac{(1-\epsilon)^2}{8}
  \cdot \inf_{\ve{w}}\sigma^2(\X, \ve{w}) + \gamma \lambda_1^\uparrow(\Gamma)} +
\frac{5\xi}{4n} & \leq & 1\:\:.\label{eqTWOCONT}
\end{eqnarray}
But, the data-model calibration assumption implies that the left-hand
side is no more than $(3/5) + (5/16) = 73/80 < 1$, and 
ineq. (\ref{condLAMBDAT}) follows.\\

It also trivially follows that $\matrice{i}_d + \Lambda_t$ has only
real eigenvalues. To prove that they are all strictly positive, we
know that the only potentially negative eigenvalue of $\matrice{u}_t$,
$\lambda_-$ (Lemma \ref{lemLAMBDAUT}) is smaller in absolute value to
$\lambda_1^\downarrow(\matrice{u}_t)$. $\matrice{v}_t$ being positive
definite, we thus have under the $(\epsilon, \tau)$-accuracy
assumption and data-model calibration:
\begin{eqnarray}
\lambda_1^\uparrow(\matrice{i}_d + \Lambda_t) & \geq & 1 -
\frac{\xi}{n} \nonumber\\
 & \geq & 1 - \frac{1}{4}  = \frac{3}{4} > 0\:\:,
\end{eqnarray}
showing $\matrice{i}_d + \Lambda_t$ is positive definite.
This ends the
proof of Lemma \ref{boundLAMBDAT}.
\end{proof}
Let $0\leq T_+\leq T$ denote the number of elementary permutations
that act between classes, and $\rho \defeq T_+ / T$ denote the
proportion of such elementary permutations among all.
\begin{theorem}\label{thAPPROX1SMB}
Suppose $\PERM_*$ is $(\epsilon, \tau)$-accurate and
$\alpha$-bounded, and the data-model calibration assumption holds. Then the following holds for all $T\geq 1$:
\begin{eqnarray}
\|\ve{\theta}^*_{T} - \ve{\theta}^*_0 \|_2 & \leq & \frac{\xi}{n} \cdot T^2 \cdot \left( \|\ve{\theta}^*_0 \|_2 +
\frac{\sqrt{\xi}}{4 X_*} \cdot \rho\right)\nonumber\\
 & \leq & \left(\frac
  {\xi}{n}\right)^{\alpha} \cdot \left(
  \|\ve{\theta}^*_0 \|_2 + \frac{\sqrt{\xi}}{4 X_*}
  \cdot \rho
\right)\:\:.\label{eqthAPPROX0}
\end{eqnarray}
\end{theorem}
\begin{proof}
We use Theorem \ref{thEXACT}, which yields from the triangle inequality:
\begin{eqnarray}
\|\ve{\theta}^*_{T} - \ve{\theta}^*_{0}\|_2 & = & \|(\matrice{h}_{T,0} - \matrice{i}_d)
\ve{\theta}^*_0\|_2 + \left\|\sum_{t=0}^{T-1} \matrice{h}_{T,t+1} \ve{\lambda}_{t}\right\|_2\:\:.\label{condAT2}
\end{eqnarray}
Denote for short $q\defeq \xi/n$. It comes from the definition of $\matrice{h}_{i,j}$ and
Lemma \ref{boundLAMBDAT} the first inequality of:
\begin{eqnarray}
\lambda_1^\downarrow\left(\matrice{h}_{T,0} - \matrice{i}_d\right) &
\leq & (1+q)^T-1 \nonumber\\
 & \leq & T^2 q\:\:,\label{condLAMBDAH}
\end{eqnarray}
where the second inequality holds because  ${T \choose k} q^k\leq
(Tq)^k\leq Tq$ for $k\geq 1$ whenever $Tq \leq 1$, which is equivalent
to
\begin{eqnarray}
T & \leq & \frac{n}{\xi}\:\:,\label{condT}
\end{eqnarray}
which is implied by the condition of $\alpha$-bounded permutation
size ($n/\xi \geq 4 \geq 1$ from the data-model
calibration assumption). We thus get 
\begin{eqnarray}
\|\left(\matrice{h}_{T,0} - \matrice{i}_d\right) \ve{\theta}^*_{0}\|_2 & \leq & T^2 q \cdot \|\ve{\theta}^*_{0}\|_2\:\:.\label{eq001}
\end{eqnarray}
Using ineq. (\ref{condAT2}), this shows the statement of the Theorem
with (\ref{condAT}). The upperbound comes from the fact that the factor in the right hand side is no more than
$(\xi/n)^\alpha$ for some $0\leq \alpha \leq 1$
provided this time the stronger constraint holds:
\begin{eqnarray}
T & \leq & \left(\frac{n}{\xi}\right)^{\frac{1-\alpha}{2}}\:\:,\label{condT2}
\end{eqnarray}
which is the condition of $\alpha$-boundedness.

Let us now have a look at the shift term in eq. (\ref{condAT2}), which depends only on the
mistakes between classes done during the permutation (which changes
the mean operator between permutations),
\begin{eqnarray}
\matrice{r} & \defeq & \sum_{t=0}^{T-1} \matrice{h}_{T,t+1} \ve{\lambda}_{t}\:\:.
\end{eqnarray}
Using eq. (\ref{defL2}), we can simplify $\matrice{r}$ since $\ve{\lambda}_{t} = 2 \matrice{v}_{t+1} \ve{\epsilon}_{t}$, so if we define $\matrice{g}_{.,.}$ from
$\matrice{h}_{.,.}$ as follows, for $0\leq j \leq i$:
\begin{eqnarray}
\matrice{g}_{i,j} & \defeq & 2 \matrice{h}_{i, j} \matrice{v}_{j} \:\:,\label{defGIJTILDE}
\end{eqnarray}
then we get
\begin{eqnarray}
\matrice{r} & \defeq & \sum_{t=0}^{T-1} \matrice{g}_{T,t+1} \ve{\epsilon}_{t}\:\:,
\end{eqnarray}
where we recall that $\ve{\epsilon}_{t} \defeq \ve{\mu}_{t+1} -
  \ve{\mu}_{t}$ is the shift in the mean operator, \textit{which is
    the null vector whenever $\PERM_t$ acts in a specific class} ($y_{\ua{t}}
  = y_{\va{t}}$). To see this, 
we remark
\begin{eqnarray}
\ve{\epsilon}_{t} & \defeq & \ve{\mu}_{t+1} -
  \ve{\mu}_{t}\nonumber\\
 & = & \sum_i y_i \cdot \left[
\begin{array}{c}
\ve{x}_{i_\anchor}\\\cline{1-1}
\ve{x}_{{(t+1)i}_\shuffle} 
\end{array}
\right] - \sum_i y_i \cdot \left[
\begin{array}{c}
\ve{x}_{i_\anchor}\\\cline{1-1}
\ve{x}_{{ti}_\shuffle} 
\end{array}
\right]\nonumber\\
 & = & \sum_i y_i \cdot \left[
\begin{array}{c}
0\\\cline{1-1}
\ve{x}_{{(t+1)i}_\shuffle} 
\end{array}
\right] - \sum_i y_i \cdot \left[
\begin{array}{c}
0\\\cline{1-1}
\ve{x}_{{ti}_\shuffle} 
\end{array}
\right]\nonumber\\
 & = & \left[
\begin{array}{c}
0\\\cline{1-1}
\sum_i y_i \cdot (\ve{x}_{{(t+1)i}_\shuffle} - \ve{x}_{{ti}_\shuffle}) 
\end{array}
\right] \defeq \left[
\begin{array}{c}
0\\\cline{1-1}
{\ve{\epsilon}_{t}}_\shuffle
\end{array}
\right]\:\:,
\end{eqnarray}
which can be simplified further since we work with the elementary
permutation $\PERM_{t}$,
\begin{eqnarray}
{\ve{\epsilon}_{t}}_\shuffle & = & y_{\ua{t}} \cdot
(\ve{x}_{{\vb{t}}}-\ve{x}_{{\ub{t}}})_\shuffle +  y_{\va{t}} \cdot
(\ve{x}_{{\ub{t}}}-\ve{x}_{{\vb{t}}})_\shuffle\nonumber\\
 & = & (y_{\ua{t}} - y_{\va{t}}) \cdot (\ve{x}_{{\vb{t}}}-\ve{x}_{{\ub{t}}})_\shuffle\:\:.
\end{eqnarray}
Hence, 
\begin{eqnarray}
\|\ve{\epsilon}_{t} \|_2 =\|{\ve{\epsilon}_{t}}_\shuffle \|_2 & = &
1_{y_{\ua{t}} \neq y_{\va{t}}} \cdot
\|(\ve{x}_{{\vb{t}}}-\ve{x}_{{\ub{t}}})_\shuffle\|_2\nonumber\\
 & \leq & 1_{y_{\ua{t}} \neq y_{\va{t}}} \cdot \sqrt{2\xi)} X_*\:\:,\label{beps}
\end{eqnarray}
from Lemma \ref{lemV},  and we see that indeed $\|\ve{\epsilon}_{t} \|_2 = 0$ when the
elementary permutation occurs within observations of the same class.\\

It follows from the data-model calibration assumption and Lemma \ref{boundLAMBDAT} that
\begin{eqnarray}
\lambda_1^\downarrow(\matrice{v}_t) & \leq & \frac{1}{n}\cdot \frac{1}{ (1-\epsilon)^2
  \cdot \inf_{\ve{w}}\sigma^2(\{\vstretch(\ve{x}_i,\ve{w})\}_{i=1}^n)
  + 8 \gamma \lambda_1^\uparrow(\Gamma)} \nonumber\\
& & = \frac{1}{8nX_*^2}\cdot \frac{X_*^2}{ \frac{(1-\epsilon)^2}{8}
  \cdot \inf_{\ve{w}}\sigma^2(\{\vstretch(\ve{x}_i,\ve{w})\}_{i=1}^n)
  + \gamma \lambda_1^\uparrow(\Gamma)} \nonumber\\
 & \leq & \frac{1}{8nX_*^2}\:\:.\label{boundLAMBDAVT2}
\end{eqnarray}
Using \citep[Problem III.6.14]{bMA}, Lemma \ref{boundLAMBDAT} and
ineq. (\ref{boundLAMBDAVT2}), we also obtain
\begin{eqnarray}
\lambda^\downarrow_1\left(\matrice{g}_{T,t+1}\right) & \leq & 2 \cdot \left(1+
  \frac{\xi}{n}
\right)^{T-t-1} \cdot \frac{1}{8nX_*^2}\:\:.
\end{eqnarray}
So, 
\begin{eqnarray}
\|\matrice{r}\|_2 & \leq & \sum_{t=0}^{T-1}
\lambda_{\mathrm{max}}\left(\matrice{g}_{T,t+1}\right)
\|\ve{\epsilon}_{t}\|_2\nonumber\\
 & \leq & \frac{1}{2\sqrt{2}} \cdot \sum_{t=0}^{T-1}
1_{y_{\ua{t}} \neq y_{\va{t}}} \cdot \left(1+
  \frac{\xi}{n}
\right)^{T-t-1} \cdot \frac{\sqrt{\xi}}{nX_*}\nonumber\\
 & & = \frac{1}{2X_*} \cdot \sqrt{\frac{\xi}{2}}\cdot \sum_{t=0}^{T-1}
1_{y_{\ua{t}} \neq y_{\va{t}}} \cdot \left(1+
  \frac{\xi}{n}
\right)^{T-t-1} \cdot \frac{\xi}{n}\:\:,\label{upperRT}
\end{eqnarray}
from ineq. (\ref{beps}). Assuming $T_+ \leq T$ errors are made by permutations
between classes and recalling $q \defeq \xi/n$, we see that the largest upperbound for $\|\matrice{r}\|_2 $ in
ineq. (\ref{upperRT}) is obtained when all $T_+$ errors happen at the
last elementary permutations in the sequence in $\PERM_*$, so we get that
\begin{eqnarray}
\|\matrice{r}\|_2 & \leq & \frac{1}{2X_*} \cdot \sqrt{\frac{\xi}{2}}\cdot
\sum_{t=0}^{T_+-1} q(1+q)^{T-t-1}\nonumber\\
 & & =  \frac{1}{2X_*} \cdot \sqrt{\frac{\xi}{2}} \cdot q(1+q)^{T-T_+}
\sum_{t=0}^{T_+-1} (1+q)^{T_+-t-1}\nonumber\\
 & =  & \frac{1}{2X_*} \cdot \sqrt{\frac{\xi}{2}} \cdot (1+q)^{T-T_+}((1+q)^{T_+}-1) \:\:.\label{ineq14}
\end{eqnarray}
It comes from ineq. (\ref{condLAMBDAH}) $(1+q)^{T_+}-1 \leq T_+^2 q$
and 
\begin{eqnarray}
 (1+q)^{T-T_+} & \leq & (T-T_+)^2q + 1\nonumber\\
 & \leq & \left(\frac{n}{\xi}\right)^{1-\alpha} \cdot \frac{\xi}{n} + 1
 \nonumber\\
 & & = \left(\frac{\xi}{n}\right)^\alpha + 1\nonumber\\
 & \leq & \frac{1}{4} + 1 < \sqrt{2}\:\:.
\end{eqnarray}
The last line is due to the data-model calibration assumption. We
finally get from ineq. (\ref{ineq14})
\begin{eqnarray}
\|\matrice{r}\|_2 & \leq & \frac{\sqrt{\xi}}{2X_*} \cdot
\frac{\xi}{n} \cdot T_+^2\nonumber\\
 &  & = \frac{\xi^{\frac{3}{2}}}{4 X_* n} \cdot T_+^2 \:\:.
\end{eqnarray}
We also remark that if $\PERM_*$ is $\alpha$-bounded, since $T_+ \leq
T$, we also have:
\begin{eqnarray}
\frac{\xi^{\frac{3}{2}}}{4 X_* n} \cdot T_+^2 & \leq &
\frac{\xi^{\frac{3}{2}}}{4 X_* n} \cdot
\left(\frac{n}{\xi}\right)^{1-\alpha}\nonumber\\
 & & = \frac{\sqrt{\xi}}{4 X_*} \cdot
\left(\frac {\xi}{n}\right)^{\alpha}\:\:.
\end{eqnarray}
Summarizing, we get
\begin{eqnarray}
\|\ve{\theta}^*_{T} - \ve{\theta}^*_0 \|_2 & \leq & a(T) \cdot
\|\ve{\theta}^*_0 \|_2 + b(T_+)\:\:,\label{eqthAPPROX1}
\end{eqnarray}
where 
\begin{eqnarray}
a(T) & \defeq &  \frac{\xi}{n} \cdot T^2 \leq \left(\frac{\xi}{n}\right)^\alpha \:\:,\label{condAT}\\
b(T_+) & \defeq & \frac{\xi^{\frac{3}{2}}}{4 X_* n}
\cdot T_+^2 \leq \frac{\sqrt{\xi}}{4 X_*} \cdot
\left(\frac{\xi}{n}\right)^{\alpha}\:\:,\label{condBT}
\end{eqnarray}
which yields the proof
of Theorem \ref{thAPPROX1SMB}.
\end{proof}
Theorem \ref{thAPPROX1SMB}  easily yields the proof of Theorem
\ref{thAPPROX1}.

\subsection{Proof of Theorem \ref{thIMMUNE}}\label{app:proof-thIMMUNE}

Remark that for any example $(\ve{x}, y)$, we have from
Cauchy-Schwartz inequality:
\begin{eqnarray}
|y (\ve{\theta}^*_{T} - \ve{\theta}^*_0)^\top \ve{x}| = |(\ve{\theta}^*_{T} - \ve{\theta}^*_0)^\top \ve{x}| & \leq &
\|\ve{\theta}^*_{T} - \ve{\theta}^*_0\|_2 \|\ve{x}\|_2\nonumber\\
 & \leq &  \left(\frac
  {\xi}{n}\right)^{\alpha} \cdot \left(
  \|\ve{\theta}^*_0 \|_2 + \frac{\sqrt{\xi}}{4 X_*}
  \cdot \rho
\right) \cdot X_*\nonumber\\
 & & = \left(\frac
  {\xi}{n}\right)^{\alpha} \cdot \left(
  \|\ve{\theta}^*_0 \|_2 X_* + \frac{\sqrt{\xi}}{4}
  \cdot \rho
\right) \:\:.
\end{eqnarray}
So, to have $|y (\ve{\theta}^*_{T} - \ve{\theta}^*_0)^\top \ve{x}|
< \kappa$ for some $\kappa > 0$, it is sufficient that
\begin{eqnarray}
 n & > &  \xi \cdot \left(
  \frac{\|\ve{\theta}^*_0 \|_2 X_*}{\kappa} + \frac{\sqrt{\xi}}{4\kappa}
  \cdot \rho
\right)^{\frac{1}{\alpha}}\:\:.
\end{eqnarray}
In this case, for any example $(\ve{x},
y)$ such that $y (\ve{\theta}^*_0)^\top \ve{x} > \kappa$, then 
\begin{eqnarray}
y (\ve{\theta}^*_T)^\top \ve{x} & = & y (\ve{\theta}^*_0)^\top \ve{x}
+ y (\ve{\theta}^*_T - \ve{\theta}^*_0)^\top \ve{x}\nonumber\\
 & \geq & y (\ve{\theta}^*_0)^\top \ve{x} - |y (\ve{\theta}^*_T -
 \ve{\theta}^*_0)^\top \ve{x}|\nonumber\\
 & > & \kappa - \kappa = 0\:\:,
\end{eqnarray}
and we get the statement of the Theorem.

\subsection{Proof of Theorem \ref{thDIFFLOSS}}\label{app:proof-thDIFFLOSS}

We want to bound the difference between the loss \textit{over the true data}
for the optimal (unknown) classifier $\ve{\theta}^*_0$ and the
classifier we learn from entity resolved data, $\ve{\theta}^*_T$,
\begin{eqnarray}
\Delta_S(\ve{\theta}^*_0, \ve{\theta}^*_T) & \defeq & \ell_{S,
    \gamma}(\ve{\theta}^*_T) - \ell_{S,
    \gamma}(\ve{\theta}^*_0)\:\:.\label{defDelta}
\end{eqnarray}
Simple arithmetics and Cauchy-Schwartz inequality allow to derive:
\begin{eqnarray}
\Delta_S(\ve{\theta}^*_0, \ve{\theta}^*_T) & = & \frac{1}{2n} \cdot
\left(
  (\ve{\theta}^*_T-\ve{\theta}^*_0)^\top\left(\sum_i  y_i
    \ve{x}_i\right) + \frac{1}{4}\cdot \sum_i \left(((\ve{\theta}^*_0)^\top
\ve{x}_i)^2 - ((\ve{\theta}^*_T)^\top
\ve{x}_i)^2\right)\right)\nonumber\\
& = & \frac{1}{2n} \cdot
\left(
  (\ve{\theta}^*_T-\ve{\theta}^*_0)^\top\ve{\mu}_0 + \frac{1}{4}\cdot \sum_i \left( (\ve{\theta}^*_0-\ve{\theta}^*_T)^\top \ve{x}_i
  \right)\left( (\ve{\theta}^*_0+\ve{\theta}^*_T)^\top \ve{x}_i \right)\right)\nonumber\\
 & \leq & \frac{1}{2} \cdot\left( \frac{1}{n}\cdot \|\ve{\theta}^*_T-\ve{\theta}^*_0\|_2
   \|\ve{\mu}_0\|_2 + \|\ve{\theta}^*_T-\ve{\theta}^*_0\|_2
   \|\ve{\theta}^*_T+\ve{\theta}^*_0\|_2 X_*^2\right)\nonumber\\
 & & = \|\ve{\theta}^*_T-\ve{\theta}^*_0\|_2\cdot\left( \frac{\|\ve{\mu}_0\|_2}{2n}
   +
   \frac{1}{2}\cdot \|\ve{\theta}^*_T+\ve{\theta}^*_0\|_2 X_*^2\right)\:\:.\label{eqB22}
\end{eqnarray}
We now need to bound $\|\ve{\theta}^*_T+\ve{\theta}^*_0\|_2$, which is easy
since the triangle inequality yields
\begin{eqnarray}
\|\ve{\theta}^*_T+\ve{\theta}^*_0\|_2 & =& \|\ve{\theta}^*_T -
\ve{\theta}^*_0 + 2 \ve{\theta}^*_0\|_2\nonumber\\
 & \leq & \|\ve{\theta}^*_T -
\ve{\theta}^*_0\|_2 + 2 \|\ve{\theta}^*_0\|_2\:\:,
\end{eqnarray}
and so Theorem \ref{thAPPROX1SMB} yields:
\begin{eqnarray}
\|\ve{\theta}^*_T+\ve{\theta}^*_0\|_2 & \leq & 2 \|\ve{\theta}^*_0\|_2  + \frac{\xi}{n} \cdot T^2
\cdot \left( \|\ve{\theta}^*_0\|_2 +
\frac{\sqrt{\xi}}{4 X_*} \cdot \rho\right)\nonumber\\
 & & = 2 \|\ve{\theta}^*_0\|_2  + \frac{\xi (\deltamargin + \delta_\rho)}{nX_*} \cdot T^2
\:\:,
\end{eqnarray}
where $\deltamargin \defeq \|\ve{\theta}^*_0\|_2 X_*$ is the maximum
margin for the optimal (unknown) classifier and $\delta_\rho \defeq
\sqrt{\xi}\rho / 4$ is a penalty due to class-mismatch
permutations. Denote for short
\begin{eqnarray}
\eta & \defeq & \|\ve{\theta}^*_0\|_2 +
\frac{\sqrt{\xi}}{4 X_*} \cdot \rho = \frac{\deltamargin + \delta_\rho}{X_*}\:\:.
\end{eqnarray}
We finally obtain, letting $\delta_{\mu_0} \defeq \|\ve{\mu}_0\|_2/(nX_*)$
($\in [0,1]$) denote the normalized mean-operator for the true dataset,
\begin{eqnarray}
\Delta_S(\ve{\theta}^*_0, \ve{\theta}^*_T) & \leq & \frac{\xi \eta}{n} \cdot T^2
\cdot \left(
\frac{\|\ve{\mu}_0\|_2}{2n} + \frac{\xi \eta X_*^2}{2n} \cdot T^2 +
\|\ve{\theta}^*_0\|_2X_*^2 \right)\nonumber\\
 & & = \frac{\xi \eta}{n} \cdot T^2
\cdot \left(
\frac{\|\ve{\mu}_0\|_2}{2n} + 
\|\ve{\theta}^*_0\|_2X_*^2 \cdot\left(1 + \frac{\xi}{2n} \cdot
  T^2\right) + \frac{\xi^{\frac{3}{2}} \rho X_*}{8n} \cdot
T^2\right)\nonumber\\ 
& = & \frac{\xi \eta X_*}{n} \cdot T^2
\cdot \left(
\frac{\|\ve{\mu}_0\|_2}{2nX_*} + 
\|\ve{\theta}^*_0\|_2X_* \cdot\left(1 + \frac{\xi}{2n} \cdot
  T^2\right) + \frac{\xi^{\frac{3}{2}} \rho}{8n} \cdot
T^2\right)\nonumber\\
 & = & \frac{\xi (\deltamargin + \delta_\rho)}{n} \cdot T^2
\cdot \left(
\frac{\delta_{\mu_0}}{2} + 
\deltamargin \cdot\left(1 + \frac{\xi}{2n} \cdot
  T^2\right) + \delta_\rho \cdot \frac{\xi}{2n} \cdot
T^2\right)\:\:.\label{eqDELTA}
\end{eqnarray}
Let us denote for short
\begin{eqnarray}
C(n) & \defeq & \frac{\xi}{n} \cdot T^2\:\:.
\end{eqnarray}
We know that under the $\alpha$-boundedness condition,
\begin{eqnarray}
C(n) & \leq & \left( \frac{\xi}{n} \right)^\alpha\nonumber\\
 & & = o(1)\:\:,
\end{eqnarray}
and we finally get from eq. (\ref{eqDELTA}),
\begin{eqnarray}
\Delta_S(\ve{\theta}^*_0, \ve{\theta}^*_T) & \leq & \frac{\deltamargin + \delta_\rho}{2} \cdot C(n)
\cdot\left( \delta_{\mu_0} + 
\deltamargin \cdot\left(2 + C(n)\right) + \delta_\rho \cdot
C(n)\right)\nonumber\\
 & \leq & \deltamarginrho \left( \delta_{\mu_0} + 
6 \deltamarginrho\right)\cdot C(n)\:\:,
\end{eqnarray}
with $\deltamarginrho \defeq (\delta_{\mbox{\tiny m}} +
\delta_\rho)/2$ the average of the margin and class-mismatch
penalties. We have used in the last inequality the fact that under the
data-model calibration assumption, $C(n) \leq (1/4)^\alpha \leq
1$. This ends the proof of Theorem \ref{thDIFFLOSS}.

\subsection{Proof of Theorem \ref{thGENTHETA}}\label{app:proof-thGENTHETA}

The Rademacher complexity is a fundamental notion in learning
\citep{bmRA}. Letting $\Sigma_n \defeq \{-1,1\}^n$, the empirical
Rademacher complexity of hypothesis class ${\mathcal{H}}$ is:
\begin{eqnarray}
R_n & \defeq & \expect_{\sigma \sim \Sigma_n}
 \sup_{h \in {\mathcal{H}}}
 \left\{
   \expect_{S}[\sigma(\ve{x}) h(\ve{x})]
   \right\} \:\:.\label{defRC1}
\end{eqnarray}
When it comes to linear classifiers, two parameters are
fundamental to bound the empirical Rademacher complexity: the maximum
norm of the classifier and the maximum norm of an observation
\citep{kstOT,pnncLF}. There are therefore two Rademacher complexities that are relevant to our
setting: 
\begin{itemize}
\item the one related to the unknown optimal classifier
  $\ve{\theta}_0$, $R_n(\ve{\theta}_0)$, to which we attach maximum
  classifier norm $\theta_*$ and example norm $X_*$. 
  It is well-known that we have \citep[Theorem 3]{kstOT}, \citep[Lemma
  1]{pnncLF}:
\begin{eqnarray}
R_n(\ve{\theta}_0) & \leq & \frac{X_*\theta_*}{\sqrt{n}}\:\:.\label{bsupRC1}
\end{eqnarray}
\item the one related to the optimal classifier that we build on our
  observed data, $R_n(\ve{\theta}_T)$. 
\end{itemize}
Notice that this latter one should \textit{also}
  be computed over the optimal dataset $S$. One may wonder how this
  would degrade any empirically computable bound, that would depend on
  $\hat{S}_T$ instead of $S$ in the supremum in eq. (\ref{defRC1}). It
  can be shown that entity resolution has this very desirable property
  in the vertical partition setting that we can in fact use any empirical
  upperbound on $X_*$, the way the bound in eq. (\ref{defRC1}) is
  computed is not going to be affected. This is what we now show.
\begin{lemma}\label{lerc1}
Let $\hat{X}_{T*} \defeq \max_i \|\hat{\ve{x}}_{Ti}\|_2$ and
$\hat{\theta}_*$ be any upperbound for $\|\ve{\theta}_T\|_2$. We have
\begin{eqnarray}
R_n(\ve{\theta}_T)  & \leq & \frac{\hat{X}_{T*} \hat{\theta}_*}{\sqrt{n}}
\end{eqnarray}
\end{lemma}
\begin{proof}
The proof follows the basic steps of \citep[Lemma 1]{pnncLF}, with
a twist to handle the replacement of $X_*$ by $\hat{X}_{T*}$. We give
the proof for completeness. We have the key
observation that $\forall \ve{\sigma} \in
\Sigma_n$,
\begin{eqnarray}
  \arg \sup_{\ve{\theta} \in {\mathcal{H}}_*}
 \left\{
   \expect_{{S}}[\sigma({\ve{x}}) \ve{\theta}^\top {\ve{x}}]
   \right\} & = & \frac{1}{n}\arg \sup_{\ve{\theta} \in {\mathcal{H}}_*}
 \left\{
   \sum_i {\sigma_i \ve{\theta}^\top {\ve{x}}_{Ti}}
   \right\}\nonumber \\
 & = & \frac{\sup_{{\mathcal{H}}_{*}}\|\theta\|_2}{\|\sum_i {\sigma_i {\ve{x}}_{Ti}}\|_2} \sum_i {\sigma_i {\ve{x}}_{Ti}}\:\:.
\end{eqnarray}
So,
\begin{eqnarray}
  {R}_n & = & \expect_{\Sigma_n}
 \sup_{h \in {\mathcal{H}}}
 \left\{
   \expect_{{\mathcal{S}}}[\sigma({\ve{x}}) h({\ve{x}})]
   \right\} \nonumber\\
 & = &\frac{\sup_{{\mathcal{H}}_{*}}\|\theta\|_2}{n}\expect_{\Sigma_n} \left[\frac{\left(\sum_{{\ve{x}}} {\sigma({\ve{x}})
     {\ve{x}}}\right)^\top \left(\sum_{{\ve{x}}} {\sigma({\ve{x}})
     {\ve{x}}}\right)}{\|\sum_{{\ve{x}}} {\sigma({\ve{x}})
     {\ve{x}}}\|_2}\right]\nonumber\\
 & = & \frac{\sup_{{\mathcal{H}}_{*}}\|\theta\|_2}{n}\expect_{\Sigma_n} \left[\|\sum_{{\ve{x}}} {\sigma({\ve{x}})
     {\ve{x}}}\|_2\right]\nonumber\\
 & = & \frac{\sup_{{\mathcal{H}}_{*}}\|\theta\|_2}{n}\times \frac{1}{|\Sigma_n|} \sum_{\Sigma_n} {\sqrt{\sum_i \|{\ve{x}}_{i}\|_2^2 + \sum_{i_1\neq i_2}
   {\sigma_{i_1} \sigma_{i_2} {\ve{x}}_{i_1}^\top {\ve{x}}_{i_2}}}}\nonumber\\
 & = & \frac{\sup_{{\mathcal{H}}_{*}}\|\theta\|_2\sqrt{\sum_i \|{\ve{x}}_{i}\|_2^2}}{n}\times \frac{1}{|\Sigma_n|} \sum_{\Sigma_n} {\sqrt{1 + \frac{ \sum_{i_1\neq i_2}
   {\sigma_{i_1} \sigma_{i_2} {\ve{x}}_{i_1}^\top {\ve{x}}_{i_2}}}{\sum_i
   \|{\ve{x}}_{i}\|_2^2}}}\nonumber\:\:.
\end{eqnarray}
Using the fact that $\sqrt{1+x} \leq 1 +
x/2$, we have
\begin{eqnarray}
R_n & \leq & \frac{\sup_{{\mathcal{H}}_{*}}\|\theta\|_2 \sqrt{\sum_i \|{\ve{x}}_{i}\|_2^2}}{n}\times \frac{1}{|\Sigma_n|}
\sum_{\Sigma_n} {\left(1 + \frac{ \sum_{i_1\neq i_2}
   {\sigma_{i_1} \sigma_{i_2} {\ve{x}}_{i_1}^\top {\ve{x}}_{i_2}}}{2\sum_i
   \|\ve{x}_i\|_2^2}\right)} \nonumber\\
 &  & = \frac{\sup_{{\mathcal{H}}_{*}}\|\theta\|_2 \sqrt{\sum_i \|{\ve{x}}_{i}\|_2^2}}{n} \left(1 +  \frac{1}{2|\Sigma_n|}
\sum_{\Sigma_n} {\frac{ \sum_{i_1\neq i_2}
   {\sigma_{i_1} \sigma_{i_2} {\ve{x}}_{i_1}^\top {\ve{x}}_{i_2}}}{\sum_i
   \|\ve{x}_i\|_2^2}}\right) \label{lineq1}\nonumber\\
 &  = & \frac{\sup_{{\mathcal{H}}_{*}}\|\theta\|_2 \sqrt{\sum_i \|\hat{\ve{x}}_{Ti}\|_2^2}}{n} \left(1 +  \frac{1}{2|\Sigma_n|}
\sum_{\Sigma_n} {\frac{ \sum_{i_1\neq i_2}
   {\sigma_{i_1} \sigma_{i_2} {\ve{x}}_{i_1}^\top {\ve{x}}_{i_2}}}{\sum_i
   \|\ve{x}_i\|_2^2}}\right) \label{lineq2}\:\:.
\end{eqnarray}
We prove eq. (\ref{lineq2}). Because $\PERM_*$ is
a permutation, we have the penultimate identity of:
\begin{eqnarray}
\sum_i \|\hat{\ve{x}}_{Ti}\|_2^2 & = & \sum_i
\|(\hat{\ve{x}}_{Ti})_\FDP\|_2^2 + \sum_i
\|(\hat{\ve{x}}_{Ti})_\LDP\|_2^2\noindent\\
 & = &  \sum_i
\|(\ve{x}_i)_\FDP\|_2^2 + \sum_i
\|(\hat{\ve{x}}_{Ti})_\LDP\|_2^2\noindent\\
 & = &  \sum_i
\|(\ve{x}_i)_\FDP\|_2^2 + \sum_i
\|(\ve{x}_i)_\LDP\|_2^2\noindent\\
 & = &  \sum_i \|\ve{x}_{i}\|_2^2\:\:.
\end{eqnarray}
The second identity follows from the fact that $\PERM_*$ acts on
indexes from $\LDP$. We obtain 
\begin{eqnarray}
R_n & \leq & \frac{\hat{X}_{T*} \hat{\theta}_*}{\sqrt{n}}
\end{eqnarray}
after remarking, following \citep[Lemma 1]{pnncLF}, that because Rademacher variables are
i.i.d., $\sum_{\Sigma_n} {\sigma_{i_1} \sigma_{i_2} } = 0$, which yields:
\begin{eqnarray}
\sum_{\Sigma_n} {\frac{ \sum_{i_1\neq i_2}
   {\sigma_{i_1} \sigma_{i_2} \ve{x}_{i_1}^\top \ve{x}_{i_2}}}{\sum_i
   \|\ve{x}_i\|_2^2}} & = & \frac{1}{\sum_i
   \|\ve{x}_i\|_2^2} \sum_{i_1\neq i_2}{\ve{x}_{i_1}^\top
   \ve{x}_{i_2}\sum_{\Sigma_n} {\sigma_{i_1} \sigma_{i_2} }} = 0\:\:, \label{brc}
\end{eqnarray}
and ends the proof of Lemma \ref{lerc1}.
\end{proof}\\
We now remark that we have from Theorem \ref{thAPPROX1} and
the triangle inequality:
\begin{eqnarray}
\|\ve{\theta}^*_T\|_2 & \leq & \left( 1+ C(n) \cdot \left( 1+ \frac{\delta_\rho}{\deltamargin}\right)
\right) \cdot \|\ve{\theta}^*_0\|_2 \:\:,
\end{eqnarray}
which means that we can let
\begin{eqnarray}
\hat{\theta}_* & \defeq &  \left( 1+ C(n) \cdot \left( 1+ \frac{\delta_\rho}{\deltamargin}\right)
\right) \cdot \theta_*\:\:.
\end{eqnarray}
Letting $L$ denote the Lipschitz constant for the Taylor loss, we get
from \citep[Theorem 7]{bmRA} that with probability $\geq 1 - \delta$
over the drawing of $S \sim \mathcal{D}^n$,
\begin{eqnarray}
\Pr_{(\ve{x}, y)\sim \mathcal{D}} \left[y (\ve{\theta}^*_T)^\top \ve{x}
\leq 0 \right] & \leq & \ell_{S,
    \gamma}(\ve{\theta}^*_T) + 2L R_n (\ve{\theta}^*_T) + \sqrt{\frac{\ln(2/\delta)}{2n}}\:\:.\label{eqGEN1}
\end{eqnarray}
So, using Theorem \ref{thDIFFLOSS},
ineq. (\ref{eqGEN1}) implies
\begin{eqnarray}
\Pr_{(\ve{x}, y)\sim \mathcal{D}} \left[y (\ve{\theta}^*_T)^\top \ve{x}
\leq 0 \right] & \leq & \ell_{S,
    \gamma}(\ve{\theta}^*_0) + \deltamarginrho \left( \delta_{\mu_0} + 
6 \deltamarginrho\right)\cdot C(n) + \frac{2L X_{*}}{\sqrt{n}}
\cdot \hat{\theta}_* +
\sqrt{\frac{\ln(2/\delta)}{2n}}\nonumber\\
 & \leq & \ell_{S,
    \gamma}(\ve{\theta}^*_0) + \deltamarginrho \left( \delta_{\mu_0} + 
6 \deltamarginrho\right)\cdot C(n) \nonumber\\
 & & + \frac{2L X_{*}}{\sqrt{n}}
\cdot \left( 1+ C(n) \cdot \left( 1+ \frac{\delta_\rho}{\deltamargin}\right)
\right) \cdot \theta_* +
\sqrt{\frac{\ln(2/\delta)}{2n}}\nonumber\\
 & \leq & \ell_{S,
    \gamma}(\ve{\theta}^*_0) + \deltamarginrho \left( \delta_{\mu_0} + 
6 \deltamarginrho\right)\cdot C(n) \nonumber\\
 & & + \frac{2L X_{*}}{\sqrt{n}}
\cdot \left( 1+ C(n) \cdot \left( 1+ \frac{\delta_\rho}{\deltamargin}\right)
\right) \cdot \theta_* +
\sqrt{\frac{\ln(2/\delta)}{2n}}\nonumber\\
 & & = \ell_{S,
    \gamma}(\ve{\theta}^*_0) + \frac{2L X_{*} \theta_*}{\sqrt{n}} +
  \sqrt{\frac{\ln(2/\delta)}{2n}} + U(n)\:\:,
\end{eqnarray}
with 
\begin{eqnarray}
U(n) & = & \left( \deltamarginrho \left( \delta_{\mu_0} + 
6 \deltamarginrho\right) + \frac{2L X_{*} \theta_*}{\sqrt{n}} \cdot
\left( 1+ \frac{\delta_\rho}{\deltamargin}\right)\right) \cdot
C(n) \nonumber\\
 & = & \left( \deltamarginrho \left( \delta_{\mu_0} + 
6 \deltamarginrho\right) + \frac{2L}{\sqrt{n}} \cdot
\left( \deltamargin + \delta_\rho\right)\right) \cdot
C(n) \nonumber\\
 & = & \deltamarginrho \cdot \left( \delta_{\mu_0} + 
6 \deltamarginrho + \frac{4L}{\sqrt{n}} \right) \cdot
C(n)\:\:,
\end{eqnarray}
achieving the proof of Theorem \ref{thGENTHETA}.

\section{Cryptographic longterm keys, entity resolution and learning}
\label{app:clk}

We stressed the importance of Theorem \ref{thAPPROX1}, which is
applicable regardless of the permutation matrix and in fact also
relevant to the non private setting. While we think this Theorem may
be useful to optimize entity resolution algorithms with the objective
to learn more accurately afterwards, a question is how does our theory
fits to the actual system we are using, \emph{cryptographic longterm keys}, (CLKs, \cite{schnell11}).

To summarize, a CLK proceed in two phases: (i) it hashes the input in
a set of objects (\textit{e.g.} bigrams, digits, etc.), (ii) for each
object, it hashes \textit{several times} its binary code into a Bloom
filter of fixed size, initialized to $0$, flipping the zeroes to $1$
for each hasing value. Hence, computing a CLK
amounts to putting at most $hs$ bits to $1$ in an array of $l$ bits,
where $h$ is the number of hash functions and $s$ is the number of
objects. Comparing two binary CLKs $z$ and $z'$ is done with the dice
coefficient:
\begin{eqnarray}
D(z, z') & \defeq & 2\times \frac{1(z\wedge z')}{1(z) + 1(z')}\:\:,
\end{eqnarray}
where $1(z)$ is the set on indexes having $1$ in the Bloom filter (binary)
encoding of $z$. Since $D(z, z') \in [0,1]$, a threshold $\tau \in [0,1]$ is in
general learned, above which $z$ and $z'$ are considered representing
the same inputs.

CLKs have the key property that if the two inputs are the same, then
$D(z, z') = 1 \geq \tau$, for any applicable $\tau$. Therefore,
whenever two entities from $\anchor$ and $\shuffle$ have the same input values,
\textit{if} these inputs do not have errors, then their Dice
coefficient is maximal and they are recognized as representing the
same entity.

Suppose for simplicity that we study the $(\epsilon, \delta)$-accuracy
assumptions on binary vectors, the common part of $\anchor$ and
$\shuffle$ being the representation of the CLK. In this case, ignoring
the features not in the Bloom filter, it is
not hard to see that the unit vector $\ve{w}$ which maximizes the
stretch $\vstretch((\hat{\ve{x}}_{ti} -
\ve{x}_{i})_{\mathsf{CLK}},\ve{w}_{\mathsf{CLK}})$ (the features of
${\mathsf{CLK}}$ are included in those of $\shuffle$) is the one whose coordinates over
the $l$ bits of the CLK, in absolute value, are proportional to the
coordinates of $\hat{\ve{x}}_{ti} \oplus \ve{x}_{i}$,
where $\oplus$ is the exclusive or. We get for this $\ve{w}$
\begin{eqnarray}
\vstretch((\hat{\ve{x}}_{ti} -
\ve{x}_{i})_{\mathsf{CLK}},\ve{w}_{\mathsf{CLK}}) & = & \sqrt{1(\hat{\ve{x}}_{ti} \oplus \ve{x}_{i})}\:\:,\label{eqstretch2}
\end{eqnarray}
while for this $\ve{w}$, we also get
\begin{eqnarray}
\vstretch(\ve{x}_{i},\ve{w}) & = & \sqrt{1(\ve{x}_{i} \wedge \neg \hat{\ve{x}}_{ti})}\:\:.
\end{eqnarray}
Assuming the CLKs have exactly $hs$ bits to $1$, we have
$1(\hat{\ve{x}}_{ti} \oplus \ve{x}_{i}) = 2 \cdot 1(\ve{x}_{i} \wedge
\neg \hat{\ve{x}}_{ti})$. We have $\sqrt{2} - 1 \approx 0.414$, so depending on the
actual norm of the observations, we might be able to find $\delta > 0$
such that $(\epsilon, \delta)$-accuracy holds for $\epsilon < 1$. Of
course, (i) we forgot all the other features of $\shuffle$ and (ii) we
have only analyzed $(\epsilon, \delta)$-accuracy for the choices of
$\ve{w}$ that maximize the left-hand side of
eq. (\ref{eqstretch2}). However, this simplistic analysis hints on the
fact that entity resolution based on CLKs is not just relevant to
private entity resolution, it may also be a good choice for learning
in our setting.



\end{document}